\documentclass{article}

\usepackage[dblblindworkshop, final]{neurips_2026}
\workshoptitle{Trustworthy AI for Good}

\usepackage[utf8]{inputenc} 
\usepackage[T1]{fontenc}    
\usepackage{hyperref}       
\usepackage{url}            
\usepackage{booktabs}       
\usepackage{amsfonts}       
\usepackage{nicefrac}       
\usepackage{microtype}      
\usepackage{xcolor}         

\usepackage{amsmath}
\usepackage{amssymb}
\usepackage{bm}
\usepackage{amsthm}
\usepackage{graphicx}
\usepackage{algorithm}
\usepackage{algorithmic}

\def\gX{{\mathcal{X}}}
\def\gY{{\mathcal{Y}}}
\def\gS{{\mathcal{S}}}
\def\gA{{\mathcal{A}}}
\def\gD{{\mathcal{D}}}
\def\vv{{\bm{v}}}
\def\vl{{\bm{l}}}
\def\vt{{\bm{t}}}
\def\vz{{\bm{z}}}
\def\vg{{\bm{g}}}
\def\Eqref#1{Eq.~\eqref{#1}}

\DeclareMathOperator*{\argmax}{arg\,max}
\DeclareMathOperator*{\argmin}{arg\,min}

\newcommand{\bcc}[1]{\left\{{#1}\right\}}
\newcommand{\brr}[1]{\left({#1}\right)}
\newcommand{\bss}[1]{\left[{#1}\right]}
\newcommand{\ipp}[2]{\left\langle{#1},{#2}\right\rangle}
\newcommand{\norm}[1]{\left\lVert#1\right\rVert}

\newcommand{\Expectover}[2]{\mathbb{E}_{#1}\!\left[#2\right]}
\newcommand{\abs}[1]{\left\vert#1\right\vert}

\newcommand{\LINECOMMENT}[1]{\STATE {\color{blue}\ttfamily\small \(\triangleright\) #1}}

\newtheorem{theorem}{Theorem}
\newtheorem{proposition}[theorem]{Proposition}

\title{Fairness-Aware Test-Time Prompt Tuning}

\author{%
  Yoann Launay\thanks{Work completed while at the Risk and Security AI Lab, Visa Inc.} \\
  University of Cambridge \\
  \texttt{yl844@cam.ac.uk} \\
  \And
  Parameswaran Kamalaruban\\
  Risk and Security AI Lab, Visa Inc.\\
  \texttt{kaparame@visa.com}\\
  \And
  Tom Kempton\thanks{Work completed while a consultant at the Risk and Security AI Lab, Visa Inc.}\\
  University of Manchester\\
  \texttt{thomas.kempton@manchester.ac.uk}\\
  \And
  Stuart Burrell\\
  Risk and Security AI Lab, Visa Inc.\\
  \texttt{sburrell@visa.com}\\
  \And
  David Sutton\\
  Risk and Security AI Lab, Visa Inc.\\
  \texttt{dsutton@visa.com}\\
}

\begin{document}

\maketitle

\begin{abstract}
Vision-language models have displayed remarkable capabilities in multi-modal understanding and are increasingly used in critical applications where economic and practical deployment constraints prohibit re-training or fine-tuning. However, these models can also exhibit systematic biases that disproportionately affect protected demographic groups and existing approaches to addressing these biases require extensive model retraining and access to demographic attributes. There is a clear need to develop test-time adaptation (TTA) approaches that improve the fairness characteristics of pretrained models under distributional shift. In this paper, we evaluate how episodic TTA affects fairness in CLIP classification under subpopulation shifts and develop \textsc{FairTPT}, a novel fairness-aware episodic TTA method that jointly minimizes target marginal entropy while maximizing spurious marginal entropy through soft-prompt tuning. We find that standard episodic TTA generally exacerbates disparities between majority and minority groups, that blinding a model to spurious attributes without degrading target performance is inherently challenging, and that excessive blinding can lead to catastrophic forgetting. This model collapse can be prevented by monitoring test-time changes in target loss within the linear regime, while still achieving fairness improvements on reactive data and preserving overall performance. \textsc{FairTPT} outperforms all state-of-the-art episodic test-time debiasing methods and establishes a foundation for robust TTA, which is essential for achieving fairness in practice.
\end{abstract}
\section{Introduction} 
\label{sec:introduction}

Vision–language models (VLMs) such as CLIP~\cite{radford2021learning} have achieved remarkable success across multi-modal tasks, including recognition, retrieval, and reasoning~\cite{alayrac2022flamingo,cherti2023reproducible,li2023blip}. A core strength of these models is zero-shot classification, which makes predictions without task-specific fine-tuning by aligning images with class-descriptive prompts in a shared embedding space. This ability is especially valuable when labeled data is scarce or fine-tuning is infeasible due to computational or deployment constraints, enabling strong cross-domain generalization~\cite{radford2021learning,jia2021scaling}. However, recent studies show that VLMs often inherit and sometimes amplify systematic biases from pre-training data~\cite{Birhane2021MultimodalDM,Agarwal2021EvaluatingCT,Hamidieh2024IdentifyingIS,Konavoor2025VisionLanguageMD,wang2021gender,hall2023vision}. Such biases can propagate to downstream zero-shot classifiers, leading to disparities in performance across sensitive attributes (e.g., gender, race) and undermine fairness in socially critical domains~\cite{healthcareimpact}. 

Test-time adaptation (TTA) has emerged as a promising approach to improve generalization beyond zero-shot prediction by adapting models to unseen distributions during inference using only unlabeled test inputs~\cite{Lee2022ConfidenceSF,Kundu22,Gong22b,Goyal22,Sinha2022,AdaContrast22}. For VLMs, episodic TTA methods such as Test-Time Prompt Tuning (\textsc{TPT})~\cite{shu2022test} and \textsc{Zero}~\cite{farina2024frustratingly} adapt soft prompts or aggregate predictions over augmented views to improve accuracy. While these methods improve average accuracy, their impact on subgroup robustness (i.e., the ability to maintain performance across all sensitive groups) and fairness remains largely unexplored. Existing evaluations primarily focus on overall performance, overlooking fairness-related metrics. Moreover, prior work has highlighted instability and hyperparameter sensitivity in TTA methods~\cite{pitfallsTTA23,IllusionProgress25}, raising concerns about their reliability in fairness-critical settings.

We address this gap by studying fairness-aware episodic TTA for VLMs, where adaptation is performed independently for each test instance and the model is reset after each episode. Formally, given an unlabeled test image $x$ with neither the target attribute $y$ nor the sensitive attribute $s$, the goal is to adapt the model (through prompt tuning) so that: (i) subgroup robustness is improved, i.e., disparities in accuracy across sensitive attributes are reduced, (ii) overall accuracy is maintained, and (iii) hyperparameter sensitivity is minimized, since tuning at test time is impractical. Existing test-time debiasing methods~\cite{gerych2024bendvlm,adilazero,lu2025mitigating} often require batched test inputs or partial supervision of sensitive attributes, making them unsuitable for strict, fully unsupervised episodic adaptation. To our knowledge, only \cite{chuang2023debiasing} propose an episodic unsupervised debiasing method (\textsc{OrthCali}). It projects out biased directions in text embeddings in a zero-shot setting, rather than following a standard TTA paradigm.

We introduce Fairness-Aware Test-Time Prompt Tuning (\textsc{FairTPT}), an episodic TTA method that jointly minimizes the marginal entropy of target attribute predictions to encourage overall accuracy, while maximizing the marginal entropy of sensitive attribute predictions to reduce reliance on spurious correlations. This dual-entropy objective mitigates subgroup disparities while preserving the simplicity and efficiency of prompt-level adaptation. To further stabilize adaptation and prevent catastrophic forgetting (where excessive debiasing degrades target performance), we propose a lightweight learning-rate adaptation heuristic that monitors target entropy changes during adaptation. Our method is fully unsupervised, operates in the episodic setting, and requires no access to sensitive attribute labels at test time.

Our contributions are summarized as follows:
\begin{enumerate}
\item \textbf{Evaluating fairness of episodic TTA methods.} We present the first systematic fairness evaluation of episodic TTA methods (\textsc{TPT}, \textsc{Zero}) for VLMs, showing that they often fail to improve subgroup robustness, can exacerbate disparities, and are highly sensitive to hyperparameters.
\item \textbf{Proposing and evaluating a fairness-aware episodic TTA method.} We propose \textsc{FairTPT}, which jointly mitigates spurious reliance and preserves overall accuracy, while being more robust to hyperparameter choices. On standard fairness benchmarks, \textsc{FairTPT} outperforms or matches baselines (\textsc{TPT}, \textsc{Zero}, \textsc{OrthCali}) in both overall and subgroup accuracy, and shows markedly improved stability under varying hyperparameters. 
\end{enumerate}

\section{Related Work} 
\label{sec:related-work}

\textbf{Test-time adaptation of VLMs.} TTA methods~\cite{wang2020tent,zhang2022memo,niutowards,shu2022test,xiao2024beyond} adapt models during inference with minimal computation, either via partial parameter tuning or tuning-free strategies~\cite{farina2024frustratingly}. Popular approaches like MEMO~\cite{zhang2022memo}, Tent~\cite{wang2020tent}, and TPT~\cite{shu2022test} leverage entropy minimization to encourage confident predictions. TTA can be applied in online settings, where the model is updated continually across test samples, or in episodic settings, where adaptation is reset to the base model for each episode or input. The latter is the focus of this work. While effective for improving average accuracy under distribution shift, TTA methods can be unstable and hyperparameter-sensitive~\cite{pitfallsTTA23,IllusionProgress25}, with fairness impacts largely unstudied.

\textbf{Training-time debiasing of VLMs.} Many works mitigate bias during training or fine-tuning~\cite{wang2021gender,berg-etal-2022-prompt,zhang2022contrastive,Seth23dear}. Examples include removing embedding dimensions correlated with sensitive attributes~\cite{wang2021gender}, adversarial prompt learning~\cite{berg-etal-2022-prompt}, residual feature debiasing~\cite{Seth23dear}, and contrastive adapter training~\cite{zhang2022contrastive}. Broader work on group robustness includes robust optimization with group labels~\cite{arjovsky2019invariant,sagawadistributionally} and label-free methods such as LfF~\cite{nam2020learning} and JTT~\cite{liu2021just}. While effective, these approaches require retraining on labeled datasets, making them impractical for deployment-time bias mitigation.

\textbf{Test-time debiasing of VLMs.} At inference, most debiasing methods perform embedding space projections to remove biased directions~\cite{chuang2023debiasing,gerych2024bendvlm,adilazero,lu2025mitigating}. However, many are not fully unsupervised with respect to sensitive attributes~\cite{gerych2024bendvlm,adilazero,lu2025mitigating}, require batched inputs~\cite{adilazero,lu2025mitigating}, or rely on LLM-generated attribute descriptions~\cite{adilazero,PerceptionCLIP}. \textsc{OrthCali}~\cite{chuang2023debiasing} is a fully episodic, unsupervised baseline applicable to our setting; however, it does not follow a standard TTA procedure. In this work, we explore fairness-aware, entropy-based episodic TTA for VLMs. \textsc{SEraser}~\cite{ma2025spurious} can be viewed as an extreme form of entropy maximization that suppresses spurious features without explicitly balancing target accuracy.

\section{Problem Setup and Background}
\label{sec:problem-setup}

\textbf{Problem Setup.} We consider a test-time image classification task in a zero-shot or test-time adaptive setting, where the goal is to predict a target attribute from an input image without access to any task-specific labeled training data. Each test instance consists of an image $x \in \gX$, an unknown target label $y \in \gY = \bcc{y_1, y_2, \dots, y_C}$, and a sensitive or spurious attribute $s \in \gS = \bcc{s_1, s_2, \dots, s_G}$. The classifier only receives $x$ at inference time; $y$ and $s$ are used solely for evaluation.

Since no labeled data is available at test time, we employ zero-shot or test-time adaptive classifiers derived from vision-language models (VLMs) such as the Contrastive Language-Image Pre-training (CLIP) model~\cite{radford2021learning}. These models are pre-trained on large-scale image-text datasets and can perform recognition tasks directly by aligning image features with prompt-based textual features. However, due to correlations in the pretraining data, VLMs may encode spurious associations between the target label $y$ and the sensitive attribute $s$. For instance, textual descriptions of certain occupations may be disproportionately associated with one gender. Consequently, classifiers derived from such VLMs can exhibit subgroup performance disparities, particularly when the joint or marginal distributions of $(x, y, s)$ shifts between the pretraining and test domains.

To systematically quantify and analyze these disparities, we adopt the group robustness framework of \cite{sagawadistributionally}. Let $P_\textnormal{test}$ denote the test distribution over triples $(x,y,s)$. For a classifier $f: \gX \to \gY$, we define the average error as $\textsc{Err}^\textnormal{avg}(f) = \Expectover{(x, y) \sim P_\textnormal{test}}{\mathbf{1}\{f(x) \ne y\}}$, and the worst-group error as $\textsc{Err}^\textnormal{wg}(f) = \max_{s \in \gS} \Expectover{(x, y) \sim P_\textnormal{test} \mid s}{\mathbf{1}\{f(x) \ne y\}}$, where $P_\textnormal{test} \mid s$ denotes the conditional distribution given the sensitive attribute value $s$. The robustness gap is then defined as the difference $\textsc{Gap}(f) = \textsc{Err}^\textnormal{avg}(f) - \textsc{Err}^\textnormal{wg}(f)$, capturing the discrepancy between the overall and worst-case group performance. The goal of this work is to design test-time adaptive classifiers that achieve high overall accuracy while minimizing the robustness gap, thereby promoting fairness under distribution shifts along sensitive attributes.

\textbf{Zero-Shot Classification.} To perform zero-shot classification, we leverage a pre-trained VLM, which learns a shared embedding space for images and natural language prompts. The model consists of an image encoder $f_\textnormal{V}$, which maps an image $x \in \gX$ to a feature vector $\vv = f_\textnormal{V}(x) \in \mathbb{R}^d$, and a text encoder $f_\textnormal{L}$, which maps a text sequence (e.g., a class-descriptive prompt) to a feature vector $\vl \in \mathbb{R}^d$. The classifier predicts the target attribute $y \in \gY$ by measuring the similarity between the image feature $\vv$ and a set of class-specific text features $\vl_y$ derived from prompts.

To construct these prompts, a hand-crafted template containing a placeholder (e.g., \texttt{"a photo of a [T-CLS] person"} where \texttt{[T-CLS]} is the placeholder for a target class) is instantiated with a textual description for each target label $y \in \gY$. For example, for the attribute \texttt{"smiling"}, we may use: \{ \texttt{"a photo of a smiling person"}, \texttt{"a photo of a non-smiling person"} \}. These completed prompts are tokenized and embedded into soft prompt representations using the model's token embedding table. For simplicity, each soft prompt is represented as a non-ordered tuple $(\vt_\textnormal{ctx}, \vt_y) \in \mathbb{R}^{L_y \times d}$, where $\vt_\textnormal{ctx}$ encodes the shared context tokens, and $\vt_y$ encodes the label-specific portion. The length $L_y$ corresponds to the total number of tokens in the prompt for label $y$, and $d$ is the token embedding dimension. This results in a soft prompt set $\bcc{\vt_\textnormal{ctx}; \gY} = \bcc{(\vt_\textnormal{ctx}, \vt_y): y \in \gY}$, which is used to generate the set of text features $\bcc{\vl_y: y \in \gY}$, where $\vl_y = f_\textnormal{L}((\vt_\textnormal{ctx}, \vt_y))$.

During inference, the image feature $\vv$ is compared to each text feature $\vl_y$ using cosine similarity, denoted $\texttt{sim}(\vl_y, \vv) = \frac{\ipp{\vl_y}{\vv}}{\norm{\vl_y} \cdot \norm{\vv}}$. These similarity scores are passed through a softmax function to produce a probability distribution over the target classes $ y \in \gY$:
\[
p(y \mid x, \bcc{\vt_\textnormal{ctx}; \gY}, \tau) ~=~ \frac{\exp\brr{\tau \cdot \texttt{sim}(\vl_y, \vv)}}{\sum_{y' \in \gY} \exp\brr{\tau \cdot \texttt{sim}(\vl_{y'}, \vv)}} ,
\]
where $\tau > 0$ is a temperature parameter that controls the sharpness of the output distribution. The final prediction is chosen from the above distribution: $\widehat{y} = \textnormal{argmax}_y p(y \mid x, \bcc{\vt_\textnormal{ctx}; \gY}, \tau)$. This zero-shot framework forms the foundation upon which we develop our fairness-aware prompt tuning method.

\textbf{Test-Time Adaptation (TTA).} We study episodic test-time adaptation, where the model processes one test image at a time and is allowed to adapt its parameters within a reasonable time before producing the final prediction. The model receives an image $x \in \gX$ and has no access to its true label $y$ or sensitive attribute $s$ during inference. Adaptation is performed using a pre-trained VLM $\bcc{f_\textnormal{V}, f_\textnormal{L}}$, together with a set of augmentation functions $\gA = \bcc{A_1, \dots, A_N}$ that generate alternative views of $x$. This is performed using AugMix~\cite{hendrycks2020augmix}. The aim is to leverage these augmented views to improve prediction robustness under distribution shifts.

We first consider Test-Time Prompt Tuning (\textsc{TPT})~\cite{shu2022test}, an adaptation of marginal entropy minimization (MEM)~\cite{zhang2022memo} to the vision-language setting. MEM was originally proposed for unimodal vision models, but \cite{shu2022test} demonstrated that it can be repurposed for VLMs by optimizing prompt embeddings rather than the model weights~\cite{zhou2022learning,zhou2022conditional,khattak2023maple}. Specifically, a trainable soft-prompt set $\bcc{\vt_\textnormal{ctx}; \gY} = \bcc{(\vt_\textnormal{ctx}, \vt_y): y \in \gY}$ is constructed as in the zero-shot procedure, where $\vt_\textnormal{ctx}$ represents the shared context and $\vt_y$ encodes the label-specific portion. For each test image $x$, the augmentation set $\gA$ is applied to generate $N$ views, producing a marginal predictive distribution:
\[
\overline{p}(y \mid x, \bcc{\vt_\textnormal{ctx}; \gY}, \tau) ~=~ \frac{1}{N} \sum_{i=1}^N p(y \mid A_i(x), \bcc{\vt_\textnormal{ctx}; \gY}, \tau),
\]
where $p(\cdot \mid \cdot)$ is computed as in the zero-shot classifier and $\tau > 0$ is the softmax temperature. Justified by the desired augmentation-invariance and improved confidence of the predictions~\cite{zhang2022memo}, the adaptation objective is to minimize the entropy of this marginal distribution for each input $x$:
\begin{equation}
\begin{aligned}
&\min_{\vt_\textnormal{ctx}} H(\overline{p}(\cdot \mid x, \bcc{\vt_\textnormal{ctx}; \gY}, \tau)) ~:=~ - \sum_{y \in \gY} \overline{p}(y \mid x, \bcc{\vt_\textnormal{ctx}; \gY}, \tau) \cdot \log \overline{p}(y \mid x, \bcc{\vt_\textnormal{ctx}; \gY}, \tau) .
\end{aligned}
\label{eq:tpt-objective}
\end{equation}
This optimization is performed for $n_\textnormal{epochs}$ iterations using an update rule $G$ with learning rate $\eta$, yielding an adapted context embedding $\vt_\textnormal{ctx}^*$. The final prediction is then obtained by applying the zero-shot classification procedure with the updated soft-prompt set $\bcc{\vt_\textnormal{ctx}^*; \gY}$. Following \cite{shu2022test}, we employ AdamW~\cite{loshchilovdecoupled} as the optimizer, with $\eta = 5e-3$ and $n_\textnormal{epochs} = 1$.\footnote{Note that we did not use CoOp or CoCoOp ~\cite{zhou2022learning, zhou2022conditional}, as they were not necessary to achieve a performance uplift in the original study.}

As a recent and simpler alternative, we also consider the \textsc{Zero} method of \cite{farina2024frustratingly}, which avoids any adaptation step and instead aggregates predictions from the augmented views directly. The predicted label is obtained as 
\begin{equation}
\begin{aligned}
&\textsc{Zero}(x, \bcc{\vt_\textnormal{ctx}; \gY}) ~=~ \argmax_{y \in \gY} \sum_{i=1}^N \lim_{\tau \to 0^{+}} p(y \mid A_i(x), \bcc{\vt_\textnormal{ctx}; \gY}, \tau) ,
\end{aligned}
\label{eq:zero-opt}
\end{equation}
corresponding to majority voting over deterministic predictions from each view. 

For both \textsc{TPT} and \textsc{Zero}, performance and stability can be improved through confidence-based view filtering~\cite{shu2022test,niutowards,farina2024frustratingly}. The idea is to retain the top-$\rho$ fraction of high-confidence (low-entropy) augmented views for the target attribute prediction. High-entropy views are discarded as they often lack sufficient information for reliable classification. Formally, the retained augmentation set is
$
\gA_\textnormal{filtered}(x) ~=~ \bcc{A_i: H(p(\cdot \mid A_i(x), \bcc{\vt_\textnormal{ctx}; \gY}, \tau)) < \delta} ,
$
where $\delta$ is a threshold chosen to retain the top-$\rho$ fraction of most confident views. In both cases, the original studies were setting $\rho=0.1$.

\section{Fairness-Aware TPT}
\label{sec:our-method}

\begin{figure*}
    \centering
    \includegraphics[width=\linewidth]{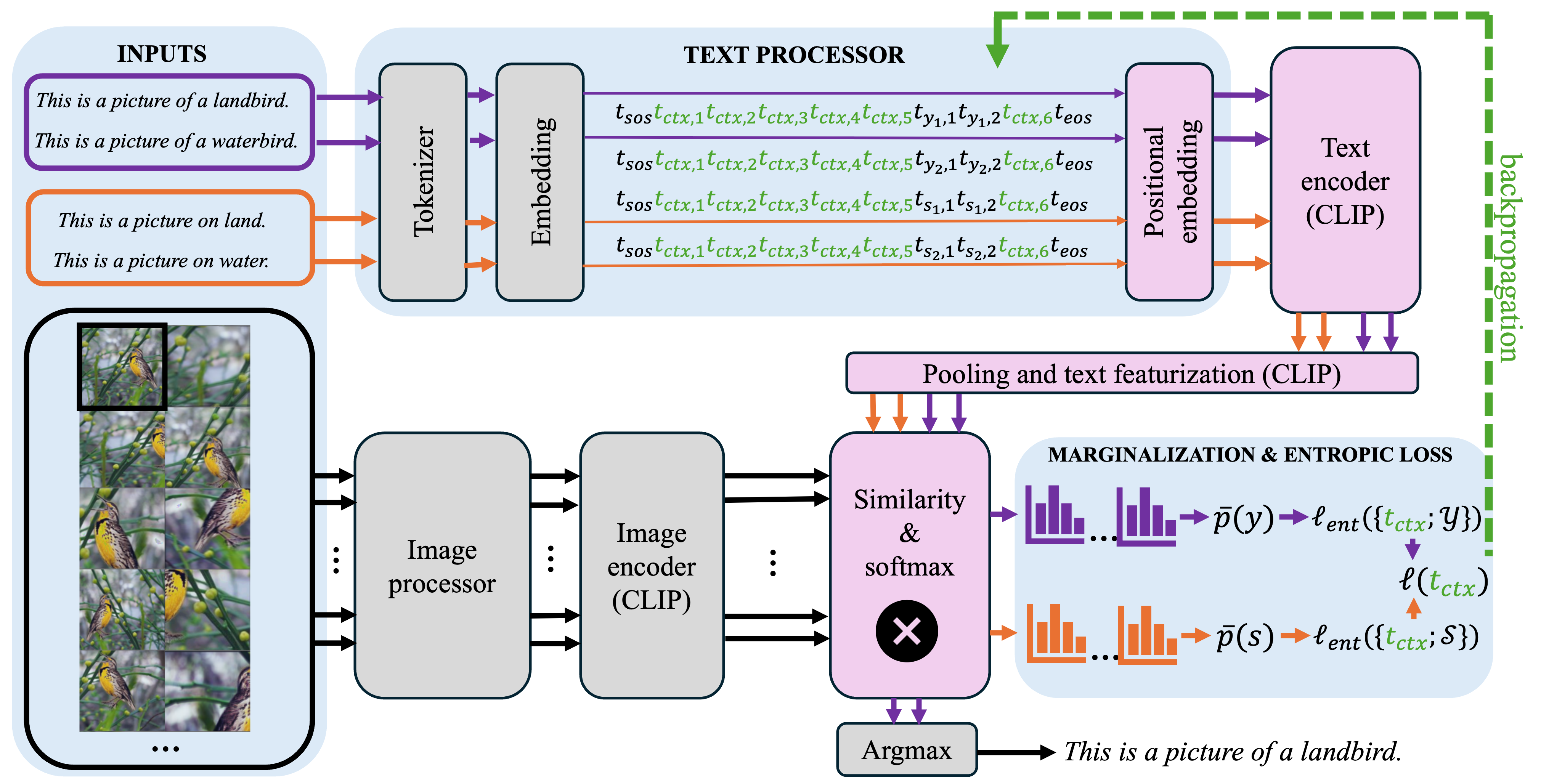}
    \caption{\textsc{FairTPT} pipeline using a single unlabeled image at a time (black frame). The soft prompt $\vt_\textnormal{ctx}$ is the only parameter tuned at test time, requiring text processor to be re-written. Pink and grey blocks have frozen weights; gradients are computed only for pink blocks. This example illustrates the objective in \Eqref{eq:fair-tpt-objective}. For the alternative objective in \Eqref{eq:fair-tpt-alt-objective}, the spurious prompts (orange) and their corresponding soft prompts would also include target labels, resulting in four spurious inputs in this example.}
    \label{fig:pipeline}
\end{figure*}

\textbf{Fairness-Aware Test-Time Prompt Tuning (\textsc{FairTPT}).} We propose Fairness-Aware Test-Time Prompt Tuning (\textsc{FairTPT}), a fairness-aware extension of test-time prompt tuning designed to improve subgroup robustness while preserving overall accuracy. The key observation is that standard entropy minimization in \textsc{TPT} encourages confident predictions by exploiting any predictive signal available at test time, which can inadvertently amplify spurious correlations. \textsc{FairTPT} counteracts this effect by explicitly maximizing the marginal entropy over a specified sensitive attribute while minimizing entropy over the target attribute. By making predictions uncertain with respect to the spurious attribute, the reliance of the target prediction on spurious cues is reduced. This yields a principled min–max formulation grounded in entropy-based debiasing~\cite{roy2019mitigating}.

As in the zero-shot classification setting, we construct a trainable target soft-prompt set $\bcc{\vt_\textnormal{ctx}; \gY} = \bcc{(\vt_\textnormal{ctx}, \vt_y): y \in \gY}$ where $\vt_\textnormal{ctx}$ is the shared context embedding and $\vt_y$ encodes the target label description. In parallel, we construct a spurious soft-prompt set $\bcc{\vt_\textnormal{ctx}; \gS} = \bcc{(\vt_\textnormal{ctx}, \vt_s): s \in \gS}$ for the sensitive attribute, reusing the same shared context $\vt_\textnormal{ctx}$ across both sets.

Following the marginal entropy minimization framework in \textsc{TPT}, we compute two marginal predictive distributions: the target marginal $\overline{p}(\cdot \mid x, \bcc{\vt_\textnormal{ctx}; \gY}, \tau)$ over $\gY$ and the spurious marginal $\overline{p}(\cdot \mid x, \bcc{\vt_\textnormal{ctx}; \gS}, \tau)$ over $\gS$, each obtained by averaging predictions over the same augmented views $\{A_i(x): A_i \in \gA\}$. To jointly promote confident predictions for the target attribute and uncertainty for the spurious attribute, we optimize the objective
\begin{equation}
\min_{\vt_\textnormal{ctx}} \frac{1}{1+\lambda_\textnormal{fair} }\ell_\textnormal{ent} (x, \bcc{\vt_\textnormal{ctx}; \gY}) - \frac{\lambda_\textnormal{fair}}{1+\lambda_\textnormal{fair}} \cdot \ell_\textnormal{ent} (x, \bcc{\vt_\textnormal{ctx}; \gS}) ,
\label{eq:fair-tpt-objective}
\end{equation}
where $\lambda_\textnormal{fair} \geq 0$ balances the accuracy–fairness trade-off\footnote{Note that $\lambda_\textnormal{fair} = 0$ recovers the \textsc{TPT} baseline, and $\lambda_\textnormal{fair} \to \infty$ corresponds to optimizing solely for reduced confidence in the spurious attribute, ignoring the target attribute.}, and for $\star \in \bcc{\gY, \gS}$
\[
\ell_\textnormal{ent} (x, \bcc{\vt_\textnormal{ctx}; \star}) ~=~ \frac{1}{\log \abs{\star}} \cdot H(\overline{p}(\cdot \mid x, \bcc{\vt_\textnormal{ctx}; \star}, \tau)) .
\]
The objective is minimized for $n_\textnormal{epochs}$ iterations using a gradient-based update rule $G$ with learning rate $\eta$, producing an adapted context embedding $\vt_\textnormal{ctx}^*$. The final prediction is then obtained by applying the zero-shot classifier with the updated target soft-prompt set $\bcc{\vt_\textnormal{ctx}^*; \gY}$. The complete \textsc{FairTPT} procedure is provided in Algorithm~\ref{alg:fair-tta-cls} in the appendix.

\textbf{Alternative Formulation for the Test-Time Optimization.} We also consider an alternative formulation in which spurious prompts are conditioned on target label information. In this case, we define a joint prompt template containing two placeholders, one for the spurious attribute and one for the target attribute, e.g., \texttt{"a photo of a [S-CLS] celebrity [T-CLS]"}. For each pair $(y, s) \in \gY \times \gS$, this yields a prompt such as \texttt{"a photo of a male celebrity smiling"} or \texttt{"a photo of a female celebrity smiling"}. Each prompt is embedded as $(\vt_\textnormal{ctx}, \vt_y, \vt_s) \in \mathbb{R}^{L_{ys} \times d}$, and for a fixed $y$ we construct a spurious soft-prompt set $\bcc{\vt_\textnormal{ctx}, \vt_y; \gS} = \bcc{(\vt_\textnormal{ctx}, \vt_y, \vt_s): s \in \gS}$. The corresponding marginal distribution $\overline{p}(\cdot \mid x, \bcc{\vt_\textnormal{ctx}, \vt_y; \gS}, \tau)$ is then used to replace the spurious term in the objective in \Eqref{eq:fair-tpt-objective}:
\begin{equation}
\begin{aligned}
&\min_{\vt_\textnormal{ctx}} \frac{1}{1+\lambda_\textnormal{fair}}\ell_\textnormal{ent} (x, \bcc{\vt_\textnormal{ctx}; \gY}) - \frac{\lambda_\textnormal{fair}}{1+\lambda_\textnormal{fair}} \cdot \frac{1}{\abs{\gY}} \sum_{y \in \gY} \ell_\textnormal{ent} (x, \bcc{\vt_\textnormal{ctx}, \vt_y; \gS}) ,
\end{aligned}
\label{eq:fair-tpt-alt-objective}
\end{equation}
where 
\[
\ell_\textnormal{ent} (x, \bcc{\vt_\textnormal{ctx}, \vt_y; \gS}) ~=~ \frac{1}{\log \abs{\gS}} \cdot H(\overline{p}(\cdot \mid x, \bcc{\vt_\textnormal{ctx}, \vt_y; \gS}, \tau)) .
\]
We refer to the objective in \Eqref{eq:fair-tpt-objective} as the \textit{S} loss, and the objective in \Eqref{eq:fair-tpt-alt-objective} as the \textit{TS} loss.

It is also possible to make the context $\vt_\textnormal{ctx}$ target-dependent and optimize it separately for each target soft prompt:
\begin{equation}
\begin{aligned}
&\min_{\{\vt_\textnormal{ctx}^{(y)}, y \in \gY\}} \frac{1}{1+\lambda_\textnormal{fair}}\cdot\frac{1}{\abs{\gY}} \sum_{y \in \gY} \ell_\textnormal{ent} (x, \{\vt_\textnormal{ctx}^{(y)}; \gY\}) - \frac{\lambda_\textnormal{fair}}{1+\lambda_\textnormal{fair}} \cdot \frac{1}{\abs{\gY}} \sum_{y \in \gY} \ell_\textnormal{ent} (x, \{\vt_\textnormal{ctx}^{(y)}, \vt_y; \gS\}) ,
\end{aligned}
\label{eq:fair-tpt-alt-objective2}
\end{equation}
where $\vt_\textnormal{ctx}^{(y)}$ is the tunable soft context for target $y$ and excludes its tokens. We refer to this second alternative as the \textit{Super TS} loss.

\textbf{Learning Rate Adaptation.} 
Large gradient updates during test-time adaptation can degrade performance, sometimes even dropping below that of zero-shot classification \cite{niutowards}. To mitigate the risk of model collapse and monitor the adaptation speed, we employ a learning rate adaptation heuristic, whose implementation is described below and presented in Algorithm~\ref{alg:elra} in the appendix.

We choose the initial learning rate $\eta_\textnormal{init}$ from a regime in which $\vt_\textnormal{ctx}$ undergoes only a small change in norm after a single optimization step. This defines the linear regime, where a first-order Taylor expansion of the loss function $\ell$ with respect to $\vt_\textnormal{ctx}$ is valid. For a simple optimizer such as SGD, the update rule is $G: \vt_\textnormal{ctx} \gets \vt_\textnormal{ctx} -\eta \cdot \nabla_{\vt_\textnormal{ctx}} {\ell}$. The linear regime is achieved when $\eta_\textnormal{init} = \sigma \cdot \norm{\vt_{\textnormal{ctx}}}/\norm{\nabla_{\vt_\textnormal{ctx}} {\ell}}$ for $\sigma \ll 1$, where we set $\sigma = 0.01$ in our pipeline.

Operating in this regime allows direct control over parameter changes by linearly scaling the learning rate. After each optimization step, we monitor the absolute change in the target loss, $\abs{\Delta \ell_\textnormal{ent} (x, \bcc{\vt_\textnormal{ctx}; \gY})}$, and bound it by $\beta$ via the rescaling $\eta \gets \eta \cdot \beta/\abs{\Delta \ell_\textnormal{ent} (x, \bcc{\vt_\textnormal{ctx}; \gY})}$. This dynamic adjustment stabilizes adaptation and prevents divergence from overly large updates. Moreover, continuity of the pre-argmax model implies that bounding parameter changes also bounds logit variations; hence, small parameter updates in the linear regime yield predictable and controlled accuracy changes.



\textbf{Multi-Objective Optimization.} The objective in \Eqref{eq:fair-tpt-objective} (or \Eqref{eq:fair-tpt-alt-objective}) comprises the target term $\ell_\textnormal{ent} (x, \bcc{\vt_\textnormal{ctx}; \gY})$ and the spurious term $- \ell_\textnormal{ent} (x, \bcc{\vt_\textnormal{ctx}; \gS})$, the gradients of which may conflict. Optimizing only their combined scalar loss risks improving one term at the expense of the other. We therefore cast the problem as a multi-objective optimization (MO) task: 
\begin{equation}
\min_{\vt_\textnormal{ctx}} \brr{\ell_\textnormal{ent} (x, \bcc{\vt_\textnormal{ctx}; \gY}), - \ell_\textnormal{ent} (x, \bcc{\vt_\textnormal{ctx}; \gS})}. 
\label{eq:fair-tpt-multi-objective}
\end{equation}
We solve this MO problem using Jacobian descent~\cite{quinton2024jacobian}, which optimizes vector-valued objectives by computing per-term gradients and aggregating them via a chosen aggregator. The aggregator can incorporate a weighting parameter $\lambda_\textnormal{fair (mo)}$ for the spurious term. Jacobian descent guarantees a simultaneous decrease in both objectives, provided the learning rate is sufficiently small~\cite{quinton2024jacobian}.  
In our experiments, we use the TorchJD library\footnote{\url{https://torchjd.org/stable/}} with the Unconflicting Projection of Gradients (UPGrad) aggregator. 

\section{Experiments}
\label{sec:experiments}

In this section, we comprehensively evaluate \textsc{FairTPT} against episodic test-time adaptation (TTA) baselines (\textsc{TPT} and \textsc{Zero}) as well as an episodic test-time debiasing baseline, \textsc{OrthCali}. Ablations are also performed to quantify the impact of key design choices and hyperparameters. 

\subsection{Experimental setup}

We use CLIP~\cite{radford2021learning} as the base model, specifically the \texttt{ViT-L/14} variant (approximately 428M parameters), obtained from Hugging Face.\footnote{\url{https://huggingface.co/openai/clip-vit-large-patch14}} 

\textbf{Datasets.} We evaluate on standard benchmarks from the algorithmic fairness literature: \textsc{CelebA}~\cite{liu2015deep}, \textsc{UTKFace}~\cite{zhang2017age}, \textsc{FairFace}~\cite{karkkainen2021fairface}, and \textsc{WaterBirds}~\cite{sagawadistributionally}. For each run and each dataset, we uniformly sample $K=1000$ images. The target and spurious attributes used in our experiments are specified in Table~\ref{table:dataset} (in the appendix), and the initial prompt templates with placeholders for target and sensitive attributes, which are required by \textsc{FairTPT} and the episodic TTA baselines, are given in Table~\ref{table:prompt-template} (in the appendix). Across these datasets, zero-shot accuracy of the base model ranges from roughly $70\%$ to over $95\%$, while the corresponding bias (formally defined below) ranges from about $5\%$ to over $40\%$. These characteristics make the suite well-suited for probing the behavior of test-time adaptation and debiasing methods under subpopulation distribution shift.

\textbf{Evaluation Metrics.} Let $M \in \mathbb{R}^{C \times C}$ denote the overall confusion matrix, where $M_{ij}$ counts examples with true label $y_i$ predicted as $y_j$ across $K$ test episodes, and let $M^{(s)} \in \mathbb{R}^{C \times C}$ be the confusion matrix restricted to group $s \in \gS$ with $K^{(s)}$ episodes. We report overall accuracy $\textsc{Acc} = \frac{1}{K} \sum_{i \in [C]} M_{ii}$, worst-group accuracy $\textsc{WGA} = \min_{s \in \gS} \frac{1}{K^{(s)}} \sum_{i \in [C]} M_{ii}^{(s)}$, and bias $\textsc{Bias} = \textsc{Acc} - \textsc{WGA}$. We also measure equalized-odds difference, defined as $\textsc{EOD} = \mathrm{mean}_{y_i} \max \{\mathrm{Gap}_{s} \textsc{TPR}^{(s)}_{y_i}, \mathrm{Gap}_{s} \textsc{FPR}^{(s)}_{y_i}\}$, where $\textsc{TPR}^{(s)}_{y_i} = M_{ii}^{(s)} / \sum_{j \in [C]} M_{ij}^{(s)}$, $\textsc{FPR}^{(s)}_{y_i} = \sum_{j \in [C], j \ne i} M_{ji}^{(s)} / \sum_{j \in [C], j \ne i} \sum_{k \in [C]} M_{jk}^{(s)}$, and $\mathrm{Gap}_{s} \textsc{TPR}^{(s)}_{y_i} = \max_{s \in \gS} \textsc{TPR}^{(s)}_{y_i} - \min_{s \in \gS} \textsc{TPR}^{(s)}_{y_i}$ (analogously for FPR). 

\subsection{Results and Discussion}

All results are averaged over five independent runs with different random seeds. We report (i) overall performance using overall accuracy, and (ii) subgroup-level performance using worst-group accuracy (WGA; higher is better), bias (lower is better), and equalized odds difference (EOD; lower is better). Unless otherwise stated, ``our method" refers to both \textsc{FairTPT} and \textsc{FairTPT (MO)}, each equipped with ELRA and \textit{S} loss.

\textbf{Questions and evaluation protocol.} We organize the analysis around three evaluations. \textbf{E1} probes the behavior of episodic TTA methods designed to improve overall accuracy (\textsc{TPT} and \textsc{Zero}), including their fairness implications and hyperparameter sensitivity. \textbf{E2} examines an episodic test-time debiasing baseline (\textsc{OrthCali}), focusing on whether it improves subgroup performance without sacrificing overall accuracy, and how sensitive it is to its hyperparameters. \textbf{E3} evaluates \textsc{FairTPT}, asking whether it retains or improves overall and subgroup-level performance relative to zero-shot and baselines, how sensitive it is to hyperparameters, and which components matter through ablations. We report both aggregated trends (averaged across datasets) and per-dataset behavior.

\setlength{\fboxsep}{1.5pt}
\begin{table*}[h]
\centering
\tiny
\resizebox{\linewidth}{!}{
\begin{tabular}{l@{\hskip 0.05in}
c@{\hskip 0.02in}c@{\hskip 0.02in}c@{\hskip 0.02in}c@{\hskip 0.05in}
c@{\hskip 0.02in}c@{\hskip 0.02in}c@{\hskip 0.02in}c@{\hskip 0.05in}
c@{\hskip 0.02in}c@{\hskip 0.02in}c@{\hskip 0.02in}c@{\hskip 0.05in}
c@{\hskip 0.02in}c@{\hskip 0.02in}c@{\hskip 0.02in}c@{\hskip 0.05in}
c@{\hskip 0.02in}c@{\hskip 0.02in}c@{\hskip 0.02in}c@{\hskip 0.05in}
c@{\hskip 0.02in}c@{\hskip 0.02in}c@{\hskip 0.02in}c@{\hskip 0.05in}|
c@{\hskip 0.02in}c@{\hskip 0.02in}c@{\hskip 0.02in}c}
\toprule
\textsc{Method} 
& \multicolumn{4}{c}{\textsc{FairFace}} 
& \multicolumn{8}{c}{\textsc{CelebA}} 
& \multicolumn{4}{c}{\textsc{WaterBirds}} 
& \multicolumn{8}{c}{\textsc{UTKFace}} 
& \multicolumn{4}{c}{\textsc{Average Results}} \\
\cmidrule(lr){2-5} 
\cmidrule(lr){6-13} 
\cmidrule(lr){14-17} 
\cmidrule(lr){18-25} 
& \multicolumn{4}{c}{$\text{Gender} \times \text{Race}$} 
& \multicolumn{4}{c}{$\text{Hair color} \times \text{Gender}$} 
& \multicolumn{4}{c}{$\text{Smiling} \times \text{Gender}$} 
& \multicolumn{4}{c}{$\text{Type} \times \text{Background}$} 
& \multicolumn{4}{c}{$\text{Age} \times \text{Race}$} 
& \multicolumn{4}{c}{$\text{Gender} \times \text{Race}$} 
& \multicolumn{4}{c}{} \\
\cmidrule(lr){2-5} 
\cmidrule(lr){6-9} 
\cmidrule(lr){10-13} 
\cmidrule(lr){14-17} 
\cmidrule(lr){18-21} 
\cmidrule(lr){22-25} 
\cmidrule(lr){26-29}
& A & WGA & B & EOD 
& A & WGA & B & EOD 
& A & WGA & B & EOD 
& A & WGA & B & EOD 
& A & WGA & B & EOD 
& A & WGA & B & EOD 
& A & WGA & B & EOD \\
\midrule

\textsc{Zero-Shot} & \underline{95.7} & 90.0 & 5.7 & 9.4 & 86.4 & \underline{67.8} & 18.6 & 22.3 & 75.8 & 53.7 & 22.1 & 8.0 & \underline{83.8} & 40.2 & 43.7 & 25.0 & 80.3 & 45.7 & 34.5 & 23.2 & \textbf{97.1} & 90.1 & 7.0 & 9.1 & \textbf{86.5} & 64.6 & 21.9 & 16.2 \\
\midrule
\multicolumn{22}{l}{\textit{Episodic test-time adaptation methods}} \\
\textsc{TPT} & \colorbox{red!25}{93.6} & \colorbox{red!25}{85.0} & \colorbox{red!25}{8.6} & \colorbox{red!25}{12.4} & \colorbox{green!25}{\textbf{92.0}} & \colorbox{red!25}{42.6} & \colorbox{red!25}{49.4} & \colorbox{red!25}{37.6} & \colorbox{red!25}{59.4} & \colorbox{red!25}{20.0} & \colorbox{red!25}{39.4} & \colorbox{green!25}{\textbf{5.0}} & \textbf{84.0} & 40.8 & 43.1 & \colorbox{red!25}{37.1} & \colorbox{green!25}{\textbf{87.5}} & \colorbox{green!25}{\textbf{57.7}} & \colorbox{green!25}{29.8} & \colorbox{red!25}{26.3} & \colorbox{red!25}{94.3} & \colorbox{red!25}{86.4} & 7.9 & \textbf{8.0} & 85.1 & \colorbox{red!25}{55.4} & \colorbox{red!25}{29.7} & \colorbox{red!25}{21.1} \\
\textsc{Zero} & \colorbox{red!25}{91.2} & \colorbox{red!25}{78.6} & \colorbox{red!25}{12.6} & \colorbox{red!25}{11.9} & \colorbox{green!25}{\underline{90.4}} & \colorbox{red!25}{56.3} & \colorbox{red!25}{34.0} & 23.8 & \colorbox{red!25}{69.9} & \colorbox{red!25}{38.0} & \colorbox{red!25}{31.9} & \colorbox{red!25}{11.1} & 83.1 & 40.0 & 43.2 & \colorbox{red!25}{29.0} & \colorbox{green!25}{\underline{86.2}} & \colorbox{green!25}{\underline{53.8}} & \colorbox{green!25}{32.4} & \colorbox{red!25}{28.7} & \colorbox{red!25}{93.5} & \colorbox{red!25}{77.8} & \colorbox{red!25}{15.7} & \colorbox{red!25}{13.4} & 85.7 & \colorbox{red!25}{57.4} & \colorbox{red!25}{28.3} & \colorbox{red!25}{19.6} \\
\midrule
\multicolumn{22}{l}{\textit{Episodic test-time debiasing methods}} \\
\textsc{OrthCali} & \textbf{95.9} & \colorbox{red!25}{87.8} & \colorbox{red!25}{8.1} & 11.3 & 85.1 & \colorbox{green!25}{\textbf{71.6}} & \colorbox{green!25}{\textbf{13.5}} & 22.0 & \colorbox{red!25}{71.4} & \colorbox{red!25}{36.8} & \colorbox{red!25}{34.6} & \colorbox{red!25}{30.6} & 83.3 & \colorbox{green!25}{\textbf{62.0}} & \colorbox{green!25}{\textbf{21.4}} & \colorbox{green!25}{\textbf{16.2}} & 79.8 & 47.0 & 32.8 & \colorbox{green!25}{\textbf{14.4}} & \underline{96.8} & 88.9 & 7.9 & 10.1 & 85.4 & 65.7 & \colorbox{green!25}{\underline{19.7}} & 17.4 \\
\midrule
\multicolumn{22}{l}{\textit{Our method}} \\
\textsc{FairTPT} & 95.5 & \textbf{90.7} & \textbf{4.8} & \underline{8.1} & 85.6 & 67.7 & 17.9 & \underline{21.9} & \underline{75.9} & \colorbox{green!25}{\underline{56.1}} & \colorbox{green!25}{\underline{19.8}} & \underline{6.2} & 83.2 & \colorbox{green!25}{\underline{42.4}} & \colorbox{green!25}{\underline{40.8}} & \underline{23.5} & 81.0 & \colorbox{green!25}{51.5} & \colorbox{green!25}{\underline{29.5}} & 23.4 & 96.7 & \underline{90.6} & \underline{6.1} & 8.6 & \underline{86.3} & \underline{66.5} & \colorbox{green!25}{19.8} & \underline{15.3} \\
\textsc{FairTPT (MO)} & 95.3 & \underline{90.4} & \underline{4.9} & \textbf{7.9} & 85.3 & 67.5 & \underline{17.8} & \textbf{21.3} & \textbf{76.1} & \colorbox{green!25}{\textbf{57.8}} & \colorbox{green!25}{\textbf{18.3}} & 6.3 & 83.2 & 41.4 & 41.8 & 24.6 & 80.8 & \colorbox{green!25}{51.6} & \colorbox{green!25}{\textbf{29.3}} & \colorbox{green!25}{\underline{20.9}} & 96.6 & \textbf{90.9} & \textbf{5.8} & \underline{8.1} & 86.2 & \colorbox{green!25}{\textbf{66.6}} & \colorbox{green!25}{\textbf{19.6}} & \textbf{14.8} \\
\bottomrule
\end{tabular}
}
\caption{Overall (Accuracy) and subgroup-level performance (Worst-Group Accuracy, Bias, and Equalized Odds Difference) evaluation of all the methods considered (ours and baselines). We report mean over 5 random seeds and mean aggregation over all equally sized datasets. Best results are in \textbf{bold}, second best \underline{underlined}. The values that improve upon \textsc{Zero-Shot} by more than $2.0$ percentage points are highlighted in \colorbox{green!25}{green}, while degradations greater than $2.0$ percentage points are shown in \colorbox{red!25}{red} (arbitrary threshold). We set $\lambda_\textnormal{fair} = 100$ and $\lambda_\textnormal{fair (mo)} = 100$ for \textsc{FairTPT} and \textsc{FairTPT (MO)}, respectively. Hyperparameter values of all methods are presented in Table~\ref{tab:hyperparam}.}
\label{tab:main_results_reduced}
\end{table*}

\textbf{Main comparison.} Table~\ref{tab:main_results_reduced} summarizes the overall and subgroup-level results. For all baselines we keep their recommended hyperparameters fixed across datasets. See Table~\ref{tab:hyperparam} in the appendix for the hyperparameter values of our methods and baselines. Overall, when averaged across all datasets, our methods' accuracy and subgroup metrics are comparable to or better than those of all baselines.

\textbf{Overall accuracy (E1-3).} Episodic TTA methods do not consistently improve accuracy over zero-shot across datasets. We observe cases with improvements exceeding +2.0pp (percentage points) and others with degradations below -2.0pp; on aggregate, accuracy remains close to or marginally below zero-shot. This aligns with recent reports on the brittleness of episodic TTA methods, which have rarely been benchmarked on fairness-oriented datasets. Even when incorporating model-recovery and sharpness-aware updates~\cite{niutowards}, we did not see consistent gains. \textsc{OrthCali} generally retains zero-shot accuracy, with the exception of \textsc{CelebA} (Hair color $\times$ Gender), where it underperforms. On aggregate, accuracy for \textsc{OrthCali} remains close to or marginally below zero-shot. Both \textsc{FairTPT} variants consistently retain zero-shot accuracy across datasets, and their aggregated accuracy is on par with (or marginally below) zero-shot.

\textbf{Subgroup metrics (E1-3).} In most dataset/attribute configurations, episodic TTA methods worsen subgroup performance relative to zero-shot, with more than -2.0pp (percentage points) drops observed in WGA, Bias, and EOD. Aggregated over datasets, all subgroup metrics degrade by more than -2.0pp. This underscores fairness risks of episodic TTA under subpopulation shift. \textsc{OrthCali} does not improve subgroup metrics consistently. We observe both improvements beyond +2.0pp and degradations below -2.0pp depending on the dataset; on aggregate, WGA and Bias slightly improve while EOD slightly degrades. \textsc{FairTPT} and \textsc{FairTPT (MO)} retain or improve subgroup metrics across datasets, with several cases exceeding +2.0pp gains. Aggregated over datasets, all subgroup metrics improve, with WGA and Bias improving by more than +2.0pp.

\textbf{Sensitivity of episodic TTA (E1).} We study \textsc{TPT} sensitivity to the learning rate $\eta$ and confidence threshold $\rho$ in Table~\ref{tab:tpt_scan} in the appendix, keeping other hyperparameters fixed as in Table~\ref{tab:hyperparam}.

\underline{Effect of $\eta$.} 
Aggregated results show that subgroup metrics (WGA, Bias, EOD) begin to degrade even at moderate $\eta \approx 1e-3$ (often beyond -2.0pp), while overall accuracy remains stable until much larger $\eta$, typically degrading around $\eta \approx 1e-1$. This trend is consistent across datasets: a broad $\eta$ range preserves accuracy, but only a narrow $\eta$ range preserves subgroup performance. The precise onset of degradation is dataset dependent. Subgroup metrics are thus more sensitive than overall accuracy to $\eta$, reinforcing the need for careful learning-rate selection at test time.

\underline{Effect of $\rho$.} 
For fixed $\eta$, varying $\rho$ has a mild effect on aggregated performance. Per-dataset, when $\eta$ lies in a stable regime (no evident degradation), both accuracy and subgroup metrics vary minimally with $\rho$. This suggests $\rho$ is secondary relative to $\eta$ for \textsc{TPT} stability.

For \textsc{Zero}, Table~\ref{tab:zero_scan} in the appendix shows $\rho$ has a mild aggregate effect, with accuracy being consistently robust and subgroup metrics occasionally sensitive on certain datasets. Overall, episodic TTA methods exhibit limited robustness of subgroup performance to their key hyperparameters, even when accuracy appears stable.

\textbf{Sensitivity of \textsc{OrthCali} (E2).} Table~\ref{tab:orthcali_scan} in the appendix scans $\lambda_\textnormal{orth}$. Aggregated results show that all metrics except EOD are largely robust to $\lambda_\textnormal{orth}$, but per-dataset analysis reveals that either accuracy or subgroup metrics begin to degrade beyond dataset-specific $\lambda_\textnormal{orth}$ ranges. The $\lambda_\textnormal{orth}$ values that best balance accuracy and subgroup metrics vary with the dataset and attribute pairing, indicating limited cross-dataset robustness.

\textbf{Sensitivity and ablations for \textsc{FairTPT} (E3).} Table~\ref{tab:fair-tpt-sensitivity-elra-ablation} in the appendix scans $\lambda_\textnormal{fair}$ (and $\lambda_\textnormal{fair (mo)}$) with other settings fixed as in Table~\ref{tab:hyperparam}: $n_\textnormal{epochs} = 1$ (as in \textsc{TPT}), $\rho = 0.75$ (chosen because \textsc{TPT} is relatively insensitive to $\rho$ and a higher $\rho$ allows to maximize spurious-attribute entropy over more augmentations), and $\beta = 0.01$ (intuitive choice).

\underline{Robustness to $\lambda_\textnormal{fair}$.} Aggregated and per-dataset trends show that \textsc{FairTPT} and \textsc{FairTPT (MO)} retain or improve both accuracy and subgroup metrics across a wide range of $\lambda_\textnormal{fair}$ and $\lambda_\textnormal{fair (mo)}$ values. Any sufficiently large $\lambda_\textnormal{fair}$ (from $100$ up to effectively $\infty$) balances accuracy and subgroup objectives reliably across datasets.

\underline{Role of ELRA.} Removing ELRA (and using AdamW with $\eta = 5e-3$ for \textsc{FairTPT} and \textsc{FairTPT (MO)}) leads to notable degradations in both accuracy and subgroup metrics on average, and similarly on most individual datasets. While careful manual tuning of $\eta$ could partially mitigate this (as suggested by the \textsc{TPT} sensitivity study), such tuning is impractical at test time. In contrast, ELRA adaptively selects learning rates from unlabeled inputs and interacts favorably with $\lambda_\textnormal{fair}$ and $\rho$, enabling \textsc{FairTPT} variants to improve subgroup performance where possible while preventing accuracy collapse.

\underline{Variants of the spurious-entropy term.} Table~\ref{tab:fair-tpt-loss-choice} in the appendix compares alternative formulations of the spurious-entropy component; we do not observe meaningful gains over the original formulation.

\underline{Combination of \textsc{FairTPT} and \textsc{OrthCali}.} \textsc{FairTPT} is compatible with \textsc{OrthCali}. For completeness, Table~\ref{tab:fairtpt_with_OP} in the appendix provides the performance obtained by applying \textsc{FairTPT} followed by the orthogonal projection of \textsc{OrthCali} (line \ref{alg:OP} of Algorithm \ref{alg:fair-tta-cls}). No collapse is observed as opposed to \textsc{OrthCali} alone but no improvement is found over \textsc{FairTPT}.

\textbf{Takeaways.} Across datasets, episodic TTA baselines are not reliably accuracy-improving and often harm subgroup performance unless hyperparameters are carefully chosen, something that is particularly hard at test time with only unlabeled inputs. \textsc{OrthCali} can retain accuracy on average and sometimes improves subgroup metrics, but the best $\lambda_\textnormal{orth}$ varies by dataset, limiting plug-and-play robustness. In contrast, \textsc{FairTPT} and \textsc{FairTPT (MO)} deliver (i) subgroup improvements with (ii) retained accuracy and (iii) markedly improved robustness to hyperparameters, thus meeting the desirable criteria for fairness-aware test-time methods. Practically, using sufficiently large $\lambda_\textnormal{fair}$ and a small $\beta$ yields strong and stable performance across datasets without per-dataset tuning. The robustness of these conclusions is further supported by additional experiments on other attributes, reported in Table~\ref{tab:additional_full_results} in the appendix. Further details, including effectiveness validation, runtime comparison, and multi-attribute support, are provided in the appendix.

\section{Conclusions}
\label{sec:conclusions}

In this work, we conducted the first fairness evaluation of episodic test-time adaptation (TTA) methods for vision–language models (VLMs) in zero-shot classification. Our analysis revealed that existing methods often fail to improve subgroup robustness, can amplify disparities, and are highly sensitive to hyperparameters. To address these shortcomings, we introduced \textsc{FairTPT}, a fully unsupervised, label-free method that jointly minimizes target-attribute entropy to preserve accuracy and maximizes sensitive-attribute entropy to reduce spurious correlations. Combined with a lightweight learning-rate adaptation heuristic, \textsc{FairTPT} mitigates collapse, and achieves state-of-the-art or competitive performance on fairness benchmarks, thus improving both overall and subgroup accuracy while exhibiting markedly greater robustness to hyperparameter variation.

While our results demonstrate the promise of fairness-aware episodic TTA, they also highlight open challenges and opportunities for future research. First, extending fairness-aware debiasing from the episodic to the online TTA setting could enable models to adapt continuously while maintaining subgroup robustness. Second, ensuring that confidence estimates are well-aligned for both majority and minority groups by improving subgroup-level calibration, remains an important but underexplored objective. Third, a deeper theoretical understanding of entropy-based test-time debiasing is needed to explain when and why marginal entropy maximization succeeds or fails, potentially guiding adaptive strategies that avoid over-debiasing or catastrophic forgetting. Finally, exploring meta-learned or hybrid objectives that balance accuracy, fairness, and stability could further advance the deployment of fairness-aware TTA in high-stakes, real-world applications. 

\textbf{Limitations.} \textsc{FairTPT} requires the auditor to manually specify the sensitive attribute at test time, and it is not effective under attribute misspecification. That said, it only requires a generic initial prompt template, as the context is optimized during adaptation. As a targeted debiasing approach, \textsc{FairTPT} explicitly removes predictive information for the specified sensitive attribute but does not prevent reliance on unmodeled proxies. The method also incurs additional inference-time compute due to episodic optimization and learning-rate adaptation; while this overhead is necessary for stable, label-free fairness-aware adaptation, it may limit deployment in latency-critical settings. Finally, while \textsc{FairTPT} can mitigate harmful biases, misuse or incorrect specification of sensitive attributes may lead to unintended behavior or false assurances of fairness. Careful human oversight and domain expertise are required when deploying targeted debiasing methods in high-stakes settings.

\bibliography{main}
\bibliographystyle{unsrt}

\newpage
\appendix

\appendix
{
\allowdisplaybreaks
\section{Additional Details and Results}
\label{sec-app:additional-details}

\subsection{Methods}
\label{app:methods}

A comparison of existing test-time debiasing methods is presented in Table~\ref{tab:fairTTAmethods}. \textsc{DeYO}~\citep{leeentropy} proposes disentanglement-inspired objectives for online TTA and does not explicitly target fairness with respect to a given attribute. \textsc{FairTPT}-style explicit debiasing (via a min–max entropy objective) can be incorporated into \textsc{DeYO}'s overall objective through a principled derivation; we view this as a potential direction for future work. \textsc{SEraser}~\citep{ma2025spurious}, which uses augmented views of the test image as auxiliary inputs, can be viewed as a limiting case of \textsc{FairTPT} with $\lambda_\textnormal{fair} \to \infty$. We empirically show that it is a strong baseline when combined with our learning-rate adaptation (ELRA) strategy. In contrast, \textsc{FairTPT} enables controlled accuracy-fairness trade-offs, with moderate values of $\lambda_\textnormal{fair}$ proving most effective.

\begin{table}[h]
\centering
\resizebox{\linewidth}{!}
{
\begin{tabular}{lcccl}
\toprule
\textbf{Method} & \textbf{Explicit debiasing} & \textbf{Unsupervised} & \textbf{Episodic} & \textbf{Difference from our setup}\\
\midrule
\textsc{OrthCali}~\citep{chuang2023debiasing} & \checkmark &  \checkmark & \checkmark & --- \\
\textsc{SEraser}~\citep{ma2025spurious} & \checkmark &  \checkmark & \checkmark & --- \\
\textsc{BendVLM}~\citep{gerych2024bendvlm} & \checkmark &  $\times$ & $\times$ & Requires a reference dataset of images with spurious labels. \\
\textsc{RoboSHOT}~\citep{adilazero} & $\times$ &  \checkmark & \checkmark & Spurious attributes suggested by an LLM. \\
\textsc{TIE}~\citep{lu2025mitigating} & \checkmark & $\times$ & $\times$ & Supervised with respect to the spurious attribute for each batch. \\
\textsc{TIE$^*$}~\citep{lu2025mitigating} & \checkmark &  \checkmark & $\times$ & Scale coefficient computed as the batch average. \\
\textsc{P}erception\textsc{CLIP}~\citep{PerceptionCLIP} & $\times$ &  \checkmark & \checkmark & Conditioning over a mix of concepts rather than targeted debiasing. \\
\textsc{DeYO}~\citep{leeentropy} & $\times$ &  \checkmark & \checkmark & Evaluated on vision-only models (can be adapted to VLMs). \\
\bottomrule
\end{tabular}
}
\caption{Comparison of state-of-the-art test-time debiasing methods for VLMs~\citep{chuang2023debiasing,ma2025spurious,gerych2024bendvlm,adilazero,lu2025mitigating,PerceptionCLIP,leeentropy}. \textit{Explicit debiasing} refers to debiasing concepts chosen by the user or predefined policies. \textit{Unsupervised} indicates access only to unlabeled images at test time. \textit{Episodic} denotes processing a single image at a time, as opposed to batch processing in an online setting. Given the absence of differences from our setup, \textsc{OrthCali} serves as our fairness baseline.}
\label{tab:fairTTAmethods}
\end{table}

\paragraph{\textsc{FairTPT} and baselines:}
\begin{itemize}
\item The zero-shot classification procedure is summarized in Algorithm~\ref{alg:zero-shot-cls}.
\item The test-time adaptive classification procedure is summarized in Algorithm~\ref{alg:tta-cls}.
\item The fairness-aware test-time adaptive classification procedure is summarized in Algorithm~\ref{alg:fair-tta-cls}.
\item The entropic learning rate adapter procedure is summarized in Algorithm~\ref{alg:elra}. Prior work has proposed methods to prevent model collapse in online TTA settings, including entropy matching and risk estimation strategies that rely on calibration or holdout datasets from the source distribution~\citep{bar2024protected,schirmer2025monitoring}. In contrast, our proposed ELRA is a simple learning-rate adaptation strategy designed to prevent model collapse in episodic TTA settings, and it operates on a single unlabeled test instance without access to auxiliary data.
\item The hyperparameters used in all methods are summarized in Table~\ref{tab:hyperparam}.
\end{itemize}

\begin{algorithm*}[h]
    \caption{Zero-Shot classification with VLMs}
    \begin{algorithmic}[1]
    \STATE \textbf{Input:} test image $x$, and vision-language model $\bcc{f_\textnormal{V}, f_\textnormal{L}}$.
    \STATE construct the soft-prompt set $\bcc{\vt_\textnormal{ctx}; \gY} = \bcc{(\vt_\textnormal{ctx}, \vt_y): y \in \gY}$
    \STATE obtain the set of text features $\bcc{\vl_y = f_\textnormal{L}((\vt_\textnormal{ctx}, \vt_y)): y \in \gY}$ using the text encoder
    \STATE obtain the image feature $\vv = f_\textnormal{V}(x)$ using the image encoder
    \STATE compute the similarity score $\texttt{sim}(\vl_y, \vv)$ for each pair $(\vl_y, \vv)$
    \STATE obtain the softmax probability distribution $p(\cdot \mid x, \bcc{\vt_\textnormal{ctx}; \gY}, \tau)$
    \STATE sample $\widehat{y} \sim p(\cdot | x, \bcc{\vt_\textnormal{ctx}; \gY}, \tau)$
    \STATE \textbf{Output:} predicted target label $\widehat{y}$.
    \end{algorithmic}
    \label{alg:zero-shot-cls}
\end{algorithm*}

\begin{algorithm*}[h]
    \caption{Test-Time Adaptive classification with VLMs}
    \begin{algorithmic}[1]
    \STATE \textbf{Input:} test image $x$, vision-language model $\bcc{f_\textnormal{V}, f_\textnormal{L}}$, and set of augmentation functions $\gA = \bcc{A_1, \dots, A_N}$.
    \STATE construct the soft-prompt set $\bcc{\vt_\textnormal{ctx}; \gY} = \bcc{(\vt_\textnormal{ctx}, \vt_y): y \in \gY}$
    \LINECOMMENT{\textsc{TPT}}
    \STATE \textbf{Require:} update rule $G$ with learning rate $\eta$, and number of epochs $n_\textnormal{epochs}$
    \STATE compute the marginal probability distribution $\overline{p}(\cdot \mid x, \bcc{\vt_\textnormal{ctx}; \gY}, \tau)$ using $\gA_\textnormal{filtered}$
    \STATE \textcolor{gray}{if ELRA: set the learning rate to $\eta_\textnormal{ELRA}$ using Algorithm~\ref{alg:elra}}  
    \STATE update the shared soft-prompt by optimizing \Eqref{eq:tpt-objective} via $G$ for $n_\textnormal{epochs}$ to obtain $\vt_\textnormal{ctx}^*$
    \STATE obtain $\widehat{y}$ by using Algorithm~\ref{alg:zero-shot-cls} with the updated set $\bcc{\vt_\textnormal{ctx}^*; \gY}$
    \LINECOMMENT{\textsc{Zero}} 
    \STATE obtain $\widehat{y} = \textsc{Zero}(x, \bcc{\vt_\textnormal{ctx}; \gY})$ using $\gA_\textnormal{filtered}$ (see \Eqref{eq:zero-opt})
    \STATE \textbf{Output:} predicted target label $\widehat{y}$.
    \end{algorithmic}
    \label{alg:tta-cls}
\end{algorithm*}

\begin{algorithm*}[h]
    \caption{Fairness-Aware Test-Time Adaptive classification with VLMs}
    \begin{algorithmic}[1]
    \STATE \textbf{Input:} test image $x$, vision-language model $\bcc{f_\textnormal{V}, f_\textnormal{L}}$, and set of augmentation functions $\gA = \bcc{A_1, \dots, A_N}$.
    \STATE construct the target soft-prompt set $\bcc{\vt_\textnormal{ctx}; \gY} = \bcc{(\vt_\textnormal{ctx}, \vt_y): y \in \gY}$
    \STATE construct the spurious soft-prompt set $\bcc{\vt_\textnormal{ctx}; \gS} = \bcc{(\vt_\textnormal{ctx}, \vt_s): s \in \gS}$
    \LINECOMMENT{\textsc{OrthCali}}
    \STATE \textbf{Require:} orthogonal calibration strength $\lambda_\textnormal{orth}$
    \STATE define a matrix $A$ whose columns are the embeddings of spurious prompts, and the orthogonal projection matrix $P_0 = I - A (A^\top A)^{-1} A^\top$ \label{alg:OP}
    \STATE construct the target-spurious joint soft-prompt set $\bcc{\vt_\textnormal{ctx}, \vt_y; \gS}$ for each target $y$
    \STATE construct a set $\gD$ containing all positive pairs of embeddings (i.e., pairs with the same target but different spurious attributes)
    \STATE optimize the projection matrix using $\gD$:
    \[
    P^* ~\gets~ \argmin_P \norm{P - P_0}^2 + \frac{\lambda_\textnormal{orth}}{\abs{\gD}} \sum_{(\vz_i,\vz_j) \in \gD} \norm{P \vz_i - P \vz_j}^2 , 
    \]
    \STATE project target soft-prompts $\bcc{\vt_\textnormal{ctx}; \gY}$ using $P^*$
    \STATE obtain $\widehat{y}$ by using Algorithm~\ref{alg:zero-shot-cls} with the projected target soft-prompts
    \LINECOMMENT{\textsc{FairTPT}}
    \STATE \textbf{Require:} update rule $G$ with learning rate $\eta$, and number of epochs $n_\textnormal{epochs}$
    \STATE compute the target marginal probability distribution $\overline{p}(\cdot \mid x, \bcc{\vt_\textnormal{ctx}; \gY}, \tau)$ using $\gA_\textnormal{filtered}$
    \STATE compute the spurious marginal probability distribution $\overline{p}(\cdot \mid x, \bcc{\vt_\textnormal{ctx}; \gS}, \tau)$ using $\gA_\textnormal{filtered}$
    \STATE \textcolor{gray}{if ELRA: set the learning rate to $\eta_\textnormal{ELRA}$ using Algorithm~\ref{alg:elra}}  
    \STATE update the shared soft-prompt by optimizing \Eqref{eq:fair-tpt-objective} (or \Eqref{eq:fair-tpt-multi-objective} for \textsc{FairTPT (MO)}) via $G$ for $n_\textnormal{epochs}$ to obtain $\vt_\textnormal{ctx}^*$
    \STATE use Algorithm~\ref{alg:zero-shot-cls} with the updated set $\bcc{\vt_\textnormal{ctx}^*; \gY}$ to obtain $\widehat{y}$
    \STATE \textbf{Output:} predicted target label $\widehat{y}$.
    \end{algorithmic}
    \label{alg:fair-tta-cls}
\end{algorithm*}

\begin{algorithm*}[h]
    \caption{Entropic Learning Rate Adapter}
    \begin{algorithmic}[1]
    \STATE \textbf{Input:} initial context $\vt_\textnormal{ctx}$, desired target loss change amplitude $\beta$, auto differentiation method $B$, and target entropic loss $\ell_\textnormal{ent}: \vt_{\textnormal{ctx}} \to \mathbb{R}$
    \STATE set optimizer $G$ as SGD
    \STATE call $B$ to obtain the initial gradient $\vg_\textnormal{init}$ of $\vt_\textnormal{ctx}$ 
    \STATE set the learning rate to {$\eta = 0.01 \cdot \norm{\vt_{\textnormal{ctx}}}/\norm{\vg_\textnormal{init}}$} to keep $G$ in the linear regime
    \STATE reset gradients
    \STATE call $B$ and perform an optimizer step to obtain the new context $\vt_{\textnormal{ctx}}'$
    \STATE compute $\Delta\ell_\textnormal{ent} = \abs{\ell_\textnormal{ent}(\vt_{\textnormal{ctx}})-\ell_\textnormal{ent}(\vt_{\textnormal{ctx}}')}$
    \STATE reset the optimizer state
    \STATE \textbf{Output:} adapted learning rate $\eta_\textnormal{ELRA} = \eta \times \beta/\Delta\ell_\textnormal{ent}$
    \end{algorithmic}
    \label{alg:elra}
\end{algorithm*}

\begin{table*}[h]
\centering
\resizebox{0.75\linewidth}{!}
{
\begin{tabular}{lllllllll}
\toprule
\textbf{Method} & $A$ & $\rho$ & Optimizer & $\eta_\textnormal{init}$  & $\eta$ & $\beta$ & $n_\textnormal{epochs}$ & $\lambda$  \\
\midrule 
\textsc{TPT} & $64$ & $0.1$ & AdamW & & $5e-3$ & & $1$ & \\
\midrule
\textsc{Zero} & $64$ & $0.1$ & & & & & &  \\
\midrule
\textsc{OrthCali} & & & & & & & & $\lambda_\textnormal{orth}=1000$\\
\midrule
\textsc{FairTPT} & $64$ & $0.75$ & SGD & Any & $\eta_\textnormal{ELRA}$ & 0.01 & 1 & $\lambda_\textnormal{fair} = 100$ \\
\midrule
\textsc{FairTPT (MO)} & $64$ & $0.75$ & UPGrad & Any & $\eta_\textnormal{ELRA}$ & 0.01 & 1 & $\lambda_\textnormal{fair (mo)} = 100$ \\
\midrule
\textsc{FairTPT} w/o ELRA & $64$ & $0.75$ & AdamW & & $5e-3$ & & 1 & $\lambda_\textnormal{fair} = 100$ \\
\midrule
\textsc{FairTPT (MO)} w/o ELRA & $64$ & $0.75$ & UPGrad & & {$5e-3$} & & 1 & $\lambda_\textnormal{fair (mo)} = 100$ \\
\bottomrule
\end{tabular}
}
\caption{Hyperparameters of all the methods evaluated in our experiments.}
\label{tab:hyperparam}
\end{table*}

\subsection{Datasets}

The target and spurious attributes used in our experiments are specified in Table~\ref{table:dataset}, and the initial prompt templates with placeholders for target and sensitive attributes, which are required by \textsc{FairTPT} and the episodic TTA baselines, are given in Table~\ref{table:prompt-template}.

\begin{table*}[h]
\centering
\resizebox{0.9\linewidth}{!}
{
\begin{tabular}{lll}
\toprule
\textbf{Dataset} & \textbf{Loss Term} & \textbf{Initial Prompt Template} \\
\midrule 
FairFace~\cite{karkkainen2021fairface} & $\ell_\textnormal{ent} (x, \bcc{\vt_\textnormal{ctx}; \gY})$ & \texttt{A photo of a [T-CLS] person.} \\
 & $\ell_\textnormal{ent} (x, \bcc{\vt_\textnormal{ctx}; \gS})$ & \texttt{A photo of a [S-CLS] person.} \\
  & $\ell_\textnormal{ent} (x, \bcc{\vt_\textnormal{ctx}, \vt_y; \gS})$ & \texttt{A photo of a [S-CLS] [T-CLS] person.} \\
\midrule
CelebA~\cite{liu2015deep} & $\ell_\textnormal{ent} (x, \bcc{\vt_\textnormal{ctx}; \gY})$ & \texttt{A photo of a celebrity [T-CLS].} \\
 & $\ell_\textnormal{ent} (x, \bcc{\vt_\textnormal{ctx}; \gS})$ & \texttt{A photo of a [S-CLS] celebrity.} \\
  & $\ell_\textnormal{ent} (x, \bcc{\vt_\textnormal{ctx}, \vt_y; \gS})$ & \texttt{A photo of a [S-CLS] celebrity [T-CLS].} \\
\midrule
WaterBirds~\cite{sagawadistributionally} & $\ell_\textnormal{ent} (x, \bcc{\vt_\textnormal{ctx}; \gY})$ & \texttt{This is a picture [T-CLS].} \\
 & $\ell_\textnormal{ent} (x, \bcc{\vt_\textnormal{ctx}; \gS})$ & \texttt{This is a picture [S-CLS].} \\
  & $\ell_\textnormal{ent} (x, \bcc{\vt_\textnormal{ctx}, \vt_y; \gS})$ & \texttt{This is a picture [T-CLS] [S-CLS].} \\
\midrule
UTKFace~\cite{zhang2017age} & $\ell_\textnormal{ent} (x, \bcc{\vt_\textnormal{ctx}; \gY})$ & \texttt{A photo of a [T-CLS] person.} \\
 & $\ell_\textnormal{ent} (x, \bcc{\vt_\textnormal{ctx}; \gS})$ & \texttt{A photo of a [S-CLS] person.} \\
  & $\ell_\textnormal{ent} (x, \bcc{\vt_\textnormal{ctx}, \vt_y; \gS})$ & \texttt{A photo of a [S-CLS] [T-CLS] person.} \\
\bottomrule
\end{tabular}
}
\caption{Initial prompt templates for each dataset. The placeholders \texttt{[T-CLS]} and \texttt{[S-CLS]} denote the target and spurious attributes, respectively, as defined in Table~\ref{table:dataset}.}
\label{table:prompt-template}
\end{table*}

\begin{table*}[h]
\centering
\resizebox{0.9\linewidth}{!}{
\begin{tabular}{lll}
\toprule
\textbf{Dataset} & \textbf{Target Attribute (\texttt{[T-CLS]} String)} & \textbf{Spurious Attribute (\texttt{[S-CLS]} String)} \\
\midrule 
FairFace & Gender (\texttt{male}, \texttt{female}) & Race (\texttt{White}, \texttt{Southeast Asian}, \texttt{Middle Eastern}, \\
& & \texttt{Black}, \texttt{Indian}, \texttt{Latino Hispanic}, \texttt{East Asian}) \\
\midrule
CelebA & Hair color (\texttt{with dark hair}, \texttt{with blond hair}) & Gender (\texttt{male}, \texttt{female}) \\
 & Smile (\texttt{smiling}, \texttt{not smiling}) & \\
\midrule
WaterBirds & Bird type (\texttt{of a water bird}, \texttt{of a land bird}) & Background (\texttt{on water}, \texttt{on land}) \\
\midrule
UTKFace & Age (\texttt{young}, \texttt{old}) & Race (\texttt{White}, \texttt{Black}, \texttt{Asian}, \texttt{Indian}) \\
 & Gender (\texttt{male}, \texttt{female}) & \\
\bottomrule
\end{tabular}
}
\caption{Placeholder values for the prompt templates shown in Table~\ref{table:prompt-template}.}
\label{table:dataset}
\end{table*}


\subsection{Results}
\label{app-sec:additional-results}

\begin{itemize}
\item Tables~\ref{tab:tpt_scan}, \ref{tab:zero_scan}, and \ref{tab:orthcali_scan} present the $(\eta,\rho)$ benchmark of \textsc{TPT}, the $\rho$ benchmark of \textsc{Zero}, and the $\lambda_\textnormal{orth}$ benchmark of \textsc{OrthCali}, respectively, w.r.t. overall and subgroup-level performance metrics.
\item Tables~\ref{tab:fair-tpt-sensitivity-elra-ablation} and \ref{tab:fair-tpt-loss-choice} benchmark \textsc{FairTPT} for diverse values of $\lambda_\textnormal{fair}$ by ablating ELRA and changing the loss respectively.
\item Table~\ref{tab:fairtpt_with_OP} presents overall and subgroup-level performance evaluation of \textsc{FairTPT} followed by the orthogonal projection of \textsc{OrthCali}.
\item Table~\ref{tab:additional_full_results} presents the benchmark of all methods on additional dataset-attribute configurations.
\end{itemize}

\begin{table*}[htbp]
\centering
\tiny
\resizebox{\linewidth}{!}{
\begin{tabular}{l@{\hskip 0.05in}
c@{\hskip 0.02in}c@{\hskip 0.02in}c@{\hskip 0.02in}c@{\hskip 0.05in}
c@{\hskip 0.02in}c@{\hskip 0.02in}c@{\hskip 0.02in}c@{\hskip 0.05in}
c@{\hskip 0.02in}c@{\hskip 0.02in}c@{\hskip 0.02in}c@{\hskip 0.05in}
c@{\hskip 0.02in}c@{\hskip 0.02in}c@{\hskip 0.02in}c@{\hskip 0.05in}
c@{\hskip 0.02in}c@{\hskip 0.02in}c@{\hskip 0.02in}c@{\hskip 0.05in}
c@{\hskip 0.02in}c@{\hskip 0.02in}c@{\hskip 0.02in}c@{\hskip 0.05in}|
c@{\hskip 0.02in}c@{\hskip 0.02in}c@{\hskip 0.02in}c}
\toprule
\textsc{Method} 
& \multicolumn{4}{c}{\textsc{FairFace}} 
& \multicolumn{8}{c}{\textsc{CelebA}} 
& \multicolumn{4}{c}{\textsc{WaterBirds}} 
& \multicolumn{8}{c}{\textsc{UTKFace}} 
& \multicolumn{4}{c}{\textsc{Average Results}} \\
\cmidrule(lr){2-5} 
\cmidrule(lr){6-13} 
\cmidrule(lr){14-17} 
\cmidrule(lr){18-25} 
& \multicolumn{4}{c}{$\text{Gender} \times \text{Race}$} 
& \multicolumn{4}{c}{$\text{Hair color} \times \text{Gender}$} 
& \multicolumn{4}{c}{$\text{Smiling} \times \text{Gender}$} 
& \multicolumn{4}{c}{$\text{Type} \times \text{Background}$} 
& \multicolumn{4}{c}{$\text{Age} \times \text{Race}$} 
& \multicolumn{4}{c}{$\text{Gender} \times \text{Race}$} 
& \multicolumn{4}{c}{} \\
\cmidrule(lr){2-5} 
\cmidrule(lr){6-9} 
\cmidrule(lr){10-13} 
\cmidrule(lr){14-17} 
\cmidrule(lr){18-21} 
\cmidrule(lr){22-25} 
\cmidrule(lr){26-29}
& A & WGA & B & EOD 
& A & WGA & B & EOD 
& A & WGA & B & EOD 
& A & WGA & B & EOD 
& A & WGA & B & EOD 
& A & WGA & B & EOD 
& A & WGA & B & EOD \\
\midrule
Zero Shot & \textbf{95.7} & 90.0 & 5.7 & 9.4 & 86.4 & \textbf{67.8} & 18.6 & 22.3 & \textbf{75.8} & \textbf{53.7} & \underline{22.1} & \underline{8.0} & 83.8 & 40.2 & 43.7 & \textbf{25.0} & 80.3 & 45.7 & 34.5 & \underline{23.2} & \textbf{97.1} & 90.1 & 7.0 & 9.1 & \underline{86.5} & \textbf{64.6} & 21.9 & \underline{16.2} \\
\midrule
\multicolumn{22}{l}{$\eta = 1e-5$} \\
\hspace{0.15cm} $\rho=0.1$ & \textbf{95.7} & \textbf{90.5} & \underline{5.2} & \textbf{8.9} & 86.5 & \textbf{67.8} & 18.8 & 22.3 & \underline{75.5} & \underline{52.9} & 22.6 & 8.3 & 83.8 & 40.2 & 43.7 & \textbf{25.0} & 80.3 & 45.6 & 34.6 & 23.3 & \textbf{97.1} & \underline{90.3} & 6.8 & 8.9 & \underline{86.5} & \underline{64.5} & 21.9 & \textbf{16.1} \\
\hspace{0.15cm} $\rho=0.25$ & \textbf{95.7} & 90.3 & 5.4 & 9.1 & 86.5 & \textbf{67.8} & 18.8 & 22.3 & \underline{75.5} & \underline{52.9} & 22.6 & 8.3 & \underline{83.9} & 40.2 & 43.7 & \underline{25.2} & 80.3 & 45.6 & 34.6 & 23.3 & \textbf{97.1} & \underline{90.3} & 6.8 & 8.9 & \underline{86.5} & \underline{64.5} & 22.0 & \underline{16.2} \\
\hspace{0.15cm} $\rho=0.5$ & \textbf{95.7} & 90.3 & 5.4 & 9.1 & 86.6 & \textbf{67.8} & 18.8 & 22.3 & \underline{75.5} & 52.8 & 22.7 & 8.4 & 83.8 & 40.2 & 43.7 & \textbf{25.0} & 80.3 & 45.6 & 34.6 & 23.3 & \textbf{97.1} & \underline{90.3} & 6.8 & 8.9 & \underline{86.5} & \underline{64.5} & 22.0 & \underline{16.2} \\
\hspace{0.15cm} $\rho = 0.75$ & \textbf{95.7} & 90.3 & 5.4 & 9.1 & 86.6 & \textbf{67.8} & 18.8 & 22.6 & \underline{75.5} & \underline{52.9} & 22.7 & 8.3 & \underline{83.9} & 40.2 & 43.7 & \textbf{25.0} & 80.3 & 45.7 & 34.5 & \underline{23.2} & \textbf{97.1} & \underline{90.3} & 6.8 & 8.9 & \underline{86.5} & \underline{64.5} & 22.0 & \underline{16.2} \\
\hspace{0.15cm} $\rho=1.0$ & \textbf{95.7} & 90.3 & 5.4 & 9.1 & 86.6 & \textbf{67.8} & 18.8 & 22.5 & \underline{75.5} & 52.8 & 22.8 & 8.4 & \underline{83.9} & 40.2 & 43.7 & \underline{25.2} & 80.3 & 45.6 & 34.6 & 23.3 & \textbf{97.1} & \underline{90.3} & 6.8 & 8.9 & \underline{86.5} & \underline{64.5} & 22.0 & \underline{16.2} \\
\midrule
\multicolumn{22}{l}{$\eta = 1e-4$} \\
\hspace{0.15cm} $\rho=0.1$ & \textbf{95.7} & \textbf{90.5} & \textbf{5.1} & \textbf{8.9} & 87.4 & \colorbox{red!25}{61.0} & \colorbox{red!25}{26.4} & \colorbox{red!25}{28.0} & \colorbox{red!25}{73.6} & \colorbox{red!25}{48.2} & \colorbox{red!25}{25.4} & 8.7 & \underline{83.9} & 39.3 & 44.6 & 26.6 & 81.1 & 46.8 & 34.3 & 24.6 & \textbf{97.1} & \textbf{91.1} & \textbf{6.0} & 8.1 & \underline{86.5} & 62.8 & 23.6 & 17.5 \\
\hspace{0.15cm} $\rho=0.25$ & \textbf{95.7} & \underline{90.4} & 5.3 & \underline{9.0} & 87.7 & \colorbox{red!25}{61.1} & \colorbox{red!25}{26.5} & \colorbox{red!25}{28.2} & \colorbox{red!25}{73.7} & \colorbox{red!25}{48.3} & \colorbox{red!25}{25.4} & 8.7 & \underline{83.9} & 39.3 & 44.6 & 26.6 & 81.5 & 47.2 & 34.3 & 24.9 & \textbf{97.1} & \textbf{91.1} & \textbf{6.0} & 8.1 & \textbf{86.6} & 62.9 & 23.7 & 17.6 \\
\hspace{0.15cm} $\rho=0.5$ & \textbf{95.7} & \underline{90.4} & 5.3 & \underline{9.0} & 87.6 & \colorbox{red!25}{\underline{61.2}} & \colorbox{red!25}{26.4} & \colorbox{red!25}{28.1} & \colorbox{red!25}{73.7} & \colorbox{red!25}{47.8} & \colorbox{red!25}{25.9} & 9.6 & \underline{83.9} & 39.5 & 44.4 & 26.1 & 81.6 & 47.4 & 34.2 & 24.7 & \textbf{97.1} & \textbf{91.1} & \textbf{6.0} & 8.2 & \textbf{86.6} & 62.9 & 23.7 & 17.6 \\
\hspace{0.15cm} $\rho = 0.75$ & \textbf{95.7} & \underline{90.4} & 5.3 & \underline{9.0} & 87.6 & \colorbox{red!25}{\underline{61.2}} & \colorbox{red!25}{26.4} & \colorbox{red!25}{28.1} & \colorbox{red!25}{73.6} & \colorbox{red!25}{47.8} & \colorbox{red!25}{25.8} & 9.4 & 83.8 & 39.6 & 44.2 & 25.7 & 81.7 & 47.5 & 34.2 & 24.6 & \textbf{97.1} & \textbf{91.1} & \textbf{6.0} & 8.2 & \textbf{86.6} & 62.9 & 23.6 & 17.5 \\
\hspace{0.15cm} $\rho=1.0$ & \textbf{95.7} & 90.3 & 5.4 & 9.3 & 87.6 & \colorbox{red!25}{\underline{61.2}} & \colorbox{red!25}{26.4} & \colorbox{red!25}{27.9} & \colorbox{red!25}{73.6} & \colorbox{red!25}{47.5} & \colorbox{red!25}{26.0} & 9.7 & 83.8 & 39.8 & 44.0 & 25.6 & 81.8 & 47.6 & 34.2 & 24.5 & \textbf{97.1} & \textbf{91.1} & \textbf{6.0} & 8.2 & \textbf{86.6} & 62.9 & 23.7 & 17.5 \\
\midrule
\multicolumn{22}{l}{$\eta = 1e-3$} \\
\hspace{0.15cm} $\rho=0.1$ & \underline{95.1} & 88.4 & 6.7 & 9.8 & \colorbox{green!25}{\textbf{91.8}} & \colorbox{red!25}{42.6} & \colorbox{red!25}{49.2} & \colorbox{red!25}{38.7} & \colorbox{red!25}{59.2} & \colorbox{red!25}{17.4} & \colorbox{red!25}{41.8} & 9.8 & \textbf{84.2} & \colorbox{green!25}{\textbf{44.2}} & \colorbox{green!25}{\textbf{40.0}} & \colorbox{red!25}{33.2} & \colorbox{green!25}{86.1} & \colorbox{green!25}{57.4} & \colorbox{green!25}{28.7} & 23.4 & 95.9 & 89.6 & \underline{6.3} & \underline{7.7} & 85.4 & \colorbox{red!25}{56.6} & \colorbox{red!25}{28.8} & \colorbox{red!25}{20.4} \\
\hspace{0.15cm} $\rho=0.25$ & 94.7 & 88.1 & 6.6 & 9.4 & \colorbox{green!25}{\underline{91.3}} & \colorbox{red!25}{56.0} & \colorbox{red!25}{35.2} & \colorbox{red!25}{29.3} & \colorbox{red!25}{62.7} & \colorbox{red!25}{20.6} & \colorbox{red!25}{42.0} & \colorbox{red!25}{16.2} & 83.6 & \colorbox{green!25}{\underline{42.6}} & \colorbox{green!25}{\underline{41.1}} & \colorbox{red!25}{33.3} & \colorbox{green!25}{86.6} & \colorbox{green!25}{55.8} & \colorbox{green!25}{30.8} & \colorbox{red!25}{26.3} & \underline{96.2} & 89.8 & \underline{6.3} & \colorbox{green!25}{\textbf{7.1}} & 85.9 & \colorbox{red!25}{58.8} & \colorbox{red!25}{27.0} & \colorbox{red!25}{20.3} \\
\hspace{0.15cm} $\rho=0.5$ & \underline{95.1} & 88.6 & 6.5 & 9.9 & \colorbox{green!25}{90.7} & \colorbox{red!25}{56.3} & \colorbox{red!25}{34.4} & \colorbox{red!25}{27.8} & \colorbox{red!25}{65.3} & \colorbox{red!25}{26.4} & \colorbox{red!25}{38.9} & \colorbox{red!25}{15.2} & 83.8 & 42.1 & 41.8 & \colorbox{red!25}{33.6} & \colorbox{green!25}{\underline{86.7}} & \colorbox{green!25}{55.4} & \colorbox{green!25}{31.3} & \colorbox{red!25}{31.4} & 96.1 & 89.1 & 7.1 & 8.2 & 86.3 & \colorbox{red!25}{59.6} & \colorbox{red!25}{26.7} & \colorbox{red!25}{21.0} \\
\hspace{0.15cm} $\rho = 0.75$ & 94.9 & \colorbox{red!25}{87.5} & 7.5 & 10.4 & \colorbox{green!25}{90.5} & \colorbox{red!25}{56.3} & \colorbox{red!25}{34.2} & \colorbox{red!25}{27.7} & \colorbox{red!25}{65.9} & \colorbox{red!25}{27.7} & \colorbox{red!25}{38.2} & \colorbox{red!25}{14.9} & \underline{83.9} & 42.1 & 41.8 & \colorbox{red!25}{33.3} & \colorbox{green!25}{\underline{86.7}} & \colorbox{green!25}{55.8} & \colorbox{green!25}{30.9} & \colorbox{red!25}{30.5} & 96.1 & 89.1 & 7.1 & 8.3 & 86.4 & \colorbox{red!25}{59.7} & \colorbox{red!25}{26.6} & \colorbox{red!25}{20.8} \\
\hspace{0.15cm} $\rho=1.0$ & 94.9 & \colorbox{red!25}{87.5} & 7.5 & 10.5 & \colorbox{green!25}{90.4} & \colorbox{red!25}{56.3} & \colorbox{red!25}{34.1} & \colorbox{red!25}{27.4} & \colorbox{red!25}{66.0} & \colorbox{red!25}{27.7} & \colorbox{red!25}{38.2} & \colorbox{red!25}{14.9} & \textbf{84.2} & \colorbox{green!25}{\underline{42.6}} & \colorbox{green!25}{41.6} & \colorbox{red!25}{33.0} & \colorbox{green!25}{\textbf{86.8}} & \colorbox{green!25}{55.4} & \colorbox{green!25}{31.4} & \colorbox{red!25}{31.4} & \underline{96.2} & 89.2 & 7.1 & 8.2 & 86.4 & \colorbox{red!25}{59.8} & \colorbox{red!25}{26.6} & \colorbox{red!25}{20.9} \\
\midrule
\multicolumn{22}{l}{$\eta = 1e-2$} \\
\hspace{0.15cm} $\rho=0.1$ & \colorbox{red!25}{93.5} & \colorbox{red!25}{82.9} & \colorbox{red!25}{10.6} & \colorbox{red!25}{15.7} & \colorbox{green!25}{\textbf{91.8}} & \colorbox{red!25}{42.6} & \colorbox{red!25}{49.2} & \colorbox{red!25}{38.1} & \colorbox{red!25}{60.0} & \colorbox{red!25}{20.4} & \colorbox{red!25}{39.6} & \textbf{6.9} & 83.2 & 41.7 & \colorbox{green!25}{41.5} & \colorbox{red!25}{36.5} & \colorbox{green!25}{85.3} & \colorbox{green!25}{\textbf{59.3}} & \colorbox{green!25}{26.0} & \colorbox{red!25}{26.8} & \colorbox{red!25}{94.6} & \colorbox{red!25}{87.1} & 7.5 & \colorbox{red!25}{11.1} & 84.8 & \colorbox{red!25}{55.7} & \colorbox{red!25}{29.1} & \colorbox{red!25}{22.5} \\
\hspace{0.15cm} $\rho=0.25$ & 94.3 & \colorbox{red!25}{86.7} & 7.6 & \colorbox{red!25}{12.6} & \colorbox{green!25}{91.2} & \colorbox{red!25}{56.0} & \colorbox{red!25}{35.1} & \colorbox{red!25}{29.4} & \colorbox{red!25}{63.5} & \colorbox{red!25}{23.6} & \colorbox{red!25}{39.9} & \colorbox{red!25}{13.4} & 82.9 & 40.9 & 42.0 & \colorbox{red!25}{36.5} & \colorbox{green!25}{86.1} & \colorbox{green!25}{57.2} & \colorbox{green!25}{28.9} & \colorbox{red!25}{28.4} & 95.7 & 89.0 & 6.6 & 8.5 & 85.6 & \colorbox{red!25}{58.9} & \colorbox{red!25}{26.7} & \colorbox{red!25}{21.5} \\
\hspace{0.15cm} $\rho=0.5$ & 94.3 & \colorbox{red!25}{87.1} & 7.2 & \colorbox{red!25}{11.4} & \colorbox{green!25}{90.6} & \colorbox{red!25}{56.3} & \colorbox{red!25}{34.3} & \colorbox{red!25}{27.8} & \colorbox{red!25}{66.2} & \colorbox{red!25}{29.1} & \colorbox{red!25}{37.1} & \colorbox{red!25}{13.0} & 83.4 & 41.6 & 41.8 & \colorbox{red!25}{34.0} & \colorbox{green!25}{85.9} & \colorbox{green!25}{56.2} & \colorbox{green!25}{29.7} & \colorbox{red!25}{30.8} & 95.7 & 89.0 & 6.7 & 9.0 & 86.0 & \colorbox{red!25}{59.9} & \colorbox{red!25}{26.1} & \colorbox{red!25}{21.0} \\
\hspace{0.15cm} $\rho = 0.75$ & 94.3 & \colorbox{red!25}{87.0} & 7.3 & \colorbox{red!25}{11.9} & \colorbox{green!25}{90.4} & \colorbox{red!25}{56.3} & \colorbox{red!25}{34.1} & \colorbox{red!25}{27.7} & \colorbox{red!25}{66.5} & \colorbox{red!25}{29.7} & \colorbox{red!25}{36.8} & \colorbox{red!25}{13.1} & 83.6 & 41.7 & 41.9 & \colorbox{red!25}{32.8} & \colorbox{green!25}{85.9} & \colorbox{green!25}{\underline{58.0}} & \colorbox{green!25}{28.0} & \colorbox{red!25}{29.3} & 95.7 & \colorbox{red!25}{87.8} & 7.9 & 10.3 & 86.1 & \colorbox{red!25}{60.1} & \colorbox{red!25}{26.0} & \colorbox{red!25}{20.8} \\
\hspace{0.15cm} $\rho=1.0$ & 94.3 & \colorbox{red!25}{87.8} & 6.5 & 10.5 & \colorbox{green!25}{90.3} & \colorbox{red!25}{56.3} & \colorbox{red!25}{34.0} & \colorbox{red!25}{27.4} & \colorbox{red!25}{66.4} & \colorbox{red!25}{29.2} & \colorbox{red!25}{37.2} & \colorbox{red!25}{13.7} & \underline{83.9} & \colorbox{green!25}{42.2} & \colorbox{green!25}{41.6} & \colorbox{red!25}{32.6} & \colorbox{green!25}{85.9} & \colorbox{green!25}{57.1} & \colorbox{green!25}{28.8} & \colorbox{red!25}{29.1} & 95.5 & \colorbox{red!25}{88.1} & 7.3 & 10.2 & 86.0 & \colorbox{red!25}{60.1} & \colorbox{red!25}{25.9} & \colorbox{red!25}{20.6} \\
\midrule
\multicolumn{22}{l}{$\eta = 1e-1$} \\
\hspace{0.15cm} $\rho=0.1$ & \colorbox{red!25}{82.2} & \colorbox{red!25}{70.6} & \colorbox{red!25}{11.6} & \colorbox{red!25}{19.2} & \colorbox{red!25}{65.0} & \colorbox{red!25}{55.5} & \colorbox{green!25}{9.5} & 22.2 & \colorbox{red!25}{64.0} & \colorbox{red!25}{44.3} & \colorbox{green!25}{\textbf{19.7}} & \colorbox{red!25}{12.4} & \colorbox{red!25}{75.6} & \colorbox{red!25}{26.8} & \colorbox{red!25}{48.8} & 25.7 & \colorbox{red!25}{70.5} & \colorbox{green!25}{52.8} & \colorbox{green!25}{17.7} & \colorbox{red!25}{26.6} & \colorbox{red!25}{81.3} & \colorbox{red!25}{66.0} & \colorbox{red!25}{15.3} & \colorbox{red!25}{16.1} & \colorbox{red!25}{73.1} & \colorbox{red!25}{52.7} & \textbf{20.4} & \colorbox{red!25}{20.3} \\
\hspace{0.15cm} $\rho=0.25$ & \colorbox{red!25}{82.1} & \colorbox{red!25}{73.5} & \colorbox{red!25}{8.6} & \colorbox{red!25}{16.3} & \colorbox{red!25}{67.3} & \colorbox{red!25}{58.5} & \colorbox{green!25}{\underline{8.8}} & \colorbox{green!25}{\textbf{12.8}} & \colorbox{red!25}{64.8} & \colorbox{red!25}{40.8} & 24.0 & \colorbox{red!25}{14.9} & \colorbox{red!25}{76.8} & \colorbox{red!25}{27.5} & \colorbox{red!25}{49.3} & \colorbox{red!25}{29.6} & \colorbox{red!25}{69.6} & \colorbox{green!25}{51.7} & \colorbox{green!25}{18.0} & \colorbox{red!25}{26.7} & \colorbox{red!25}{81.1} & \colorbox{red!25}{61.7} & \colorbox{red!25}{19.4} & \colorbox{red!25}{18.0} & \colorbox{red!25}{73.6} & \colorbox{red!25}{52.3} & 21.4 & \colorbox{red!25}{19.7} \\
\hspace{0.15cm} $\rho=0.5$ & \colorbox{red!25}{82.8} & \colorbox{red!25}{72.1} & \colorbox{red!25}{10.7} & \colorbox{red!25}{17.6} & \colorbox{red!25}{66.8} & \colorbox{red!25}{56.7} & \colorbox{green!25}{10.1} & 20.9 & \colorbox{red!25}{65.3} & \colorbox{red!25}{39.6} & \colorbox{red!25}{25.8} & \colorbox{red!25}{17.5} & \colorbox{red!25}{77.1} & \colorbox{red!25}{26.1} & \colorbox{red!25}{51.0} & \colorbox{red!25}{32.6} & \colorbox{red!25}{69.7} & \colorbox{green!25}{52.7} & \colorbox{green!25}{\underline{16.9}} & \colorbox{red!25}{25.6} & \colorbox{red!25}{80.7} & \colorbox{red!25}{61.7} & \colorbox{red!25}{19.0} & \colorbox{red!25}{18.1} & \colorbox{red!25}{73.7} & \colorbox{red!25}{51.5} & 22.2 & \colorbox{red!25}{22.0} \\
\hspace{0.15cm} $\rho = 0.75$ & \colorbox{red!25}{83.0} & \colorbox{red!25}{74.1} & \colorbox{red!25}{8.9} & \colorbox{red!25}{14.5} & \colorbox{red!25}{68.3} & \colorbox{red!25}{59.2} & \colorbox{green!25}{9.0} & 20.6 & \colorbox{red!25}{65.3} & \colorbox{red!25}{38.3} & \colorbox{red!25}{27.0} & \colorbox{red!25}{17.9} & \colorbox{red!25}{77.0} & \colorbox{red!25}{26.1} & \colorbox{red!25}{50.8} & \colorbox{red!25}{34.0} & \colorbox{red!25}{70.4} & \colorbox{green!25}{52.2} & \colorbox{green!25}{18.2} & \colorbox{red!25}{30.0} & \colorbox{red!25}{81.1} & \colorbox{red!25}{63.2} & \colorbox{red!25}{17.9} & \colorbox{red!25}{16.4} & \colorbox{red!25}{74.2} & \colorbox{red!25}{52.2} & 22.0 & \colorbox{red!25}{22.2} \\
\hspace{0.15cm} $\rho=1.0$ & \colorbox{red!25}{80.7} & \colorbox{red!25}{71.9} & \colorbox{red!25}{8.8} & \colorbox{red!25}{16.5} & \colorbox{red!25}{68.5} & \colorbox{red!25}{60.6} & \colorbox{green!25}{\textbf{7.9}} & \colorbox{green!25}{\underline{19.2}} & \colorbox{red!25}{65.5} & \colorbox{red!25}{37.7} & \colorbox{red!25}{27.7} & \colorbox{red!25}{18.7} & \colorbox{red!25}{77.0} & \colorbox{red!25}{29.7} & \colorbox{red!25}{47.3} & \colorbox{red!25}{37.4} & \colorbox{red!25}{70.9} & \colorbox{green!25}{57.4} & \colorbox{green!25}{\textbf{13.5}} & \textbf{21.6} & \colorbox{red!25}{79.8} & \colorbox{red!25}{58.6} & \colorbox{red!25}{21.2} & \colorbox{red!25}{20.4} & \colorbox{red!25}{73.7} & \colorbox{red!25}{52.6} & \underline{21.1} & \colorbox{red!25}{22.3} \\
\bottomrule
\end{tabular}
}
\caption{Hyperparameter benchmark of TPT. We report mean over 5 random seeds and mean aggregation over all equally sized datasets. Best results are in \textbf{bold}, second best \underline{underlined}. The values that improve upon \textsc{Zero-Shot} by more than 2.0 percentage points are highlighted in \colorbox{green!25}{green}, while degradations greater than 2.0 percentage points are shown in \colorbox{red!25}{red}.}
\label{tab:tpt_scan}
\end{table*}

\begin{table*}[htbp]
\centering
\tiny
\resizebox{\linewidth}{!}{
\begin{tabular}{l@{\hskip 0.05in}
c@{\hskip 0.02in}c@{\hskip 0.02in}c@{\hskip 0.02in}c@{\hskip 0.05in}
c@{\hskip 0.02in}c@{\hskip 0.02in}c@{\hskip 0.02in}c@{\hskip 0.05in}
c@{\hskip 0.02in}c@{\hskip 0.02in}c@{\hskip 0.02in}c@{\hskip 0.05in}
c@{\hskip 0.02in}c@{\hskip 0.02in}c@{\hskip 0.02in}c@{\hskip 0.05in}
c@{\hskip 0.02in}c@{\hskip 0.02in}c@{\hskip 0.02in}c@{\hskip 0.05in}
c@{\hskip 0.02in}c@{\hskip 0.02in}c@{\hskip 0.02in}c@{\hskip 0.05in}|
c@{\hskip 0.02in}c@{\hskip 0.02in}c@{\hskip 0.02in}c}
\toprule
\textsc{Method} 
& \multicolumn{4}{c}{\textsc{FairFace}} 
& \multicolumn{8}{c}{\textsc{CelebA}} 
& \multicolumn{4}{c}{\textsc{WaterBirds}} 
& \multicolumn{8}{c}{\textsc{UTKFace}} 
& \multicolumn{4}{c}{\textsc{Average Results}} \\
\cmidrule(lr){2-5} 
\cmidrule(lr){6-13} 
\cmidrule(lr){14-17} 
\cmidrule(lr){18-25} 
& \multicolumn{4}{c}{$\text{Gender} \times \text{Race}$} 
& \multicolumn{4}{c}{$\text{Hair color} \times \text{Gender}$} 
& \multicolumn{4}{c}{$\text{Smiling} \times \text{Gender}$} 
& \multicolumn{4}{c}{$\text{Type} \times \text{Background}$} 
& \multicolumn{4}{c}{$\text{Age} \times \text{Race}$} 
& \multicolumn{4}{c}{$\text{Gender} \times \text{Race}$} 
& \multicolumn{4}{c}{} \\
\cmidrule(lr){2-5} 
\cmidrule(lr){6-9} 
\cmidrule(lr){10-13} 
\cmidrule(lr){14-17} 
\cmidrule(lr){18-21} 
\cmidrule(lr){22-25} 
\cmidrule(lr){26-29}
& A & WGA & B & EOD 
& A & WGA & B & EOD 
& A & WGA & B & EOD 
& A & WGA & B & EOD 
& A & WGA & B & EOD 
& A & WGA & B & EOD 
& A & WGA & B & EOD \\
\midrule
Zero Shot & \textbf{95.7} & \textbf{90.0} & \textbf{5.7} & 9.4 & 86.4 & \textbf{67.8} & \textbf{18.6} & \textbf{22.3} & \textbf{75.8} & \textbf{53.7} & \textbf{22.1} & \textbf{8.0} & 83.8 & 40.2 & 43.7 & \textbf{25.0} & 80.3 & 45.7 & 34.5 & \textbf{23.2} & \textbf{97.1} & \textbf{90.1} & \textbf{7.0} & \textbf{9.1} & \textbf{86.5} & \textbf{64.6} & \textbf{21.9} & \textbf{16.2} \\
\midrule
$\rho=0.1$ & \colorbox{red!25}{91.2} & \colorbox{red!25}{78.6} & \colorbox{red!25}{12.6} & \colorbox{red!25}{11.9} & \colorbox{green!25}{\textbf{90.4}} & \colorbox{red!25}{\underline{56.3}} & \colorbox{red!25}{34.0} & \underline{23.8} & \colorbox{red!25}{\underline{69.9}} & \colorbox{red!25}{\underline{38.0}} & \colorbox{red!25}{\underline{31.9}} & \colorbox{red!25}{\underline{11.1}} & 83.1 & 40.0 & 43.2 & \colorbox{red!25}{\underline{29.0}} & \colorbox{green!25}{86.2} & \colorbox{green!25}{53.8} & \colorbox{green!25}{32.4} & \colorbox{red!25}{28.7} & \colorbox{red!25}{93.5} & \colorbox{red!25}{77.8} & \colorbox{red!25}{15.7} & \colorbox{red!25}{13.4} & 85.7 & \colorbox{red!25}{57.4} & \colorbox{red!25}{28.3} & \colorbox{red!25}{\underline{19.6}} \\
$\rho=0.25$ & \colorbox{red!25}{93.4} & \colorbox{red!25}{83.3} & \colorbox{red!25}{10.1} & \colorbox{red!25}{11.4} & \colorbox{green!25}{\textbf{90.4}} & \colorbox{red!25}{\underline{56.3}} & \colorbox{red!25}{34.1} & \colorbox{red!25}{27.1} & \colorbox{red!25}{68.7} & \colorbox{red!25}{34.3} & \colorbox{red!25}{34.4} & \colorbox{red!25}{12.6} & 83.9 & 41.9 & 42.0 & \colorbox{red!25}{31.0} & \colorbox{green!25}{86.6} & \colorbox{green!25}{56.0} & \colorbox{green!25}{30.5} & \colorbox{red!25}{30.3} & \colorbox{red!25}{94.8} & \colorbox{red!25}{\underline{81.4}} & \colorbox{red!25}{\underline{13.4}} & \colorbox{red!25}{\underline{12.9}} & 86.3 & \colorbox{red!25}{58.9} & \colorbox{red!25}{27.4} & \colorbox{red!25}{20.9} \\
$\rho=0.5$ & 94.1 & \colorbox{red!25}{86.2} & \colorbox{red!25}{7.8} & \underline{9.3} & \colorbox{green!25}{\underline{90.3}} & \colorbox{red!25}{\underline{56.3}} & \colorbox{red!25}{34.0} & \colorbox{red!25}{28.0} & \colorbox{red!25}{68.2} & \colorbox{red!25}{33.1} & \colorbox{red!25}{35.1} & \colorbox{red!25}{13.2} & \underline{84.1} & 42.1 & 42.0 & \colorbox{red!25}{31.1} & \colorbox{green!25}{\underline{87.0}} & \colorbox{green!25}{\textbf{58.3}} & \colorbox{green!25}{\textbf{28.7}} & \colorbox{red!25}{\underline{28.1}} & \colorbox{red!25}{94.9} & \colorbox{red!25}{80.2} & \colorbox{red!25}{14.7} & \colorbox{red!25}{14.2} & \underline{86.4} & \colorbox{red!25}{\underline{59.4}} & \colorbox{red!25}{\underline{27.1}} & \colorbox{red!25}{20.6} \\
$\rho = 0.75$ & 94.2 & \colorbox{red!25}{86.1} & \colorbox{red!25}{8.1} & 9.8 & \colorbox{green!25}{90.2} & \colorbox{red!25}{\underline{56.3}} & \colorbox{red!25}{33.9} & \colorbox{red!25}{27.7} & \colorbox{red!25}{68.2} & \colorbox{red!25}{32.5} & \colorbox{red!25}{35.7} & \colorbox{red!25}{13.7} & \underline{84.1} & \colorbox{green!25}{\textbf{43.1}} & \colorbox{green!25}{\textbf{41.0}} & \colorbox{red!25}{\underline{29.0}} & \colorbox{green!25}{\underline{87.0}} & \colorbox{green!25}{\underline{57.0}} & \colorbox{green!25}{\underline{30.0}} & \colorbox{red!25}{29.0} & \colorbox{red!25}{\underline{95.1}} & \colorbox{red!25}{80.5} & \colorbox{red!25}{14.6} & \colorbox{red!25}{14.2} & \textbf{86.5} & \colorbox{red!25}{59.2} & \colorbox{red!25}{27.2} & \colorbox{red!25}{20.6} \\
$\rho=1.0$ & \underline{94.3} & \colorbox{red!25}{\underline{87.0}} & \underline{7.3} & \textbf{8.9} & \colorbox{green!25}{90.2} & \colorbox{red!25}{\underline{56.3}} & \colorbox{red!25}{\underline{33.8}} & \colorbox{red!25}{27.8} & \colorbox{red!25}{68.5} & \colorbox{red!25}{32.9} & \colorbox{red!25}{35.6} & \colorbox{red!25}{14.2} & \textbf{84.2} & \colorbox{green!25}{\underline{42.6}} & \colorbox{green!25}{\underline{41.6}} & \colorbox{red!25}{30.2} & \colorbox{green!25}{\textbf{87.1}} & \colorbox{green!25}{56.4} & \colorbox{green!25}{30.7} & \colorbox{red!25}{31.1} & \colorbox{red!25}{\underline{95.1}} & \colorbox{red!25}{79.9} & \colorbox{red!25}{15.2} & \colorbox{red!25}{15.0} & \textbf{86.5} & \colorbox{red!25}{59.2} & \colorbox{red!25}{27.4} & \colorbox{red!25}{21.2} \\
\bottomrule
\end{tabular}
}
\caption{Hyperparameter benchmark of \textsc{Zero}. We report mean over 5 random seeds and mean aggregation over all equally sized datasets. Best results are in \textbf{bold}, second best \underline{underlined}. The values that improve upon \textsc{Zero-Shot} by more than 2.0 percentage points are highlighted in \colorbox{green!25}{green}, while degradations greater than 2.0 percentage points are shown in \colorbox{red!25}{red}. }
\label{tab:zero_scan}
\end{table*}

\begin{table*}[htbp]
\centering
\tiny
\resizebox{\linewidth}{!}{
\begin{tabular}{l@{\hskip 0.05in}
c@{\hskip 0.02in}c@{\hskip 0.02in}c@{\hskip 0.02in}c@{\hskip 0.05in}
c@{\hskip 0.02in}c@{\hskip 0.02in}c@{\hskip 0.02in}c@{\hskip 0.05in}
c@{\hskip 0.02in}c@{\hskip 0.02in}c@{\hskip 0.02in}c@{\hskip 0.05in}
c@{\hskip 0.02in}c@{\hskip 0.02in}c@{\hskip 0.02in}c@{\hskip 0.05in}
c@{\hskip 0.02in}c@{\hskip 0.02in}c@{\hskip 0.02in}c@{\hskip 0.05in}
c@{\hskip 0.02in}c@{\hskip 0.02in}c@{\hskip 0.02in}c@{\hskip 0.05in}|
c@{\hskip 0.02in}c@{\hskip 0.02in}c@{\hskip 0.02in}c}
\toprule
\textsc{Method} 
& \multicolumn{4}{c}{\textsc{FairFace}} 
& \multicolumn{8}{c}{\textsc{CelebA}} 
& \multicolumn{4}{c}{\textsc{WaterBirds}} 
& \multicolumn{8}{c}{\textsc{UTKFace}} 
& \multicolumn{4}{c}{\textsc{Average Results}} \\
\cmidrule(lr){2-5} 
\cmidrule(lr){6-13} 
\cmidrule(lr){14-17} 
\cmidrule(lr){18-25} 
& \multicolumn{4}{c}{$\text{Gender} \times \text{Race}$} 
& \multicolumn{4}{c}{$\text{Hair color} \times \text{Gender}$} 
& \multicolumn{4}{c}{$\text{Smiling} \times \text{Gender}$} 
& \multicolumn{4}{c}{$\text{Type} \times \text{Background}$} 
& \multicolumn{4}{c}{$\text{Age} \times \text{Race}$} 
& \multicolumn{4}{c}{$\text{Gender} \times \text{Race}$} 
& \multicolumn{4}{c}{} \\
\cmidrule(lr){2-5} 
\cmidrule(lr){6-9} 
\cmidrule(lr){10-13} 
\cmidrule(lr){14-17} 
\cmidrule(lr){18-21} 
\cmidrule(lr){22-25} 
\cmidrule(lr){26-29}
& A & WGA & B & EOD 
& A & WGA & B & EOD 
& A & WGA & B & EOD 
& A & WGA & B & EOD 
& A & WGA & B & EOD 
& A & WGA & B & EOD 
& A & WGA & B & EOD \\
\midrule
Zero Shot & 95.7 & 90.0 & 5.7 & 9.4 & \textbf{86.4} & 67.8 & 18.6 & \underline{22.3} & \textbf{75.8} & \textbf{53.7} & \textbf{22.1} & \textbf{8.0} & 83.8 & 40.2 & 43.7 & 25.0 & \underline{80.3} & 45.7 & 34.5 & 23.2 & \textbf{97.1} & \textbf{90.1} & \textbf{7.0} & \textbf{9.1} & \textbf{86.5} & 64.6 & 21.9 & 16.2 \\
\midrule
$\lambda_\textnormal{orth}=0$ & \textbf{96.2} & \textbf{90.9} & \textbf{5.3} & \textbf{8.5} & \underline{86.2} & \colorbox{green!25}{\underline{70.7}} & \colorbox{green!25}{15.5} & 22.7 & \colorbox{red!25}{\underline{71.8}} & \colorbox{red!25}{\underline{40.1}} & \colorbox{red!25}{\underline{31.7}} & \colorbox{red!25}{\underline{24.4}} & 85.5 & \colorbox{green!25}{55.7} & \colorbox{green!25}{29.9} & \colorbox{green!25}{\underline{4.3}} & \textbf{81.1} & 45.2 & 36.0 & \colorbox{green!25}{18.9} & 96.9 & \underline{89.6} & \underline{7.3} & \underline{9.4} & 86.3 & 65.4 & 20.9 & \underline{14.7} \\
$\lambda_\textnormal{orth}=1$ & \textbf{96.2} & \textbf{90.9} & \textbf{5.3} & \textbf{8.5} & \underline{86.2} & \colorbox{green!25}{\underline{70.7}} & \colorbox{green!25}{15.5} & 22.7 & \colorbox{red!25}{71.7} & \colorbox{red!25}{39.9} & \colorbox{red!25}{31.8} & \colorbox{red!25}{24.5} & \colorbox{green!25}{\underline{86.0}} & \colorbox{green!25}{58.5} & \colorbox{green!25}{27.5} & \colorbox{green!25}{\textbf{3.6}} & \textbf{81.1} & 45.2 & 36.0 & \colorbox{green!25}{18.2} & 96.9 & \underline{89.6} & \underline{7.3} & \underline{9.4} & \underline{86.4} & \underline{65.8} & 20.6 & \textbf{14.5} \\
$\lambda_\textnormal{orth}=10$ & \underline{96.1} & \underline{90.6} & \underline{5.5} & \underline{8.8} & 86.1 & \colorbox{green!25}{\underline{70.7}} & \colorbox{green!25}{15.4} & 22.5 & \colorbox{red!25}{71.5} & \colorbox{red!25}{38.3} & \colorbox{red!25}{33.2} & \colorbox{red!25}{27.6} & \colorbox{green!25}{\textbf{86.5}} & \colorbox{green!25}{\textbf{63.1}} & \colorbox{green!25}{\underline{23.4}} & \colorbox{green!25}{8.5} & \textbf{81.1} & 45.2 & 36.0 & \colorbox{green!25}{18.6} & \underline{97.0} & \underline{89.6} & \underline{7.3} & \underline{9.4} & \underline{86.4} & \textbf{66.3} & \underline{20.1} & 15.9 \\
$\lambda_\textnormal{orth} = 1000$ & 95.9 & \colorbox{red!25}{87.8} & \colorbox{red!25}{8.1} & 11.3 & 85.1 & \colorbox{green!25}{\textbf{71.6}} & \colorbox{green!25}{\underline{13.5}} & \textbf{22.0} & \colorbox{red!25}{71.4} & \colorbox{red!25}{36.8} & \colorbox{red!25}{34.6} & \colorbox{red!25}{30.6} & 83.3 & \colorbox{green!25}{\underline{62.0}} & \colorbox{green!25}{\textbf{21.4}} & \colorbox{green!25}{16.2} & 79.8 & \underline{47.0} & \underline{32.8} & \colorbox{green!25}{\underline{14.4}} & 96.8 & 88.9 & 7.9 & 10.1 & 85.4 & 65.7 & \colorbox{green!25}{\textbf{19.7}} & 17.4 \\
$\lambda_\textnormal{orth}=100000$ & 95.8 & \colorbox{red!25}{86.5} & \colorbox{red!25}{9.3} & \colorbox{red!25}{12.7} & 85.0 & \colorbox{green!25}{\textbf{71.6}} & \colorbox{green!25}{\textbf{13.4}} & \textbf{22.0} & \colorbox{red!25}{71.4} & \colorbox{red!25}{36.4} & \colorbox{red!25}{35.0} & \colorbox{red!25}{31.2} & 83.0 & \colorbox{green!25}{56.4} & \colorbox{green!25}{26.7} & \colorbox{green!25}{13.9} & 79.4 & \colorbox{green!25}{\textbf{47.9}} & \colorbox{green!25}{\textbf{31.5}} & \colorbox{green!25}{\textbf{13.9}} & 96.8 & 88.7 & 8.2 & 10.3 & 85.2 & 64.6 & 20.7 & 17.3 \\
\bottomrule
\end{tabular}
}
\caption{Hyperparameter benchmark of \textsc{OrthCali}. We report mean over 5 random seeds and mean aggregation over all equally sized datasets. Best results are in \textbf{bold}, second best \underline{underlined}. The values that improve upon \textsc{Zero-Shot} by more than 2.0 percentage points are highlighted in \colorbox{green!25}{green}, while degradations greater than 2.0 percentage points are shown in \colorbox{red!25}{red}. }
\label{tab:orthcali_scan}
\end{table*}

\setlength{\fboxsep}{1.5pt}
\begin{table*}[h]
\centering
\tiny
\resizebox{\linewidth}{!}{
\begin{tabular}{l@{\hskip 0.05in}
c@{\hskip 0.02in}c@{\hskip 0.02in}c@{\hskip 0.02in}c@{\hskip 0.05in}
c@{\hskip 0.02in}c@{\hskip 0.02in}c@{\hskip 0.02in}c@{\hskip 0.05in}
c@{\hskip 0.02in}c@{\hskip 0.02in}c@{\hskip 0.02in}c@{\hskip 0.05in}
c@{\hskip 0.02in}c@{\hskip 0.02in}c@{\hskip 0.02in}c@{\hskip 0.05in}
c@{\hskip 0.02in}c@{\hskip 0.02in}c@{\hskip 0.02in}c@{\hskip 0.05in}
c@{\hskip 0.02in}c@{\hskip 0.02in}c@{\hskip 0.02in}c@{\hskip 0.05in}|
c@{\hskip 0.02in}c@{\hskip 0.02in}c@{\hskip 0.02in}c}
\toprule
\textsc{Method} 
& \multicolumn{4}{c}{\textsc{FairFace}} 
& \multicolumn{8}{c}{\textsc{CelebA}} 
& \multicolumn{4}{c}{\textsc{WaterBirds}} 
& \multicolumn{8}{c}{\textsc{UTKFace}} 
& \multicolumn{4}{c}{\textsc{Average Results}} \\
\cmidrule(lr){2-5} 
\cmidrule(lr){6-13} 
\cmidrule(lr){14-17} 
\cmidrule(lr){18-25} 
& \multicolumn{4}{c}{$\text{Gender} \times \text{Race}$} 
& \multicolumn{4}{c}{$\text{Hair color} \times \text{Gender}$} 
& \multicolumn{4}{c}{$\text{Smiling} \times \text{Gender}$} 
& \multicolumn{4}{c}{$\text{Type} \times \text{Background}$} 
& \multicolumn{4}{c}{$\text{Age} \times \text{Race}$} 
& \multicolumn{4}{c}{$\text{Gender} \times \text{Race}$} 
& \multicolumn{4}{c}{} \\
\cmidrule(lr){2-5} 
\cmidrule(lr){6-9} 
\cmidrule(lr){10-13} 
\cmidrule(lr){14-17} 
\cmidrule(lr){18-21} 
\cmidrule(lr){22-25} 
\cmidrule(lr){26-29}
& A & WGA & B & EOD 
& A & WGA & B & EOD 
& A & WGA & B & EOD 
& A & WGA & B & EOD 
& A & WGA & B & EOD 
& A & WGA & B & EOD 
& A & WGA & B & EOD \\
\midrule
Zero Shot & \textbf{95.7} & 90.0 & 5.7 & 9.4 & 86.4 & \textbf{67.8} & 18.6 & 22.3 & 75.8 & 53.7 & 22.1 & 8.0 & 83.8 & 40.2 & 43.7 & 25.0 & 80.3 & 45.7 & 34.5 & 23.2 & \textbf{97.1} & 90.1 & 7.0 & 9.1 & \textbf{86.5} & 64.6 & 21.9 & 16.2 \\
\midrule
\multicolumn{22}{l}{\textsc{FairTPT}}\\
\hspace{0.1cm}$\lambda_{\textnormal{fair}} = 1$ & \underline{95.5} & 90.5 & 5.0 & 8.8 & 86.2 & \underline{67.7} & 18.5 & 22.0 & 75.3 & 52.6 & 22.7 & 8.1 & 83.9 & 41.0 & 42.8 & 26.2 & 80.3 & 46.5 & 33.8 & 22.5 & \underline{97.0} & \underline{90.6} & 6.5 & 8.7 & \underline{86.4} & 64.8 & 21.6 & 16.1 \\
\hspace{0.1cm}$\lambda_{\textnormal{fair}} = 100$ & \underline{95.5} & \textbf{90.7} & \textbf{4.8} & \underline{8.1} & 85.6 & \underline{67.7} & 17.9 & 21.9 & 75.9 & \colorbox{green!25}{56.1} & \colorbox{green!25}{19.8} & 6.2 & 83.2 & \colorbox{green!25}{42.4} & \colorbox{green!25}{40.8} & 23.5 & 81.0 & \colorbox{green!25}{51.5} & \colorbox{green!25}{29.5} & 23.4 & 96.7 & \underline{90.6} & \underline{6.1} & 8.6 & 86.3 & \underline{66.5} & \colorbox{green!25}{\underline{19.8}} & \underline{15.3} \\
\hspace{0.1cm}$\lambda_{\textnormal{fair}} = 5000$ & 95.4 & \underline{90.6} & \underline{4.9} & \underline{8.1} & 85.5 & 67.5 & 18.0 & 21.5 & 75.9 & 55.6 & 20.3 & \colorbox{green!25}{\textbf{5.8}} & 83.3 & \colorbox{green!25}{43.6} & \colorbox{green!25}{39.7} & 23.4 & 80.2 & \colorbox{green!25}{49.4} & \colorbox{green!25}{30.8} & \colorbox{red!25}{26.4} & 96.7 & 90.1 & 6.6 & 9.0 & 86.2 & 66.1 & 20.0 & 15.7 \\
\midrule
\multicolumn{22}{l}{\textsc{FairTPT (MO)}}\\
\hspace{0.1cm}$\lambda_{\textnormal{fair (mo)}} = 1$ & 95.4 & 90.0 & 5.4 & 9.5 & 86.3 & \textbf{67.8} & 18.6 & 22.1 & 75.3 & 52.3 & 23.0 & 8.4 & 83.8 & 40.9 & 42.9 & 26.2 & 80.3 & 46.7 & 33.6 & \underline{22.3} & \underline{97.0} & \underline{90.6} & 6.5 & 8.7 & \underline{86.4} & 64.7 & 21.7 & 16.2 \\
\hspace{0.1cm}$\lambda_{\textnormal{fair (mo)}} = 100$ & 95.3 & 90.4 & \underline{4.9} & \textbf{7.9} & 85.3 & 67.5 & 17.8 & 21.3 & 76.1 & \colorbox{green!25}{57.8} & \colorbox{green!25}{18.3} & 6.3 & 83.2 & 41.4 & 41.8 & 24.6 & 80.8 & \colorbox{green!25}{51.6} & \colorbox{green!25}{29.3} & \colorbox{green!25}{\textbf{20.9}} & 96.6 & \textbf{90.9} & \textbf{5.8} & \underline{8.1} & 86.2 & \colorbox{green!25}{\textbf{66.6}} & \colorbox{green!25}{\textbf{19.6}} & \textbf{14.8} \\
\hspace{0.1cm}$\lambda_{\textnormal{fair (mo)}} = 5000$ & 95.3 & 90.2 & 5.2 & 8.4 & 85.2 & 67.4 & 17.8 & 22.0 & 76.0 & \colorbox{green!25}{56.3} & \colorbox{green!25}{19.7} & 6.2 & 83.2 & 41.7 & \colorbox{green!25}{41.5} & 24.8 & 80.9 & \colorbox{green!25}{48.9} & \colorbox{green!25}{32.0} & \colorbox{red!25}{26.3} & 96.7 & 90.3 & 6.3 & 8.5 & 86.2 & 65.8 & 20.4 & 16.0 \\
\midrule
\multicolumn{22}{l}{\textsc{FairTPT} \textit{special cases}}\\
\hspace{0.1cm}$\lambda_{\textnormal{fair}} = 0$ & 95.1 & 90.1 & 5.0 & 9.2 & 86.9 & \colorbox{red!25}{65.8} & \colorbox{red!25}{21.1} & 23.9 & 74.4 & \colorbox{red!25}{50.5} & 23.9 & 8.1 & 83.7 & 40.2 & 43.5 & 25.2 & 80.3 & 46.2 & 34.1 & 22.8 & 96.6 & 90.3 & 6.3 & 8.8 & 86.2 & 63.8 & 22.3 & 16.4 \\
\hspace{0.1cm}$\lambda_{\textnormal{fair}} = \infty$ & \underline{95.5} & 90.1 & 5.4 & 9.0 & 85.2 & 67.4 & 17.8 & 21.6 & 75.9 & \colorbox{green!25}{55.8} & \colorbox{green!25}{20.1} & \colorbox{green!25}{\underline{5.9}} & 83.5 & \colorbox{green!25}{44.5} & \colorbox{green!25}{39.0} & 25.0 & 80.8 & \colorbox{green!25}{49.7} & \colorbox{green!25}{31.1} & \colorbox{red!25}{27.6} & 96.8 & 90.3 & 6.5 & 8.9 & 86.3 & 66.3 & 20.0 & 15.8 \\
\midrule
\multicolumn{22}{l}{\textsc{FairTPT} \textit{without ELRA}} \\
\hspace{0.1cm}$\lambda_{\textnormal{fair}} = 1$ & 94.5 & \colorbox{red!25}{79.6} & \colorbox{red!25}{14.9} & \colorbox{red!25}{17.6} & \colorbox{green!25}{89.9} & \colorbox{red!25}{56.3} & \colorbox{red!25}{33.6} & \colorbox{red!25}{28.8} & \colorbox{red!25}{67.7} & \colorbox{red!25}{30.2} & \colorbox{red!25}{37.5} & \colorbox{red!25}{16.8} & \textbf{84.4} & \colorbox{green!25}{43.6} & \colorbox{green!25}{40.8} & \colorbox{red!25}{30.3} & \colorbox{green!25}{\underline{87.0}} & \colorbox{green!25}{56.4} & \colorbox{green!25}{30.6} & \colorbox{red!25}{31.8} & 95.7 & 88.4 & 7.3 & 8.5 & \textbf{86.5} & \colorbox{red!25}{59.1} & \colorbox{red!25}{27.4} & \colorbox{red!25}{22.3} \\
\hspace{0.1cm}$\lambda_{\textnormal{fair}} = 100$ & \colorbox{red!25}{88.6} & \colorbox{red!25}{25.2} & \colorbox{red!25}{63.3} & \colorbox{red!25}{71.9} & \colorbox{red!25}{68.7} & \colorbox{red!25}{51.1} & 17.6 & \colorbox{red!25}{25.3} & \colorbox{green!25}{79.0} & \colorbox{green!25}{\underline{65.9}} & \colorbox{green!25}{\underline{13.1}} & \colorbox{red!25}{25.0} & \underline{84.3} & \colorbox{green!25}{\underline{45.2}} & \colorbox{green!25}{39.1} & \underline{23.3} & \colorbox{red!25}{76.5} & \colorbox{green!25}{\textbf{62.4}} & \colorbox{green!25}{\underline{14.2}} & \colorbox{red!25}{30.9} & \colorbox{red!25}{82.4} & \colorbox{red!25}{15.2} & \colorbox{red!25}{67.2} & \colorbox{red!25}{80.9} & \colorbox{red!25}{79.9} & \colorbox{red!25}{44.2} & \colorbox{red!25}{35.8} & \colorbox{red!25}{39.4} \\
\hspace{0.1cm}$\lambda_{\textnormal{fair}} = 5000$ & \colorbox{red!25}{88.2} & \colorbox{red!25}{24.8} & \colorbox{red!25}{63.4} & \colorbox{red!25}{72.4} & \colorbox{red!25}{67.3} & \colorbox{red!25}{50.3} & \underline{17.0} & 23.5 & \colorbox{green!25}{\textbf{80.4}} & \colorbox{red!25}{25.0} & \colorbox{red!25}{25.0} & 6.1 & 84.1 & \colorbox{green!25}{\textbf{45.4}} & \colorbox{green!25}{\textbf{38.7}} & \colorbox{green!25}{\textbf{22.7}} & \colorbox{red!25}{75.2} & \colorbox{green!25}{\underline{62.3}} & \colorbox{green!25}{25.0} & \colorbox{red!25}{29.2} & \colorbox{red!25}{82.0} & \colorbox{red!25}{13.7} & \colorbox{red!25}{68.2} & \colorbox{red!25}{82.0} & \colorbox{red!25}{79.5} & \colorbox{red!25}{44.3} & \colorbox{red!25}{35.2} & \colorbox{red!25}{39.3} \\
\midrule
\multicolumn{22}{l}{\textsc{FairTPT (MO)} \textit{without ELRA}} \\
\hspace{0.1cm}$\lambda_{\textnormal{fair (mo)}} = 1$ & 94.5 & \colorbox{red!25}{80.4} & \colorbox{red!25}{14.2} & \colorbox{red!25}{16.5} & \colorbox{green!25}{\underline{90.0}} & \colorbox{red!25}{56.3} & \colorbox{red!25}{33.7} & \colorbox{red!25}{28.8} & \colorbox{red!25}{67.6} & \colorbox{red!25}{29.7} & \colorbox{red!25}{37.9} & \colorbox{red!25}{17.4} & \textbf{84.4} & \colorbox{green!25}{43.8} & \colorbox{green!25}{40.6} & \colorbox{red!25}{30.1} & \colorbox{green!25}{\underline{87.0}} & \colorbox{green!25}{55.9} & \colorbox{green!25}{31.1} & \colorbox{red!25}{32.3} & 95.7 & 88.4 & 7.3 & 8.5 & \textbf{86.5} & \colorbox{red!25}{59.1} & \colorbox{red!25}{27.5} & \colorbox{red!25}{22.3} \\
\hspace{0.1cm}$\lambda_{\textnormal{fair (mo)}} = 100$ & \colorbox{red!25}{90.0} & \colorbox{red!25}{29.9} & \colorbox{red!25}{60.1} & \colorbox{red!25}{67.6} & \colorbox{red!25}{70.9} & \colorbox{red!25}{56.2} & \colorbox{green!25}{\textbf{14.7}} & \colorbox{green!25}{\underline{20.3}} & \colorbox{green!25}{78.7} & \colorbox{green!25}{62.5} & \colorbox{green!25}{16.2} & 7.0 & \underline{84.3} & \colorbox{green!25}{45.0} & \colorbox{green!25}{39.3} & 23.8 & 78.8 & \colorbox{red!25}{25.0} & \colorbox{green!25}{14.9} & \colorbox{red!25}{31.6} & \colorbox{red!25}{85.2} & \colorbox{red!25}{25.9} & \colorbox{red!25}{59.3} & \colorbox{red!25}{70.3} & \colorbox{red!25}{81.3} & \colorbox{red!25}{47.2} & \colorbox{red!25}{34.1} & \colorbox{red!25}{36.8} \\
\hspace{0.1cm}$\lambda_{\textnormal{fair (mo)}} = 5000$ & \colorbox{red!25}{89.6} & \colorbox{red!25}{29.0} & \colorbox{red!25}{60.6} & \colorbox{red!25}{68.5} & \colorbox{red!25}{69.4} & \colorbox{red!25}{54.6} & \colorbox{green!25}{\textbf{14.7}} & \colorbox{green!25}{\textbf{19.5}} & \colorbox{green!25}{\underline{79.9}} & \colorbox{green!25}{\textbf{68.1}} & \colorbox{green!25}{\textbf{11.8}} & \colorbox{red!25}{25.0} & \underline{84.3} & \colorbox{green!25}{45.0} & \colorbox{green!25}{39.3} & 24.2 & \colorbox{red!25}{77.4} & \colorbox{red!25}{25.0} & \colorbox{green!25}{\textbf{13.8}} & \colorbox{red!25}{31.6} & \colorbox{red!25}{84.7} & \colorbox{red!25}{25.9} & \colorbox{red!25}{58.8} & \colorbox{red!25}{70.6} & \colorbox{red!25}{80.9} & \colorbox{red!25}{47.7} & \colorbox{red!25}{33.2} & \colorbox{red!25}{36.4} \\
\midrule
\multicolumn{22}{l}{\textsc{FairTPT} \textit{special cases without ELRA}} \\
\hspace{0.1cm}$\lambda_{\textnormal{fair}} = 0$ & \colorbox{red!25}{93.6} & \colorbox{red!25}{85.0} & \colorbox{red!25}{8.6} & \colorbox{red!25}{12.4} & \colorbox{green!25}{\textbf{92.0}} & \colorbox{red!25}{42.6} & \colorbox{red!25}{49.4} & \colorbox{red!25}{37.6} & \colorbox{red!25}{59.4} & \colorbox{red!25}{20.0} & \colorbox{red!25}{39.4} & \colorbox{red!25}{25.0} & 84.0 & 40.8 & 43.1 & \colorbox{red!25}{37.1} & \colorbox{green!25}{\textbf{87.5}} & \colorbox{green!25}{57.7} & \colorbox{green!25}{29.8} & \colorbox{red!25}{26.3} & \colorbox{red!25}{94.3} & \colorbox{red!25}{86.4} & 7.9 & \textbf{8.0} & 85.1 & \colorbox{red!25}{55.4} & \colorbox{red!25}{29.7} & \colorbox{red!25}{21.1} \\
\hspace{0.1cm}$\lambda_{\textnormal{fair}} = \infty$ & \colorbox{red!25}{88.2} & \colorbox{red!25}{24.8} & \colorbox{red!25}{63.4} & \colorbox{red!25}{72.3} & \colorbox{red!25}{67.3} & \colorbox{red!25}{50.1} & 17.2 & 23.6 & \colorbox{green!25}{\textbf{80.4}} & \colorbox{red!25}{25.0} & \colorbox{red!25}{25.0} & 6.6 & 84.2 & \colorbox{green!25}{\textbf{45.4}} & \colorbox{green!25}{\underline{38.8}} & \colorbox{green!25}{\textbf{22.7}} & \colorbox{red!25}{75.1} & \colorbox{green!25}{62.0} & \colorbox{green!25}{25.0} & \colorbox{red!25}{28.8} & \colorbox{red!25}{82.0} & \colorbox{red!25}{13.3} & \colorbox{red!25}{68.6} & \colorbox{red!25}{82.4} & \colorbox{red!25}{79.5} & \colorbox{red!25}{44.2} & \colorbox{red!25}{35.3} & \colorbox{red!25}{39.4} \\
\bottomrule
\end{tabular}
}
\caption{Overall and subgroup-level performance evaluation of \textsc{FairTPT} and \textsc{FairTPT (MO)} (with and without ELRA). We report mean over 5 random seeds and mean aggregation over all equally sized datasets. Best results are in \textbf{bold}, second best \underline{underlined}. The values that improve upon \textsc{Zero-Shot} by more than 2.0 percentage points are highlighted in \colorbox{green!25}{green}, while degradations greater than 2.0 percentage points are shown in \colorbox{red!25}{red}.}
\label{tab:fair-tpt-sensitivity-elra-ablation}
\end{table*}

\setlength{\fboxsep}{1.5pt}
\begin{table*}[h]
\centering
\tiny
\resizebox{\linewidth}{!}{
\begin{tabular}{l@{\hskip 0.05in}
c@{\hskip 0.02in}c@{\hskip 0.02in}c@{\hskip 0.02in}c@{\hskip 0.05in}
c@{\hskip 0.02in}c@{\hskip 0.02in}c@{\hskip 0.02in}c@{\hskip 0.05in}
c@{\hskip 0.02in}c@{\hskip 0.02in}c@{\hskip 0.02in}c@{\hskip 0.05in}
c@{\hskip 0.02in}c@{\hskip 0.02in}c@{\hskip 0.02in}c@{\hskip 0.05in}
c@{\hskip 0.02in}c@{\hskip 0.02in}c@{\hskip 0.02in}c@{\hskip 0.05in}
c@{\hskip 0.02in}c@{\hskip 0.02in}c@{\hskip 0.02in}c@{\hskip 0.05in}|
c@{\hskip 0.02in}c@{\hskip 0.02in}c@{\hskip 0.02in}c}
\toprule
\textsc{Method} 
& \multicolumn{4}{c}{\textsc{FairFace}} 
& \multicolumn{8}{c}{\textsc{CelebA}} 
& \multicolumn{4}{c}{\textsc{WaterBirds}} 
& \multicolumn{8}{c}{\textsc{UTKFace}} 
& \multicolumn{4}{c}{\textsc{Average Results}} \\
\cmidrule(lr){2-5} 
\cmidrule(lr){6-13} 
\cmidrule(lr){14-17} 
\cmidrule(lr){18-25} 
& \multicolumn{4}{c}{$\text{Gender} \times \text{Race}$} 
& \multicolumn{4}{c}{$\text{Hair color} \times \text{Gender}$} 
& \multicolumn{4}{c}{$\text{Smiling} \times \text{Gender}$} 
& \multicolumn{4}{c}{$\text{Type} \times \text{Background}$} 
& \multicolumn{4}{c}{$\text{Age} \times \text{Race}$} 
& \multicolumn{4}{c}{$\text{Gender} \times \text{Race}$} 
& \multicolumn{4}{c}{} \\
\cmidrule(lr){2-5} 
\cmidrule(lr){6-9} 
\cmidrule(lr){10-13} 
\cmidrule(lr){14-17} 
\cmidrule(lr){18-21} 
\cmidrule(lr){22-25} 
\cmidrule(lr){26-29}
& A & WGA & B & EOD 
& A & WGA & B & EOD 
& A & WGA & B & EOD 
& A & WGA & B & EOD 
& A & WGA & B & EOD 
& A & WGA & B & EOD 
& A & WGA & B & EOD \\
\midrule

Zero Shot & \textbf{95.7} & 90.0 & 5.7 & 9.4 & 86.4 & 67.8 & 18.6 & 22.3 & 75.8 & 53.7 & 22.1 & 8.0 & 83.8 & 40.2 & 43.7 & 25.0 & 80.3 & 45.7 & 34.5 & 23.2 & \textbf{97.1} & 90.1 & 7.0 & 9.1 & \textbf{86.5} & 64.6 & 21.9 & 16.2 \\
\midrule
\multicolumn{22}{l}{\textsc{FairTPT} and \textsc{FairTPT (MO)} \textit{with S loss}}\\
\hspace{0.3cm}$\lambda_{\textnormal{fair}} = 1$ & 95.5 & 90.5 & 5.0 & 8.8 & 86.2 & 67.7 & 18.5 & 22.0 & 75.3 & 52.6 & 22.7 & 8.1 & 83.9 & 41.0 & 42.8 & 26.2 & 80.3 & 46.5 & 33.8 & 22.5 & \underline{97.0} & 90.6 & 6.5 & 8.7 & \underline{86.4} & 64.8 & 21.6 & 16.1 \\
\hspace{0.3cm}$\lambda_{\textnormal{fair}} = 100$ & 95.5 & \textbf{90.7} & \textbf{4.8} & \underline{8.1} & 85.6 & 67.7 & 17.9 & 21.9 & 75.9 & \colorbox{green!25}{56.1} & \colorbox{green!25}{19.8} & 6.2 & 83.2 & \colorbox{green!25}{42.4} & \colorbox{green!25}{40.8} & \underline{23.5} & 81.0 & \colorbox{green!25}{\underline{51.5}} & \colorbox{green!25}{\underline{29.5}} & 23.4 & 96.7 & 90.6 & 6.1 & 8.6 & 86.3 & \underline{66.5} & \colorbox{green!25}{\underline{19.8}} & 15.3 \\
\hspace{0.3cm}$\lambda_{\textnormal{fair}} = 5000$ & 95.4 & \underline{90.6} & \underline{4.9} & \underline{8.1} & 85.5 & 67.5 & 18.0 & 21.5 & 75.9 & 55.6 & 20.3 & \colorbox{green!25}{5.8} & 83.3 & \colorbox{green!25}{\underline{43.6}} & \colorbox{green!25}{\underline{39.7}} & \textbf{23.4} & 80.2 & \colorbox{green!25}{49.4} & \colorbox{green!25}{30.8} & \colorbox{red!25}{26.4} & 96.7 & 90.1 & 6.6 & 9.0 & 86.2 & 66.1 & 20.0 & 15.7 \\
\hspace{0.3cm}$\lambda_{\textnormal{fair}} = \infty$ & 95.5 & 90.1 & 5.4 & 9.0 & 85.2 & 67.4 & \underline{17.8} & 21.6 & 75.9 & \colorbox{green!25}{55.8} & \colorbox{green!25}{20.1} & \colorbox{green!25}{5.9} & 83.5 & \colorbox{green!25}{\textbf{44.5}} & \colorbox{green!25}{\textbf{39.0}} & 25.0 & 80.8 & \colorbox{green!25}{49.7} & \colorbox{green!25}{31.1} & \colorbox{red!25}{27.6} & 96.8 & 90.3 & 6.5 & 8.9 & 86.3 & 66.3 & 20.0 & 15.8 \\
\hspace{0.3cm}$\lambda_{\textnormal{fair (mo)}} = 1$ & 95.4 & 90.0 & 5.4 & 9.5 & 86.3 & 67.8 & 18.6 & 22.1 & 75.3 & 52.3 & 23.0 & 8.4 & 83.8 & 40.9 & 42.9 & 26.2 & 80.3 & 46.7 & 33.6 & 22.3 & \underline{97.0} & 90.6 & 6.5 & 8.7 & \underline{86.4} & 64.7 & 21.7 & 16.2 \\
\hspace{0.3cm}$\lambda_{\textnormal{fair (mo)}} = 100$ & 95.3 & 90.4 & \underline{4.9} & \textbf{7.9} & 85.3 & 67.5 & \underline{17.8} & 21.3 & \textbf{76.1} & \colorbox{green!25}{\textbf{57.8}} & \colorbox{green!25}{\textbf{18.3}} & 6.3 & 83.2 & 41.4 & 41.8 & 24.6 & 80.8 & \colorbox{green!25}{\textbf{51.6}} & \colorbox{green!25}{\textbf{29.3}} & \colorbox{green!25}{20.9} & 96.6 & 90.9 & 5.8 & 8.1 & 86.2 & \colorbox{green!25}{\textbf{66.6}} & \colorbox{green!25}{\textbf{19.6}} & \textbf{14.8} \\
\hspace{0.3cm}$\lambda_{\textnormal{fair (mo)}} = 5000$ & 95.3 & 90.2 & 5.2 & 8.4 & 85.2 & 67.4 & \underline{17.8} & 22.0 & \underline{76.0} & \colorbox{green!25}{\underline{56.3}} & \colorbox{green!25}{19.7} & 6.2 & 83.2 & 41.7 & \colorbox{green!25}{41.5} & 24.8 & 80.9 & \colorbox{green!25}{48.9} & \colorbox{green!25}{32.0} & \colorbox{red!25}{26.3} & 96.7 & 90.3 & 6.3 & 8.5 & 86.2 & 65.8 & 20.4 & 16.0 \\
\midrule
\multicolumn{22}{l}{\textsc{FairTPT} and \textsc{FairTPT (MO)} \textit{with TS loss}}\\
\hspace{0.3cm}$\lambda_{\textnormal{fair}} = 1$ & 95.5 & 90.1 & 5.3 & 9.4 & \textbf{86.6} & 67.8 & 18.8 & 22.3 & 75.2 & 52.5 & 22.7 & 7.6 & 83.9 & 40.5 & 43.4 & 26.7 & 80.4 & 46.9 & 33.5 & 22.1 & 96.9 & 90.6 & 6.3 & 8.6 & \underline{86.4} & 64.7 & 21.7 & 16.1 \\
\hspace{0.3cm}$\lambda_{\textnormal{fair}} = 100$ & 95.3 & 89.1 & 6.2 & 9.9 & 86.1 & \textbf{68.0} & 18.1 & \colorbox{green!25}{\underline{20.2}} & 75.0 & 54.0 & 20.9 & \colorbox{green!25}{\textbf{5.3}} & \underline{84.0} & \colorbox{green!25}{42.2} & 41.8 & 24.8 & \textbf{81.6} & \colorbox{green!25}{48.0} & 33.6 & 24.3 & 96.5 & 90.2 & 6.2 & 8.9 & \underline{86.4} & 65.3 & 21.1 & 15.6 \\
\hspace{0.3cm}$\lambda_{\textnormal{fair}} = 5000$ & 95.2 & 89.6 & 5.7 & 9.7 & 86.2 & \underline{67.9} & 18.3 & 20.4 & 75.2 & 55.0 & 20.2 & \colorbox{green!25}{6.0} & \textbf{84.1} & \colorbox{green!25}{42.2} & 41.9 & 24.6 & \underline{81.5} & 47.0 & 34.5 & 24.5 & 96.5 & 90.5 & 6.0 & 8.3 & \textbf{86.5} & 65.4 & 21.1 & 15.6 \\
\hspace{0.3cm}$\lambda_{\textnormal{fair}} = \infty$ & 95.2 & 88.6 & 6.6 & 10.5 & 86.3 & \underline{67.9} & 18.4 & \colorbox{green!25}{20.3} & 75.1 & 54.9 & 20.2 & \colorbox{green!25}{\textbf{5.3}} & \textbf{84.1} & 42.1 & 42.1 & 25.1 & \underline{81.5} & 46.1 & 35.4 & \colorbox{red!25}{27.1} & 96.4 & 90.8 & 5.7 & 8.2 & \underline{86.4} & 65.1 & 21.4 & 16.1 \\
\hspace{0.3cm}$\lambda_{\textnormal{fair (mo)}} = 1$ & 95.5 & 90.5 & 5.0 & 9.0 & \underline{86.5} & 67.8 & 18.7 & 22.1 & 75.3 & 52.7 & 22.7 & 7.8 & \underline{84.0} & 40.7 & 43.4 & 26.3 & 80.4 & 46.9 & 33.6 & 22.1 & 96.9 & 90.6 & 6.3 & 8.6 & \textbf{86.5} & 64.8 & 21.6 & 16.0 \\
\hspace{0.3cm}$\lambda_{\textnormal{fair (mo)}} = 100$ & 95.0 & \colorbox{red!25}{88.0} & 7.1 & 10.9 & 86.1 & \underline{67.9} & 18.1 & \colorbox{green!25}{\underline{20.2}} & 75.0 & 52.6 & 22.4 & 7.1 & \textbf{84.1} & 41.4 & 42.7 & 25.6 & 80.9 & 47.0 & 33.9 & 22.8 & 96.3 & \underline{91.0} & \textbf{5.2} & \textbf{7.3} & 86.2 & 64.7 & 21.6 & 15.6 \\
\hspace{0.3cm}$\lambda_{\textnormal{fair (mo)}} = 5000$ & 94.9 & \colorbox{red!25}{88.0} & 6.8 & 10.7 & 86.1 & \underline{67.9} & 18.2 & 20.4 & 75.7 & 55.1 & 20.5 & \colorbox{green!25}{\underline{5.4}} & \underline{84.0} & 41.5 & 42.5 & 25.1 & 80.9 & 45.9 & 34.9 & 24.7 & 96.2 & 90.2 & 5.9 & 8.8 & 86.3 & 64.8 & 21.5 & 15.9 \\
\midrule
\multicolumn{22}{l}{\textsc{FairTPT} and \textsc{FairTPT (MO)} \textit{with Super TS loss}}\\
\hspace{0.3cm}$\lambda_{\textnormal{fair}} = 1$ & \textbf{95.7} & 90.5 & 5.2 & 9.0 & \underline{86.5} & 67.7 & 18.8 & 22.4 & 75.3 & 52.4 & 22.9 & 8.2 & 83.9 & 40.0 & 43.9 & 26.1 & 80.3 & 46.3 & 34.0 & 22.7 & \textbf{97.1} & 90.6 & 6.5 & 8.6 & \textbf{86.5} & 64.6 & 21.9 & 16.2 \\
\hspace{0.3cm}$\lambda_{\textnormal{fair}} = 100$ & 94.8 & \colorbox{red!25}{87.7} & 7.1 & \colorbox{red!25}{11.7} & 85.7 & \underline{67.9} & \textbf{17.7} & 21.1 & 75.2 & 52.8 & 22.4 & 7.5 & 83.5 & 39.3 & 44.2 & 25.5 & 81.1 & \colorbox{green!25}{47.8} & 33.3 & 25.0 & 96.5 & 90.6 & 5.9 & 8.3 & 86.1 & 64.4 & 21.8 & 15.7 \\
\hspace{0.3cm}$\lambda_{\textnormal{fair}} = 5000$ & 94.8 & 89.0 & 5.8 & 10.2 & 85.6 & 67.8 & \underline{17.8} & \colorbox{green!25}{\underline{20.2}} & 75.2 & 55.0 & 20.2 & 6.5 & 83.6 & 39.3 & 44.3 & 25.7 & 81.1 & \colorbox{green!25}{48.6} & \colorbox{green!25}{32.5} & \colorbox{green!25}{\textbf{20.4}} & 96.5 & 90.0 & 6.5 & 8.3 & 86.1 & 65.0 & 21.2 & \underline{15.2} \\
\hspace{0.3cm}$\lambda_{\textnormal{fair}} = \infty$ & 95.0 & 88.3 & 6.7 & 10.2 & 85.7 & 67.9 & 17.8 & \colorbox{green!25}{19.8} & 75.5 & 55.9 & 19.6 & 6.5 & 83.5 & 39.5 & 44.0 & 25.5 & 81.1 & 49.1 & 32.0 & 21.3 & 96.4 & 90.3 & 6.1 & 8.4 & 86.2 & 65.2 & 21.0 & 15.3 \\
\hspace{0.3cm}$\lambda_{\textnormal{fair (mo)}} = 1$ & \underline{95.6} & 90.3 & 5.3 & 9.3 & \underline{86.5} & 67.7 & 18.8 & 22.4 & 75.4 & 52.6 & 22.8 & 8.2 & \underline{84.0} & 39.8 & 44.2 & 26.7 & 80.3 & 46.3 & 34.0 & 22.7 & \textbf{97.1} & 90.6 & 6.5 & 8.6 & \textbf{86.5} & 64.5 & 21.9 & 16.3 \\
\hspace{0.3cm}$\lambda_{\textnormal{fair (mo)}} = 100$ & 95.1 & 88.3 & 6.9 & 9.9 & 85.9 & 67.7 & 18.2 & 20.6 & 75.4 & 52.5 & 22.9 & 8.8 & 83.7 & 39.5 & 44.2 & 25.7 & 80.2 & 46.8 & 33.5 & 22.2 & 96.5 & \textbf{91.1} & \underline{5.5} & \underline{7.5} & 86.1 & 64.3 & 21.8 & 15.8 \\
\hspace{0.3cm}$\lambda_{\textnormal{fair (mo)}} = 5000$ & 95.2 & 89.3 & 6.0 & 9.6 & 85.8 & 67.8 & 18.0 & \colorbox{green!25}{\textbf{19.8}} & 75.4 & 52.2 & 23.2 & 9.8 & 83.7 & 39.3 & 44.4 & 25.7 & 80.6 & \colorbox{green!25}{48.9} & \colorbox{green!25}{31.7} & \colorbox{green!25}{\underline{20.8}} & 96.1 & 89.7 & 6.4 & 8.7 & 86.1 & 64.5 & 21.6 & 15.7 \\
\bottomrule
\end{tabular}
}
\caption{Impact of the loss on the experiment. We report mean over 5 data random seeds and mean aggregation over all equally sized datasets. Best results are in \textbf{bold}, second best \underline{underlined}. The values that improve upon \textsc{Zero-Shot} by more than 2.0 percentage points are highlighted in \colorbox{green!25}{green}, while degradations greater than 2.0 percentage points are shown in \colorbox{red!25}{red}.}
\label{tab:fair-tpt-loss-choice}
\end{table*}

\setlength{\fboxsep}{1.5pt}
\begin{table*}[h]
\centering
\tiny
\resizebox{\linewidth}{!}{
\begin{tabular}{l@{\hskip 0.05in}
c@{\hskip 0.02in}c@{\hskip 0.02in}c@{\hskip 0.02in}c@{\hskip 0.05in}
c@{\hskip 0.02in}c@{\hskip 0.02in}c@{\hskip 0.02in}c@{\hskip 0.05in}
c@{\hskip 0.02in}c@{\hskip 0.02in}c@{\hskip 0.02in}c@{\hskip 0.05in}
c@{\hskip 0.02in}c@{\hskip 0.02in}c@{\hskip 0.02in}c@{\hskip 0.05in}
c@{\hskip 0.02in}c@{\hskip 0.02in}c@{\hskip 0.02in}c@{\hskip 0.05in}
c@{\hskip 0.02in}c@{\hskip 0.02in}c@{\hskip 0.02in}c@{\hskip 0.05in}|
c@{\hskip 0.02in}c@{\hskip 0.02in}c@{\hskip 0.02in}c}
\toprule
\textsc{Method} 
& \multicolumn{4}{c}{\textsc{FairFace}} 
& \multicolumn{8}{c}{\textsc{CelebA}} 
& \multicolumn{4}{c}{\textsc{WaterBirds}} 
& \multicolumn{8}{c}{\textsc{UTKFace}} 
& \multicolumn{4}{c}{\textsc{Average Results}} \\
\cmidrule(lr){2-5} 
\cmidrule(lr){6-13} 
\cmidrule(lr){14-17} 
\cmidrule(lr){18-25} 
& \multicolumn{4}{c}{$\text{Gender} \times \text{Race}$} 
& \multicolumn{4}{c}{$\text{Hair color} \times \text{Gender}$} 
& \multicolumn{4}{c}{$\text{Smiling} \times \text{Gender}$} 
& \multicolumn{4}{c}{$\text{Type} \times \text{Background}$} 
& \multicolumn{4}{c}{$\text{Age} \times \text{Race}$} 
& \multicolumn{4}{c}{$\text{Gender} \times \text{Race}$} 
& \multicolumn{4}{c}{} \\
\cmidrule(lr){2-5} 
\cmidrule(lr){6-9} 
\cmidrule(lr){10-13} 
\cmidrule(lr){14-17} 
\cmidrule(lr){18-21} 
\cmidrule(lr){22-25} 
\cmidrule(lr){26-29}
& A & WGA & B & EOD 
& A & WGA & B & EOD 
& A & WGA & B & EOD 
& A & WGA & B & EOD 
& A & WGA & B & EOD 
& A & WGA & B & EOD 
& A & WGA & B & EOD \\
\midrule
\textsc{Zero-Shot} & 95.7 & 90.0 & 5.7 & 9.4 & \textbf{86.4} & \underline{67.8} & 18.6 & 22.3 & 75.8 & 53.7 & \underline{22.1} & \underline{8.0} & \textbf{83.8} & 40.2 & 43.7 & 25.0 & 80.3 & 45.7 & 34.5 & \underline{23.2} & \textbf{97.1} & 90.1 & 7.0 & 9.1 & \underline{86.5} & 64.6 & 21.9 & \underline{16.2} \\
\midrule
\multicolumn{22}{l}{\textit{Episodic test-time debiasing methods}} \\
\textsc{OrthCali} & \textbf{95.9} & \colorbox{red!25}{87.8} & \colorbox{red!25}{8.1} & 11.3 & 85.1 & \colorbox{green!25}{\textbf{71.6}} & \colorbox{green!25}{\textbf{13.5}} & \underline{22.0} & \colorbox{red!25}{71.4} & \colorbox{red!25}{36.8} & \colorbox{red!25}{34.6} & \colorbox{red!25}{30.6} & 83.3 & \colorbox{green!25}{\textbf{62.0}} & \colorbox{green!25}{\textbf{21.4}} & \colorbox{green!25}{\textbf{16.2}} & 79.8 & 47.0 & 32.8 & \colorbox{green!25}{\textbf{14.4}} & \underline{96.8} & 88.9 & 7.9 & 10.1 & 85.4 & 65.7 & \colorbox{green!25}{\textbf{19.7}} & 17.4 \\
\midrule
\multicolumn{22}{l}{\textit{Our methods}} \\
\textsc{FairTPT} & 95.5 & \underline{90.7} & \textbf{4.8} & \textbf{8.1} & 85.6 & 67.7 & \underline{17.9} & \textbf{21.9} & \underline{75.9} & \colorbox{green!25}{\textbf{56.1}} & \colorbox{green!25}{\textbf{19.8}} & \textbf{6.2} & 83.2 & \colorbox{green!25}{42.4} & \colorbox{green!25}{\underline{40.8}} & 23.5 & \underline{81.0} & \colorbox{green!25}{\textbf{51.5}} & \colorbox{green!25}{\textbf{29.5}} & 23.4 & 96.7 & \textbf{90.6} & \textbf{6.1} & \textbf{8.6} & 86.3 & \textbf{66.5} & \colorbox{green!25}{\underline{19.8}} & \textbf{15.3} \\
\textsc{FairTPT}+\textsc{OrthCali} & \underline{95.8} & \textbf{90.8} & \underline{5.0} & \underline{8.6} & \underline{85.7} & 67.6 & 18.1 & 23.6 & \textbf{76.7} & \underline{54.0} & 22.7 & \colorbox{red!25}{17.0} & \underline{83.7} & \colorbox{green!25}{\underline{42.9}} & \colorbox{green!25}{\underline{40.8}} & \colorbox{green!25}{\underline{17.2}} & \textbf{81.2} & \colorbox{green!25}{\underline{50.0}} & \colorbox{green!25}{\underline{31.2}} & 23.4 & 96.7 & \underline{90.3} & \underline{6.4} & \underline{8.9} & \textbf{86.6} & \underline{65.9} & 20.7 & 16.4 \\
\bottomrule
\end{tabular}
}
\caption{Overall and subgroup-level performance evaluation of \textsc{FairTPT} followed by the orthogonal projection of \textsc{OrthCali} \citep{chuang2023debiasing}; we set $\lambda_\textnormal{fair} = 100$ and $\lambda_\textnormal{orth} = 0$. The spurious embeddings used for projection are those obtained after \textsc{FairTPT}. We report mean over 5 random seeds and mean aggregation over all equally sized datasets. Best results are in \textbf{bold}, second best \underline{underlined}. The values that improve upon \textsc{Zero-Shot} by more than 2.0 percentage points are highlighted in \colorbox{green!25}{green}, while degradations greater than 2.0 percentage points are shown in \colorbox{red!25}{red}.}
\label{tab:fairtpt_with_OP}
\end{table*}

\setlength{\fboxsep}{1.5pt}
\begin{table*}[htbp]
\centering
\tiny
\resizebox{\linewidth}{!}{
\begin{tabular}{l@{\hskip 0.05in}
c@{\hskip 0.02in}c@{\hskip 0.02in}c@{\hskip 0.02in}c@{\hskip 0.05in}
c@{\hskip 0.02in}c@{\hskip 0.02in}c@{\hskip 0.02in}c@{\hskip 0.05in}
c@{\hskip 0.02in}c@{\hskip 0.02in}c@{\hskip 0.02in}c@{\hskip 0.05in}
c@{\hskip 0.02in}c@{\hskip 0.02in}c@{\hskip 0.02in}c@{\hskip 0.05in}
c@{\hskip 0.02in}c@{\hskip 0.02in}c@{\hskip 0.02in}c@{\hskip 0.05in}
c@{\hskip 0.02in}c@{\hskip 0.02in}c@{\hskip 0.02in}c@{\hskip 0.05in}
c@{\hskip 0.02in}c@{\hskip 0.02in}c@{\hskip 0.02in}c@{\hskip 0.05in}|
c@{\hskip 0.02in}c@{\hskip 0.02in}c@{\hskip 0.02in}c}
\toprule
\textsc{Method} 
& \multicolumn{12}{c}{\textsc{FairFace}}
& \multicolumn{8}{c}{\textsc{CelebA}} 
& \multicolumn{8}{c}{\textsc{UTKFace}} 
& \multicolumn{4}{c}{\textsc{Average Results}} \\
\cmidrule(lr){2-13} 
\cmidrule(lr){14-21} 
\cmidrule(lr){22-29}  
& \multicolumn{4}{c}{$\text{Age} \times \text{Race}$} 
& \multicolumn{4}{c}{$\text{Age} \times \text{Gender}$} 
& \multicolumn{4}{c}{$\text{Race} \times \text{Gender}$} 
& \multicolumn{4}{c}{$\text{Makeup} \times \text{Gender}$} 
& \multicolumn{4}{c}{$\text{Glasses} \times \text{Gender}$} 
& \multicolumn{4}{c}{$\text{Age} \times \text{Gender}$} 
& \multicolumn{4}{c}{$\text{Race} \times \text{Gender}$} 
& \multicolumn{4}{c}{} \\
\cmidrule(lr){2-5} 
\cmidrule(lr){6-9} 
\cmidrule(lr){10-13} 
\cmidrule(lr){14-17} 
\cmidrule(lr){18-21} 
\cmidrule(lr){22-25} 
\cmidrule(lr){26-29}
\cmidrule(lr){30-33}
& A & WGA & B & EOD 
& A & WGA & B & EOD 
& A & WGA & B & EOD 
& A & WGA & B & EOD 
& A & WGA & B & EOD 
& A & WGA & B & EOD 
& A & WGA & B & EOD
& A & WGA & B & EOD \\
\midrule

Zero Shot & 83.3 & 40.1 & 43.2 & 34.2 & 83.3 & 52.8 & 30.5 & 8.0 & 57.6 & 2.5 & 55.1 & 7.1 & 66.1 & 26.7 & 39.4 & 38.6 & 94.5 & 78.8 & 15.7 & 7.9 & 80.3 & 45.3 & 35.0 & 12.3 & 62.3 & 14.3 & 47.9 & 14.0 & 75.3 & 37.2 & 38.1 & 17.4 \\
\midrule
\multicolumn{22}{l}{\textit{Episodic test-time adaptation methods}} \\
\textsc{TPT} & \colorbox{green!25}{\textbf{86.6}} & \colorbox{green!25}{54.9} & \colorbox{green!25}{31.7} & \colorbox{green!25}{27.8} & \colorbox{green!25}{\textbf{86.6}} & \colorbox{green!25}{61.3} & \colorbox{green!25}{25.3} & \colorbox{red!25}{12.7} & \colorbox{green!25}{\textbf{64.9}} & \colorbox{green!25}{\textbf{27.3}} & \colorbox{green!25}{\textbf{37.6}} & \textbf{6.5} & \colorbox{red!25}{51.3} & \colorbox{red!25}{13.9} & \colorbox{green!25}{37.4} & \colorbox{green!25}{26.6} & 95.0 & \colorbox{red!25}{23.9} & \colorbox{red!25}{71.1} & \colorbox{red!25}{14.4} & \colorbox{green!25}{\textbf{87.5}} & \colorbox{green!25}{60.7} & \colorbox{green!25}{26.8} & \colorbox{red!25}{20.2} & \colorbox{green!25}{\textbf{73.9}} & \colorbox{green!25}{\textbf{42.9}} & \colorbox{green!25}{\textbf{31.0}} & \colorbox{green!25}{10.5} & \colorbox{green!25}{\underline{78.0}} & \colorbox{green!25}{40.7} & 37.3 & 17.0 \\
\textsc{TPT} \textit{with ELRA} & 83.5 & 41.8 & 41.7 & 33.5 & 83.5 & 53.5 & 30.0 & 8.4 & 57.8 & 2.7 & 55.1 & 7.2 & 65.8 & 26.3 & 39.5 & 38.1 & 95.0 & \colorbox{red!25}{75.0} & \colorbox{red!25}{20.0} & \textbf{6.3} & 80.3 & 45.2 & 35.1 & 12.6 & 62.5 & 14.6 & 47.9 & 14.3 & 75.5 & 37.0 & 38.5 & 17.2 \\
\textsc{Zero} \\
\hspace{0.3cm} $\rho=0.1$ & 85.2 & \colorbox{green!25}{51.3} & \colorbox{green!25}{33.9} & \colorbox{green!25}{\underline{23.7}} & 85.2 & \colorbox{green!25}{55.4} & 29.8 & \colorbox{red!25}{15.5} & 57.9 & \colorbox{green!25}{5.1} & \colorbox{green!25}{52.8} & 8.0 & \colorbox{red!25}{59.9} & \colorbox{red!25}{19.5} & 40.4 & \colorbox{green!25}{34.5} & 96.4 & \colorbox{red!25}{39.6} & \colorbox{red!25}{56.8} & \colorbox{red!25}{16.2} & \colorbox{green!25}{86.2} & \colorbox{green!25}{58.4} & \colorbox{green!25}{27.8} & \colorbox{red!25}{18.0} & \colorbox{red!25}{54.3} & \colorbox{red!25}{8.2} & 46.1 & \colorbox{red!25}{19.0} & 75.0 & \colorbox{red!25}{33.9} & \colorbox{red!25}{41.1} & 19.3 \\
\hspace{0.3cm} $\rho=0.25$ & \colorbox{green!25}{85.5} & \colorbox{green!25}{53.1} & \colorbox{green!25}{32.4} & \colorbox{green!25}{24.9} & \colorbox{green!25}{85.5} & \colorbox{green!25}{58.2} & \colorbox{green!25}{27.4} & \colorbox{red!25}{16.2} & 59.0 & \colorbox{green!25}{5.3} & 53.6 & 8.4 & \colorbox{red!25}{58.1} & \colorbox{red!25}{17.6} & 40.5 & \colorbox{green!25}{33.6} & \colorbox{green!25}{96.8} & \colorbox{red!25}{50.8} & \colorbox{red!25}{46.0} & \colorbox{red!25}{11.7} & \colorbox{green!25}{86.6} & \colorbox{green!25}{59.1} & \colorbox{green!25}{27.5} & \colorbox{red!25}{19.9} & \colorbox{red!25}{54.3} & \colorbox{red!25}{5.3} & 49.0 & \colorbox{red!25}{19.4} & 75.1 & 35.6 & 39.5 & 19.2 \\
\hspace{0.3cm} $\rho=0.5$ & \colorbox{green!25}{85.7} & \colorbox{green!25}{53.6} & \colorbox{green!25}{32.1} & \colorbox{green!25}{27.5} & \colorbox{green!25}{85.7} & \colorbox{green!25}{60.2} & \colorbox{green!25}{25.5} & \colorbox{red!25}{15.5} & \colorbox{green!25}{59.6} & \colorbox{green!25}{4.8} & 54.9 & 8.2 & \colorbox{red!25}{57.4} & \colorbox{red!25}{16.9} & 40.5 & \colorbox{green!25}{33.2} & \colorbox{green!25}{96.8} & \colorbox{red!25}{48.1} & \colorbox{red!25}{48.7} & \colorbox{red!25}{18.4} & \colorbox{green!25}{87.0} & \colorbox{green!25}{60.4} & \colorbox{green!25}{26.6} & \colorbox{red!25}{20.5} & \colorbox{red!25}{54.9} & \colorbox{red!25}{4.7} & \colorbox{red!25}{50.2} & \colorbox{red!25}{19.6} & 75.3 & 35.5 & 39.8 & \colorbox{red!25}{20.4} \\
\hspace{0.3cm} $\rho=0.75$ & \colorbox{green!25}{85.8} & \colorbox{green!25}{56.9} & \colorbox{green!25}{28.9} & \colorbox{green!25}{24.7} & \colorbox{green!25}{85.8} & \colorbox{green!25}{61.7} & \colorbox{green!25}{24.1} & \colorbox{red!25}{14.0} & \colorbox{green!25}{59.9} & \colorbox{green!25}{4.6} & 55.4 & 7.5 & \colorbox{red!25}{57.1} & \colorbox{red!25}{17.1} & 40.0 & \colorbox{green!25}{32.3} & \colorbox{green!25}{96.9} & \colorbox{red!25}{51.8} & \colorbox{red!25}{45.0} & \colorbox{red!25}{15.4} & \colorbox{green!25}{87.0} & \colorbox{green!25}{60.4} & \colorbox{green!25}{26.6} & \colorbox{red!25}{19.9} & \colorbox{red!25}{54.8} & \colorbox{red!25}{4.3} & \colorbox{red!25}{50.5} & \colorbox{red!25}{19.0} & 75.4 & 36.7 & 38.6 & 19.0 \\
\hspace{0.3cm} $\rho=1.0$ & \colorbox{green!25}{85.9} & \colorbox{green!25}{57.0} & \colorbox{green!25}{28.9} & \colorbox{green!25}{25.9} & \colorbox{green!25}{\underline{85.9}} & \colorbox{green!25}{61.8} & \colorbox{green!25}{24.1} & \colorbox{red!25}{13.8} & \colorbox{green!25}{60.0} & 4.4 & 55.6 & 7.2 & \colorbox{red!25}{57.5} & \colorbox{red!25}{16.9} & 40.6 & \colorbox{green!25}{33.3} & \colorbox{green!25}{97.0} & \colorbox{red!25}{55.5} & \colorbox{red!25}{41.4} & \colorbox{red!25}{12.4} & \colorbox{green!25}{\underline{87.1}} & \colorbox{green!25}{59.8} & \colorbox{green!25}{27.3} & \colorbox{red!25}{20.8} & \colorbox{red!25}{54.9} & \colorbox{red!25}{4.2} & \colorbox{red!25}{50.7} & \colorbox{red!25}{19.1} & 75.5 & 37.1 & 38.4 & 18.9 \\
\midrule
\multicolumn{22}{l}{\textit{Episodic test-time debiasing methods}} \\
\textsc{OrthCali} \\
\hspace{0.3cm} $\lambda_\textnormal{orth}=0.0$ & 81.6 & \colorbox{green!25}{43.5} & \colorbox{green!25}{38.1} & \colorbox{red!25}{43.0} & 83.5 & 53.3 & 30.2 & 8.5 & 57.6 & 2.7 & 54.8 & 7.4 & \colorbox{green!25}{\textbf{75.7}} & \colorbox{green!25}{\textbf{43.9}} & \colorbox{green!25}{31.8} & 37.7 & 93.4 & \textbf{79.3} & \textbf{14.1} & \colorbox{red!25}{9.9} & 81.0 & 45.9 & 35.0 & 13.5 & 63.8 & 15.7 & 48.1 & 14.8 & 76.7 & \colorbox{green!25}{40.6} & \colorbox{green!25}{36.0} & 19.3 \\
\hspace{0.3cm} $\lambda_\textnormal{orth}=1.0$ & 81.7 & \colorbox{green!25}{44.0} & \colorbox{green!25}{37.7} & \colorbox{red!25}{42.5} & 83.3 & \colorbox{green!25}{55.0} & \colorbox{green!25}{28.2} & 7.1 & 57.7 & 2.7 & 54.9 & 7.5 & \colorbox{green!25}{\textbf{75.7}} & \colorbox{green!25}{\underline{43.8}} & \colorbox{green!25}{31.9} & 37.8 & 93.4 & 79.0 & 14.3 & 9.6 & 80.3 & 46.3 & 34.0 & 12.5 & 63.9 & 16.0 & 47.9 & 14.7 & 76.5 & \colorbox{green!25}{41.0} & \colorbox{green!25}{35.6} & 18.8 \\
\hspace{0.3cm} $\lambda_\textnormal{orth}=10.0$ & 81.6 & \colorbox{green!25}{42.5} & \colorbox{green!25}{39.2} & \colorbox{red!25}{44.0} & 82.2 & 52.1 & 30.1 & \colorbox{red!25}{14.3} & 57.6 & 2.5 & 55.1 & 7.4 & \colorbox{green!25}{\underline{75.4}} & \colorbox{green!25}{43.3} & \colorbox{green!25}{32.1} & 38.0 & 93.1 & 79.0 & \textbf{14.1} & \colorbox{red!25}{10.8} & \colorbox{red!25}{76.5} & \colorbox{red!25}{37.8} & \colorbox{red!25}{38.7} & 12.6 & 64.1 & \colorbox{green!25}{18.2} & \colorbox{green!25}{45.9} & 13.1 & 75.8 & \colorbox{green!25}{39.3} & 36.4 & \colorbox{red!25}{20.0} \\
\hspace{0.3cm} $\lambda_\textnormal{orth}=1000.0$ & \colorbox{red!25}{78.2} & \colorbox{green!25}{42.6} & \colorbox{green!25}{35.6} & \colorbox{red!25}{45.1} & \colorbox{red!25}{72.2} & \colorbox{red!25}{10.5} & \colorbox{red!25}{61.7} & \colorbox{red!25}{72.0} & 57.7 & 3.3 & 54.3 & 7.7 & \colorbox{green!25}{69.5} & \colorbox{green!25}{31.8} & 37.8 & 40.4 & \colorbox{red!25}{87.2} & \colorbox{red!25}{72.8} & 14.5 & \colorbox{red!25}{20.0} & \colorbox{red!25}{67.8} & \colorbox{red!25}{6.4} & \colorbox{red!25}{61.4} & \colorbox{red!25}{52.2} & \colorbox{green!25}{66.8} & \colorbox{green!25}{32.8} & \colorbox{green!25}{34.0} & \colorbox{green!25}{\underline{4.3}} & \colorbox{red!25}{71.3} & \colorbox{red!25}{28.6} & \colorbox{red!25}{42.7} & \colorbox{red!25}{34.5} \\
\hspace{0.3cm} $\lambda_\textnormal{orth}=100000$ & \colorbox{red!25}{76.7} & \colorbox{green!25}{44.2} & \colorbox{green!25}{32.5} & \colorbox{red!25}{45.6} & \colorbox{red!25}{71.5} & \colorbox{red!25}{9.8} & \colorbox{red!25}{61.7} & \colorbox{red!25}{73.1} & 57.4 & 2.9 & 54.6 & 7.8 & \colorbox{green!25}{68.1} & \colorbox{green!25}{29.7} & 38.4 & 40.2 & \colorbox{red!25}{86.3} & \colorbox{red!25}{72.1} & \underline{14.2} & \colorbox{red!25}{21.2} & \colorbox{red!25}{67.2} & \colorbox{red!25}{6.1} & \colorbox{red!25}{61.2} & \colorbox{red!25}{53.1} & \colorbox{green!25}{67.4} & \colorbox{green!25}{34.4} & \colorbox{green!25}{\underline{33.1}} & \colorbox{green!25}{\textbf{3.9}} & \colorbox{red!25}{70.7} & \colorbox{red!25}{28.4} & \colorbox{red!25}{42.2} & \colorbox{red!25}{35.0} \\
\midrule
\multicolumn{22}{l}{\textsc{FairTPT} and \textsc{FairTPT (MO)} \textit{with S loss and no ELRA}}\\
\hspace{0.3cm}$\lambda_{\textnormal{fair}} = 1$ & \colorbox{green!25}{85.5} & \colorbox{green!25}{57.0} & \colorbox{green!25}{28.5} & \colorbox{green!25}{26.1} & \colorbox{green!25}{85.6} & \colorbox{green!25}{61.9} & \colorbox{green!25}{23.7} & \colorbox{red!25}{14.3} & \colorbox{green!25}{63.6} & \colorbox{green!25}{\underline{20.8}} & \colorbox{green!25}{42.8} & \underline{6.7} & \colorbox{red!25}{57.9} & \colorbox{red!25}{17.5} & 40.5 & \colorbox{green!25}{33.4} & \colorbox{green!25}{\textbf{97.2}} & \colorbox{red!25}{56.7} & \colorbox{red!25}{40.5} & \colorbox{red!25}{12.5} & \colorbox{green!25}{86.7} & \colorbox{green!25}{58.6} & \colorbox{green!25}{28.1} & \colorbox{red!25}{21.5} & \colorbox{green!25}{\underline{73.2}} & \colorbox{green!25}{\underline{39.7}} & \colorbox{green!25}{33.4} & 12.1 & \colorbox{green!25}{\textbf{78.5}} & \colorbox{green!25}{44.6} & \colorbox{green!25}{33.9} & 18.1 \\
\hspace{0.3cm}$\lambda_{\textnormal{fair}} = 100$ & \colorbox{red!25}{77.1} & \colorbox{green!25}{\underline{61.7}} & \colorbox{green!25}{15.4} & \colorbox{green!25}{28.3} & \colorbox{red!25}{67.3} & 52.9 & \colorbox{green!25}{14.4} & \colorbox{red!25}{17.2} & 56.4 & 3.3 & \colorbox{green!25}{53.1} & \colorbox{red!25}{12.1} & 64.4 & 27.7 & \colorbox{green!25}{36.7} & \colorbox{green!25}{34.5} & 95.2 & \colorbox{red!25}{63.9} & \colorbox{red!25}{31.4} & \colorbox{red!25}{22.1} & \colorbox{red!25}{70.3} & \colorbox{green!25}{53.0} & \colorbox{green!25}{17.3} & \colorbox{red!25}{28.3} & \colorbox{green!25}{66.3} & \colorbox{green!25}{20.4} & 46.0 & 15.7 & \colorbox{red!25}{71.0} & \colorbox{green!25}{40.4} & \colorbox{green!25}{30.6} & \colorbox{red!25}{22.6} \\
\hspace{0.3cm}$\lambda_{\textnormal{fair}} = 5000$ & \colorbox{red!25}{75.7} & \colorbox{green!25}{60.6} & \colorbox{green!25}{\textbf{15.1}} & \colorbox{green!25}{26.9} & \colorbox{red!25}{65.0} & \colorbox{red!25}{49.6} & \colorbox{green!25}{15.4} & \colorbox{red!25}{17.7} & 56.1 & 3.3 & \colorbox{green!25}{52.8} & \colorbox{red!25}{12.1} & 64.6 & \colorbox{green!25}{29.6} & \colorbox{green!25}{35.0} & \colorbox{green!25}{32.5} & 94.5 & \colorbox{red!25}{60.7} & \colorbox{red!25}{33.8} & \colorbox{red!25}{25.9} & \colorbox{red!25}{68.7} & \colorbox{green!25}{51.2} & \colorbox{green!25}{17.5} & \colorbox{red!25}{28.6} & \colorbox{green!25}{66.0} & \colorbox{green!25}{20.4} & \colorbox{green!25}{45.7} & 15.7 & \colorbox{red!25}{70.1} & \colorbox{green!25}{39.3} & \colorbox{green!25}{30.7} & \colorbox{red!25}{22.8} \\
\hspace{0.3cm}$\lambda_{\textnormal{fair}} = \infty$ & \colorbox{red!25}{75.6} & \colorbox{green!25}{60.4} & \colorbox{green!25}{\underline{15.2}} & \colorbox{green!25}{27.1} & \colorbox{red!25}{64.9} & \colorbox{red!25}{49.5} & \colorbox{green!25}{15.4} & \colorbox{red!25}{17.7} & 56.0 & 3.5 & \colorbox{green!25}{52.5} & \colorbox{red!25}{12.1} & 64.7 & \colorbox{green!25}{29.9} & \colorbox{green!25}{34.8} & \colorbox{green!25}{32.4} & 94.5 & \colorbox{red!25}{59.0} & \colorbox{red!25}{35.4} & \colorbox{red!25}{27.6} & \colorbox{red!25}{68.4} & \colorbox{green!25}{51.2} & \colorbox{green!25}{17.3} & \colorbox{red!25}{28.8} & \colorbox{green!25}{66.0} & \colorbox{green!25}{20.8} & \colorbox{green!25}{45.3} & 15.4 & \colorbox{red!25}{70.0} & \colorbox{green!25}{39.2} & \colorbox{green!25}{30.8} & \colorbox{red!25}{23.0} \\
\hspace{0.3cm}$\lambda_{\textnormal{fair (mo)}} = 1$ & \colorbox{green!25}{85.6} & \colorbox{green!25}{57.0} & \colorbox{green!25}{28.6} & \colorbox{green!25}{26.1} & \colorbox{green!25}{85.6} & \colorbox{green!25}{61.3} & \colorbox{green!25}{24.3} & \colorbox{red!25}{15.0} & \colorbox{green!25}{63.6} & \colorbox{green!25}{20.5} & \colorbox{green!25}{43.1} & 6.9 & \colorbox{red!25}{57.9} & \colorbox{red!25}{17.2} & 40.6 & \colorbox{green!25}{33.5} & \colorbox{green!25}{\underline{97.1}} & \colorbox{red!25}{56.1} & \colorbox{red!25}{41.0} & \colorbox{red!25}{12.4} & \colorbox{green!25}{86.7} & \colorbox{green!25}{58.5} & \colorbox{green!25}{28.2} & \colorbox{red!25}{21.7} & \colorbox{green!25}{72.8} & \colorbox{green!25}{38.7} & \colorbox{green!25}{34.1} & 12.4 & \colorbox{green!25}{\textbf{78.5}} & \colorbox{green!25}{44.2} & \colorbox{green!25}{34.3} & 18.3 \\
\hspace{0.3cm}$\lambda_{\textnormal{fair (mo)}} = 100$ & \colorbox{red!25}{78.5} & \colorbox{green!25}{\textbf{62.5}} & \colorbox{green!25}{16.0} & \colorbox{green!25}{27.3} & \colorbox{red!25}{73.6} & \colorbox{green!25}{63.6} & \colorbox{green!25}{10.0} & \colorbox{red!25}{14.4} & 58.0 & \colorbox{green!25}{5.6} & \colorbox{green!25}{52.4} & \colorbox{red!25}{10.6} & 64.7 & 25.3 & 39.4 & 37.7 & 96.1 & \colorbox{red!25}{63.1} & \colorbox{red!25}{32.9} & \colorbox{red!25}{21.5} & \colorbox{red!25}{74.6} & \colorbox{green!25}{52.4} & \colorbox{green!25}{22.2} & \colorbox{red!25}{33.3} & \colorbox{green!25}{68.4} & \colorbox{green!25}{22.1} & 46.3 & \colorbox{red!25}{16.5} & 73.4 & \colorbox{green!25}{42.1} & \colorbox{green!25}{31.3} & \colorbox{red!25}{23.1} \\
\hspace{0.3cm}$\lambda_{\textnormal{fair (mo)}} = 5000$ & \colorbox{red!25}{77.4} & \colorbox{green!25}{61.2} & \colorbox{green!25}{16.2} & \colorbox{green!25}{28.3} & \colorbox{red!25}{71.9} & \colorbox{green!25}{60.2} & \colorbox{green!25}{11.7} & \colorbox{red!25}{15.4} & 57.8 & \colorbox{green!25}{5.2} & \colorbox{green!25}{52.6} & \colorbox{red!25}{10.8} & 64.7 & 26.4 & 38.3 & 36.8 & 95.4 & \colorbox{red!25}{63.1} & \colorbox{red!25}{32.2} & \colorbox{red!25}{22.2} & \colorbox{red!25}{72.6} & \colorbox{green!25}{51.7} & \colorbox{green!25}{21.0} & \colorbox{red!25}{34.4} & \colorbox{green!25}{68.5} & \colorbox{green!25}{22.7} & \colorbox{green!25}{45.8} & \colorbox{red!25}{16.3} & \colorbox{red!25}{72.6} & \colorbox{green!25}{41.5} & \colorbox{green!25}{31.1} & \colorbox{red!25}{23.4} \\
\midrule\multicolumn{22}{l}{\textsc{FairTPT} and \textsc{FairTPT (MO)} \textit{with TS loss and no ELRA}}\\
\hspace{0.3cm}$\lambda_{\textnormal{fair}} = 1$ & 85.2 & \colorbox{green!25}{56.4} & \colorbox{green!25}{28.8} & \colorbox{green!25}{27.5} & \colorbox{green!25}{85.6} & \colorbox{green!25}{61.3} & \colorbox{green!25}{24.3} & \colorbox{red!25}{15.3} & \colorbox{green!25}{\underline{64.0}} & \colorbox{green!25}{17.8} & \colorbox{green!25}{46.3} & 8.4 & \colorbox{red!25}{56.6} & \colorbox{red!25}{17.4} & 39.2 & \colorbox{green!25}{30.9} & \colorbox{green!25}{96.9} & \colorbox{red!25}{54.6} & \colorbox{red!25}{42.3} & \colorbox{red!25}{13.0} & \colorbox{green!25}{86.9} & \colorbox{green!25}{59.0} & \colorbox{green!25}{27.9} & \colorbox{red!25}{20.3} & \colorbox{green!25}{69.9} & \colorbox{green!25}{31.7} & \colorbox{green!25}{38.3} & 13.4 & \colorbox{green!25}{77.9} & \colorbox{green!25}{42.6} & \colorbox{green!25}{35.3} & 18.4 \\
\hspace{0.3cm}$\lambda_{\textnormal{fair}} = 100$ & \colorbox{red!25}{76.9} & \colorbox{green!25}{57.6} & \colorbox{green!25}{19.3} & 36.0 & \colorbox{red!25}{80.1} & \colorbox{green!25}{70.5} & \colorbox{green!25}{9.6} & \colorbox{red!25}{16.7} & \colorbox{red!25}{55.3} & 4.2 & \colorbox{green!25}{51.1} & \colorbox{red!25}{10.4} & \colorbox{red!25}{60.9} & \colorbox{green!25}{30.7} & \colorbox{green!25}{30.2} & \colorbox{green!25}{22.8} & 95.9 & \colorbox{red!25}{62.5} & \colorbox{red!25}{33.4} & \colorbox{red!25}{17.5} & 81.0 & \colorbox{green!25}{\underline{71.5}} & \colorbox{green!25}{9.5} & \colorbox{red!25}{18.6} & \colorbox{red!25}{58.6} & \colorbox{green!25}{20.7} & \colorbox{green!25}{37.9} & \colorbox{green!25}{6.5} & \colorbox{red!25}{72.7} & \colorbox{green!25}{\underline{45.4}} & \colorbox{green!25}{27.3} & 18.4 \\
\hspace{0.3cm}$\lambda_{\textnormal{fair}} = 5000$ & \colorbox{red!25}{75.8} & \colorbox{green!25}{57.6} & \colorbox{green!25}{18.2} & \colorbox{red!25}{36.2} & \colorbox{red!25}{78.8} & \colorbox{green!25}{68.3} & \colorbox{green!25}{10.5} & \colorbox{red!25}{17.2} & \colorbox{red!25}{54.2} & 3.5 & \colorbox{green!25}{50.7} & \colorbox{red!25}{11.8} & \colorbox{red!25}{60.8} & \colorbox{green!25}{32.1} & \colorbox{green!25}{\underline{28.7}} & \colorbox{green!25}{21.9} & 95.0 & \colorbox{red!25}{64.4} & \colorbox{red!25}{30.6} & \colorbox{red!25}{17.1} & 79.5 & \colorbox{green!25}{70.8} & \colorbox{green!25}{\textbf{8.7}} & \colorbox{red!25}{18.6} & \colorbox{red!25}{58.1} & \colorbox{green!25}{21.3} & \colorbox{green!25}{36.8} & \colorbox{green!25}{6.9} & \colorbox{red!25}{71.7} & \colorbox{green!25}{\underline{45.4}} & \colorbox{green!25}{\textbf{26.3}} & 18.5 \\
\hspace{0.3cm}$\lambda_{\textnormal{fair}} = \infty$ & \colorbox{red!25}{75.8} & \colorbox{green!25}{57.6} & \colorbox{green!25}{18.1} & \colorbox{red!25}{36.2} & \colorbox{red!25}{78.7} & \colorbox{green!25}{68.1} & \colorbox{green!25}{10.6} & \colorbox{red!25}{17.4} & \colorbox{red!25}{54.2} & 3.5 & \colorbox{green!25}{50.7} & \colorbox{red!25}{11.7} & \colorbox{red!25}{60.8} & \colorbox{green!25}{32.3} & \colorbox{green!25}{\textbf{28.5}} & \colorbox{green!25}{21.7} & 94.9 & \colorbox{red!25}{64.0} & \colorbox{red!25}{30.9} & \colorbox{red!25}{17.4} & 79.7 & \colorbox{green!25}{71.0} & \colorbox{green!25}{\textbf{8.7}} & \colorbox{red!25}{18.5} & \colorbox{red!25}{58.2} & \colorbox{green!25}{21.2} & \colorbox{green!25}{37.0} & \colorbox{green!25}{7.0} & \colorbox{red!25}{71.8} & \colorbox{green!25}{\underline{45.4}} & \colorbox{green!25}{\underline{26.4}} & 18.6 \\
\hspace{0.3cm}$\lambda_{\textnormal{fair (mo)}} = 1$ & 85.2 & \colorbox{green!25}{56.4} & \colorbox{green!25}{28.8} & \colorbox{green!25}{28.3} & \colorbox{green!25}{85.7} & \colorbox{green!25}{61.5} & \colorbox{green!25}{24.2} & \colorbox{red!25}{14.8} & \colorbox{green!25}{63.9} & \colorbox{green!25}{17.1} & \colorbox{green!25}{46.8} & 8.6 & \colorbox{red!25}{56.2} & \colorbox{red!25}{16.9} & 39.3 & \colorbox{green!25}{30.9} & \colorbox{green!25}{96.9} & \colorbox{red!25}{54.6} & \colorbox{red!25}{42.3} & \colorbox{red!25}{13.0} & \colorbox{green!25}{87.0} & \colorbox{green!25}{59.0} & \colorbox{green!25}{27.9} & \colorbox{red!25}{20.4} & \colorbox{green!25}{69.9} & \colorbox{green!25}{31.4} & \colorbox{green!25}{38.6} & 13.6 & \colorbox{green!25}{77.8} & \colorbox{green!25}{42.4} & \colorbox{green!25}{35.4} & 18.5 \\
\hspace{0.3cm}$\lambda_{\textnormal{fair (mo)}} = 100$ & \colorbox{red!25}{78.9} & \colorbox{green!25}{55.0} & \colorbox{green!25}{23.8} & \colorbox{red!25}{38.4} & 81.9 & \colorbox{green!25}{\textbf{75.3}} & \colorbox{green!25}{\textbf{6.5}} & \colorbox{red!25}{12.7} & 58.5 & \colorbox{green!25}{7.5} & \colorbox{green!25}{51.0} & 8.9 & \colorbox{red!25}{58.9} & 26.6 & \colorbox{green!25}{32.3} & \colorbox{green!25}{23.7} & 96.3 & \colorbox{red!25}{60.2} & \colorbox{red!25}{36.2} & \colorbox{red!25}{18.4} & 81.8 & \colorbox{green!25}{71.3} & \colorbox{green!25}{10.5} & \colorbox{red!25}{16.4} & 62.6 & \colorbox{green!25}{21.0} & \colorbox{green!25}{41.6} & \colorbox{green!25}{9.2} & 74.1 & \colorbox{green!25}{45.3} & \colorbox{green!25}{28.8} & 18.2 \\
\hspace{0.3cm}$\lambda_{\textnormal{fair (mo)}} = 5000$ & \colorbox{red!25}{78.1} & \colorbox{green!25}{57.3} & \colorbox{green!25}{20.7} & 35.9 & \colorbox{red!25}{81.2} & \colorbox{green!25}{73.8} & \colorbox{green!25}{7.3} & \colorbox{red!25}{13.4} & 58.1 & \colorbox{green!25}{7.5} & \colorbox{green!25}{50.6} & 8.7 & \colorbox{red!25}{59.1} & 27.6 & \colorbox{green!25}{31.5} & \colorbox{green!25}{23.5} & 95.6 & \colorbox{red!25}{61.9} & \colorbox{red!25}{33.7} & \colorbox{red!25}{19.6} & 80.7 & \colorbox{green!25}{\textbf{71.7}} & \colorbox{green!25}{\underline{9.0}} & \colorbox{red!25}{16.2} & 62.5 & \colorbox{green!25}{21.4} & \colorbox{green!25}{41.2} & \colorbox{green!25}{9.0} & 73.6 & \colorbox{green!25}{\textbf{45.9}} & \colorbox{green!25}{27.7} & 18.0 \\
\midrule\multicolumn{22}{l}{\textsc{FairTPT} and \textsc{FairTPT (MO)} \textit{with Super TS loss and no ELRA}}\\
\hspace{0.3cm}$\lambda_{\textnormal{fair}} = 1$ & \colorbox{green!25}{\underline{86.0}} & \colorbox{green!25}{56.8} & \colorbox{green!25}{29.2} & \colorbox{green!25}{\textbf{23.4}} & \colorbox{green!25}{\underline{85.9}} & \colorbox{green!25}{59.7} & \colorbox{green!25}{26.2} & \colorbox{red!25}{16.4} & \colorbox{green!25}{61.4} & \colorbox{green!25}{10.0} & \colorbox{green!25}{51.4} & 7.0 & \colorbox{red!25}{55.7} & \colorbox{red!25}{16.3} & 39.4 & \colorbox{green!25}{30.8} & \colorbox{green!25}{96.8} & \colorbox{red!25}{51.2} & \colorbox{red!25}{45.6} & \colorbox{red!25}{15.6} & \colorbox{green!25}{87.0} & \colorbox{green!25}{59.0} & \colorbox{green!25}{28.0} & \colorbox{red!25}{21.0} & \colorbox{red!25}{59.8} & 12.4 & 47.3 & \colorbox{red!25}{17.2} & 76.1 & 37.9 & 38.2 & 18.8 \\
\hspace{0.3cm}$\lambda_{\textnormal{fair}} = 100$ & \colorbox{red!25}{74.6} & \colorbox{green!25}{45.6} & \colorbox{green!25}{29.0} & \colorbox{red!25}{41.5} & \colorbox{red!25}{80.6} & \colorbox{green!25}{71.5} & \colorbox{green!25}{9.2} & \colorbox{red!25}{17.4} & 56.7 & \colorbox{green!25}{6.5} & \colorbox{green!25}{50.3} & \colorbox{red!25}{20.4} & \colorbox{red!25}{56.7} & 25.9 & \colorbox{green!25}{30.7} & \colorbox{green!25}{\underline{21.4}} & \colorbox{red!25}{81.5} & \colorbox{red!25}{22.1} & \colorbox{red!25}{59.4} & \colorbox{red!25}{48.7} & \colorbox{green!25}{82.9} & \colorbox{green!25}{66.7} & \colorbox{green!25}{16.2} & 11.5 & \colorbox{green!25}{68.3} & \colorbox{red!25}{11.6} & \colorbox{red!25}{56.6} & \colorbox{red!25}{27.0} & \colorbox{red!25}{71.6} & 35.7 & \colorbox{green!25}{35.9} & \colorbox{red!25}{26.8} \\
\hspace{0.3cm}$\lambda_{\textnormal{fair}} = 5000$ & \colorbox{red!25}{72.0} & \colorbox{green!25}{43.6} & \colorbox{green!25}{28.4} & \colorbox{red!25}{47.1} & \colorbox{red!25}{77.2} & \colorbox{green!25}{65.0} & \colorbox{green!25}{12.2} & \colorbox{red!25}{21.2} & \colorbox{red!25}{50.2} & \colorbox{green!25}{4.5} & \colorbox{green!25}{45.7} & \colorbox{red!25}{25.1} & \colorbox{red!25}{56.5} & 27.4 & \colorbox{green!25}{29.1} & \colorbox{green!25}{28.0} & \colorbox{red!25}{68.4} & \colorbox{red!25}{15.8} & \colorbox{red!25}{52.7} & \colorbox{red!25}{62.6} & 79.8 & \colorbox{green!25}{65.9} & \colorbox{green!25}{13.9} & 12.0 & \colorbox{green!25}{64.6} & \colorbox{red!25}{9.0} & \colorbox{red!25}{55.6} & \colorbox{red!25}{32.7} & \colorbox{red!25}{67.0} & \colorbox{red!25}{33.0} & \colorbox{green!25}{33.9} & \colorbox{red!25}{32.7} \\
\hspace{0.3cm}$\lambda_{\textnormal{fair}} = \infty$ & \colorbox{red!25}{71.9} & \colorbox{green!25}{43.5} & \colorbox{green!25}{28.3} & \colorbox{red!25}{46.5} & \colorbox{red!25}{77.0} & \colorbox{green!25}{64.7} & \colorbox{green!25}{12.4} & \colorbox{red!25}{21.4} & \colorbox{red!25}{49.7} & 4.0 & \colorbox{green!25}{45.7} & \colorbox{red!25}{25.8} & \colorbox{red!25}{56.3} & 27.3 & \colorbox{green!25}{29.1} & \colorbox{green!25}{27.9} & \colorbox{red!25}{68.0} & \colorbox{red!25}{15.4} & \colorbox{red!25}{52.5} & \colorbox{red!25}{62.9} & 79.8 & \colorbox{green!25}{66.3} & \colorbox{green!25}{13.5} & 12.5 & \colorbox{green!25}{64.4} & \colorbox{red!25}{9.1} & \colorbox{red!25}{55.3} & \colorbox{red!25}{33.1} & \colorbox{red!25}{66.7} & \colorbox{red!25}{32.9} & \colorbox{green!25}{33.8} & \colorbox{red!25}{32.9} \\
\hspace{0.3cm}$\lambda_{\textnormal{fair (mo)}} = 1$ & \colorbox{green!25}{\underline{86.0}} & \colorbox{green!25}{56.2} & \colorbox{green!25}{29.7} & \colorbox{green!25}{24.0} & \colorbox{green!25}{\underline{85.9}} & \colorbox{green!25}{59.8} & \colorbox{green!25}{26.1} & \colorbox{red!25}{16.3} & \colorbox{green!25}{61.4} & \colorbox{green!25}{9.8} & \colorbox{green!25}{51.5} & 7.1 & \colorbox{red!25}{55.6} & \colorbox{red!25}{15.9} & 39.7 & \colorbox{green!25}{31.2} & \colorbox{green!25}{96.8} & \colorbox{red!25}{51.2} & \colorbox{red!25}{45.6} & \colorbox{red!25}{15.6} & \colorbox{green!25}{87.0} & \colorbox{green!25}{58.8} & \colorbox{green!25}{28.2} & \colorbox{red!25}{21.2} & \colorbox{red!25}{59.8} & \colorbox{red!25}{12.3} & 47.5 & \colorbox{red!25}{17.2} & 76.1 & 37.7 & 38.3 & 18.9 \\
\hspace{0.3cm}$\lambda_{\textnormal{fair (mo)}} = 100$ & 81.8 & \colorbox{green!25}{47.3} & \colorbox{green!25}{34.5} & \colorbox{green!25}{29.8} & 84.1 & \colorbox{green!25}{74.4} & \colorbox{green!25}{9.7} & \colorbox{red!25}{12.2} & 59.2 & \colorbox{green!25}{17.2} & \colorbox{green!25}{42.0} & \colorbox{red!25}{13.4} & \colorbox{red!25}{51.8} & \colorbox{red!25}{16.8} & \colorbox{green!25}{35.0} & \colorbox{green!25}{22.6} & \colorbox{red!25}{88.7} & \colorbox{red!25}{24.1} & \colorbox{red!25}{64.6} & \colorbox{red!25}{46.6} & \colorbox{green!25}{84.0} & \colorbox{green!25}{64.0} & \colorbox{green!25}{20.1} & \colorbox{green!25}{9.2} & 63.8 & \colorbox{green!25}{29.6} & \colorbox{green!25}{34.1} & \colorbox{green!25}{10.8} & 73.4 & 39.1 & \colorbox{green!25}{34.3} & \colorbox{red!25}{20.7} \\
\hspace{0.3cm}$\lambda_{\textnormal{fair (mo)}} = 5000$ & \colorbox{red!25}{79.9} & \colorbox{green!25}{44.8} & \colorbox{green!25}{35.1} & \colorbox{green!25}{31.5} & 81.7 & \colorbox{green!25}{\underline{74.9}} & \colorbox{green!25}{\underline{6.8}} & \colorbox{red!25}{11.4} & 56.8 & \colorbox{green!25}{18.0} & \colorbox{green!25}{\underline{38.8}} & \colorbox{red!25}{17.2} & \colorbox{red!25}{51.2} & \colorbox{red!25}{17.2} & \colorbox{green!25}{34.0} & \colorbox{green!25}{\textbf{21.3}} & \colorbox{red!25}{77.3} & \colorbox{red!25}{16.9} & \colorbox{red!25}{60.4} & \colorbox{red!25}{63.1} & \colorbox{green!25}{82.3} & \colorbox{green!25}{65.8} & \colorbox{green!25}{16.5} & \colorbox{green!25}{\textbf{7.5}} & 63.6 & \colorbox{green!25}{26.4} & \colorbox{green!25}{37.2} & 12.4 & \colorbox{red!25}{70.4} & 37.7 & \colorbox{green!25}{32.7} & \colorbox{red!25}{23.5} \\
\midrule
\multicolumn{22}{l}{\textsc{FairTPT} and \textsc{FairTPT (MO)} \textit{with S loss}}\\
\hspace{0.3cm}$\lambda_{\textnormal{fair}} = 1$ & 83.5 & \colorbox{green!25}{44.2} & \colorbox{green!25}{39.3} & \colorbox{green!25}{32.0} & 83.4 & 53.9 & 29.5 & 8.2 & 57.7 & 2.7 & 55.0 & 7.3 & 66.0 & 26.3 & 39.6 & 38.6 & 94.9 & 78.9 & 16.0 & 7.8 & 80.4 & 45.7 & 34.7 & 12.3 & 62.4 & 14.3 & 48.1 & 14.5 & 75.5 & 38.0 & 37.5 & 17.2 \\
\hspace{0.3cm}$\lambda_{\textnormal{fair}} = 100$ & 83.0 & \colorbox{green!25}{50.0} & \colorbox{green!25}{33.0} & \colorbox{green!25}{31.0} & 83.1 & \colorbox{green!25}{60.0} & \colorbox{green!25}{23.1} & 7.7 & 57.2 & 4.0 & 53.2 & 6.8 & 67.1 & 27.8 & 39.3 & 39.1 & 94.6 & 78.2 & 16.4 & \underline{7.6} & 80.0 & \colorbox{green!25}{49.5} & \colorbox{green!25}{30.5} & 12.6 & 62.2 & 15.8 & 46.4 & 13.6 & 75.3 & \colorbox{green!25}{40.8} & \colorbox{green!25}{34.6} & 16.9 \\
\hspace{0.3cm}$\lambda_{\textnormal{fair}} = 5000$ & 82.7 & \colorbox{green!25}{50.4} & \colorbox{green!25}{32.3} & \colorbox{green!25}{29.1} & 83.3 & \colorbox{green!25}{62.7} & \colorbox{green!25}{20.6} & 6.6 & 57.2 & 3.8 & 53.4 & 7.2 & 67.1 & 27.6 & 39.5 & 39.4 & 94.6 & \colorbox{red!25}{75.4} & \colorbox{red!25}{19.2} & \colorbox{red!25}{11.5} & 80.2 & \colorbox{green!25}{49.2} & \colorbox{green!25}{31.0} & 12.6 & 62.0 & 15.4 & 46.5 & 13.6 & 75.3 & \colorbox{green!25}{40.7} & \colorbox{green!25}{34.6} & 17.1 \\
\hspace{0.3cm}$\lambda_{\textnormal{fair}} = \infty$ & 82.7 & \colorbox{green!25}{48.9} & \colorbox{green!25}{33.8} & \colorbox{green!25}{31.0} & 83.2 & \colorbox{green!25}{63.6} & \colorbox{green!25}{19.6} & \colorbox{green!25}{\underline{5.4}} & 57.4 & 3.8 & 53.6 & 7.4 & 67.2 & 27.3 & 39.8 & 39.9 & 94.8 & \colorbox{red!25}{74.3} & \colorbox{red!25}{20.4} & \colorbox{red!25}{11.9} & 80.1 & \colorbox{green!25}{49.2} & \colorbox{green!25}{30.9} & 12.4 & 62.0 & 15.4 & 46.7 & 13.5 & 75.3 & \colorbox{green!25}{40.4} & \colorbox{green!25}{35.0} & 17.4 \\
\hspace{0.3cm}$\lambda_{\textnormal{fair (mo)}} = 1$ & 83.5 & \colorbox{green!25}{43.6} & \colorbox{green!25}{39.9} & \colorbox{green!25}{32.2} & 83.4 & 53.7 & 29.7 & 8.2 & 57.7 & 2.7 & 55.0 & 7.3 & 65.9 & 26.1 & 39.8 & 38.8 & 94.9 & 78.6 & 16.3 & 8.1 & 80.3 & 45.5 & 34.8 & 12.3 & 62.5 & 14.3 & 48.1 & 14.5 & 75.5 & 37.8 & 37.7 & 17.3 \\
\hspace{0.3cm}$\lambda_{\textnormal{fair (mo)}} = 100$ & 83.4 & \colorbox{green!25}{49.0} & \colorbox{green!25}{34.4} & \colorbox{green!25}{30.5} & 83.0 & \colorbox{green!25}{60.0} & \colorbox{green!25}{23.0} & 7.7 & 57.2 & 3.9 & 53.3 & 7.1 & 67.1 & 28.0 & 39.1 & 38.6 & 94.8 & 77.5 & 17.4 & 7.8 & 80.2 & \colorbox{green!25}{51.8} & \colorbox{green!25}{28.5} & 11.8 & 63.2 & \colorbox{green!25}{18.3} & \colorbox{green!25}{44.9} & 12.9 & 75.6 & \colorbox{green!25}{41.2} & \colorbox{green!25}{34.4} & 16.6 \\
\hspace{0.3cm}$\lambda_{\textnormal{fair (mo)}} = 5000$ & 83.0 & \colorbox{green!25}{48.8} & \colorbox{green!25}{34.2} & \colorbox{green!25}{28.9} & 82.6 & \colorbox{green!25}{59.3} & \colorbox{green!25}{23.3} & 7.6 & 57.3 & 4.1 & 53.2 & \underline{6.7} & 67.1 & 27.8 & 39.3 & 39.7 & 94.6 & 78.5 & 16.1 & 7.9 & 80.0 & \colorbox{green!25}{50.1} & \colorbox{green!25}{29.9} & 12.6 & 63.0 & \colorbox{green!25}{18.4} & \colorbox{green!25}{44.6} & 12.5 & 75.4 & \colorbox{green!25}{41.0} & \colorbox{green!25}{34.4} & 16.6 \\
\midrule
\multicolumn{22}{l}{\textsc{FairTPT} and \textsc{FairTPT (MO)} \textit{with TS loss}}\\
\hspace{0.3cm}$\lambda_{\textnormal{fair}} = 1$ & 83.6 & \colorbox{green!25}{43.3} & \colorbox{green!25}{40.3} & 32.5 & 83.6 & 53.7 & 29.9 & 8.8 & 57.7 & 2.7 & 55.0 & 7.3 & 65.7 & 26.1 & 39.6 & 38.4 & 94.8 & 78.2 & 16.6 & 8.2 & 80.4 & 45.8 & 34.6 & 11.7 & 62.5 & 14.3 & 48.2 & 14.5 & 75.5 & 37.7 & 37.7 & 17.3 \\
\hspace{0.3cm}$\lambda_{\textnormal{fair}} = 100$ & 83.9 & \colorbox{green!25}{44.3} & \colorbox{green!25}{39.5} & 33.6 & 83.6 & \colorbox{green!25}{60.5} & \colorbox{green!25}{23.0} & \colorbox{green!25}{6.0} & 57.4 & 2.7 & 54.7 & 6.8 & 65.5 & 26.7 & 38.8 & 37.4 & 94.2 & 77.9 & 16.3 & 8.7 & 80.5 & \colorbox{green!25}{50.6} & \colorbox{green!25}{29.9} & \colorbox{green!25}{8.4} & 61.7 & 14.3 & 47.3 & 13.8 & 75.2 & \colorbox{green!25}{39.6} & \colorbox{green!25}{35.7} & 16.4 \\
\hspace{0.3cm}$\lambda_{\textnormal{fair}} = 5000$ & 83.6 & \colorbox{green!25}{45.9} & \colorbox{green!25}{37.7} & \colorbox{green!25}{29.4} & 83.4 & \colorbox{green!25}{60.5} & \colorbox{green!25}{22.9} & 6.4 & 57.4 & 2.9 & 54.6 & 6.8 & 65.6 & 26.5 & 39.1 & 38.1 & 94.5 & \underline{79.2} & 15.4 & 8.0 & 80.5 & \colorbox{green!25}{51.1} & \colorbox{green!25}{29.4} & \colorbox{green!25}{8.8} & 61.8 & 14.9 & 46.9 & 13.4 & 75.3 & \colorbox{green!25}{40.1} & \colorbox{green!25}{35.1} & 15.8 \\
\hspace{0.3cm}$\lambda_{\textnormal{fair}} = \infty$ & 83.2 & \colorbox{green!25}{43.8} & \colorbox{green!25}{39.4} & \colorbox{green!25}{31.2} & 83.6 & \colorbox{green!25}{61.4} & \colorbox{green!25}{22.1} & \colorbox{green!25}{5.8} & 57.6 & 2.9 & 54.7 & 6.9 & 65.5 & 26.6 & 38.9 & 37.5 & 94.2 & 77.4 & 16.9 & 8.9 & 80.6 & \colorbox{green!25}{51.5} & \colorbox{green!25}{29.1} & \colorbox{green!25}{8.0} & 62.0 & 14.9 & 47.1 & 13.4 & 75.3 & \colorbox{green!25}{39.8} & \colorbox{green!25}{35.5} & 16.0 \\
\hspace{0.3cm}$\lambda_{\textnormal{fair (mo)}} = 1$ & 83.6 & \colorbox{green!25}{43.8} & \colorbox{green!25}{39.8} & \colorbox{green!25}{32.0} & 83.5 & 53.7 & 29.8 & 8.5 & 57.7 & 2.7 & 55.0 & 7.3 & 65.5 & 25.6 & 39.9 & 38.7 & 94.8 & 78.2 & 16.6 & 8.2 & 80.5 & 45.6 & 34.9 & 11.9 & 62.4 & 14.1 & 48.3 & 14.5 & 75.4 & 37.7 & 37.8 & 17.3 \\
\hspace{0.3cm}$\lambda_{\textnormal{fair (mo)}} = 100$ & 83.2 & \colorbox{green!25}{43.3} & \colorbox{green!25}{39.9} & 33.7 & 83.4 & \colorbox{green!25}{61.7} & \colorbox{green!25}{21.6} & \colorbox{green!25}{\textbf{5.0}} & 57.6 & 4.2 & 53.3 & 7.0 & 64.7 & 25.6 & 39.1 & 37.4 & 94.7 & \colorbox{red!25}{75.8} & \colorbox{red!25}{18.9} & 9.7 & 80.8 & \colorbox{green!25}{51.9} & \colorbox{green!25}{28.9} & \colorbox{green!25}{\underline{7.7}} & 62.9 & \colorbox{green!25}{16.7} & 46.2 & 13.7 & 75.3 & \colorbox{green!25}{39.9} & \colorbox{green!25}{35.4} & 16.3 \\
\hspace{0.3cm}$\lambda_{\textnormal{fair (mo)}} = 5000$ & 83.3 & \colorbox{green!25}{44.0} & \colorbox{green!25}{39.3} & \colorbox{green!25}{32.2} & 83.6 & \colorbox{green!25}{61.3} & \colorbox{green!25}{22.2} & \colorbox{green!25}{\underline{5.4}} & 57.7 & \colorbox{green!25}{4.8} & \colorbox{green!25}{52.9} & 7.1 & 64.6 & 26.0 & 38.6 & 36.7 & 94.4 & \colorbox{red!25}{76.7} & \colorbox{red!25}{17.7} & \colorbox{red!25}{10.1} & 80.7 & \colorbox{green!25}{51.4} & \colorbox{green!25}{29.3} & \colorbox{green!25}{7.8} & 62.8 & \colorbox{green!25}{16.4} & 46.4 & 14.2 & 75.3 & \colorbox{green!25}{40.1} & \colorbox{green!25}{35.2} & 16.2 \\
\midrule
\multicolumn{22}{l}{\textsc{FairTPT} and \textsc{FairTPT (MO)} \textit{with Super TS loss}}\\
\hspace{0.3cm}$\lambda_{\textnormal{fair}} = 1$ & 83.5 & \colorbox{green!25}{43.2} & \colorbox{green!25}{40.3} & 32.3 & 83.5 & 53.3 & 30.2 & 8.8 & 57.9 & 2.5 & 55.3 & 7.2 & 65.7 & 26.1 & 39.6 & 38.3 & 94.9 & 78.5 & 16.4 & 8.1 & 80.4 & 45.5 & 34.9 & 12.2 & 62.2 & 13.7 & 48.5 & 14.4 & 75.4 & 37.5 & 37.9 & 17.3 \\
\hspace{0.3cm}$\lambda_{\textnormal{fair}} = 100$ & 83.6 & \colorbox{green!25}{44.9} & \colorbox{green!25}{38.6} & \colorbox{green!25}{30.6} & 83.5 & \colorbox{green!25}{58.6} & \colorbox{green!25}{24.9} & 8.0 & 57.6 & \colorbox{green!25}{6.6} & \colorbox{green!25}{51.0} & 6.9 & 65.0 & 26.2 & 38.8 & 36.7 & 93.9 & \colorbox{red!25}{76.3} & 17.6 & 9.7 & 80.8 & \colorbox{green!25}{50.0} & \colorbox{green!25}{30.8} & \colorbox{green!25}{9.5} & 62.5 & \colorbox{green!25}{20.9} & \colorbox{green!25}{41.6} & \colorbox{green!25}{9.4} & 75.3 & \colorbox{green!25}{40.5} & \colorbox{green!25}{34.8} & 15.8 \\
\hspace{0.3cm}$\lambda_{\textnormal{fair}} = 5000$ & 83.3 & \colorbox{green!25}{43.2} & \colorbox{green!25}{40.1} & 32.5 & 83.5 & \colorbox{green!25}{59.6} & \colorbox{green!25}{23.9} & 7.8 & 57.5 & \colorbox{green!25}{7.0} & \colorbox{green!25}{50.4} & 6.8 & 65.4 & 27.2 & 38.2 & \colorbox{green!25}{36.3} & 93.9 & 76.9 & 17.0 & 8.6 & 80.7 & \colorbox{green!25}{49.7} & \colorbox{green!25}{31.1} & \colorbox{green!25}{9.1} & 62.3 & \colorbox{green!25}{21.4} & \colorbox{green!25}{41.0} & \colorbox{green!25}{8.4} & 75.2 & \colorbox{green!25}{40.7} & \colorbox{green!25}{34.5} & \underline{15.6} \\
\hspace{0.3cm}$\lambda_{\textnormal{fair}} = \infty$ & 83.3 & \colorbox{green!25}{45.0} & \colorbox{green!25}{38.4} & \colorbox{green!25}{28.6} & 83.4 & \colorbox{green!25}{58.8} & \colorbox{green!25}{24.6} & 7.9 & 57.6 & \colorbox{green!25}{6.9} & \colorbox{green!25}{50.7} & \underline{6.7} & 65.2 & 27.3 & 37.9 & \colorbox{green!25}{35.7} & 93.9 & \colorbox{red!25}{76.4} & 17.5 & 9.2 & 80.9 & \colorbox{green!25}{50.4} & \colorbox{green!25}{30.5} & \colorbox{green!25}{8.9} & 62.5 & \colorbox{green!25}{21.1} & \colorbox{green!25}{41.4} & \colorbox{green!25}{8.8} & 75.3 & \colorbox{green!25}{40.8} & \colorbox{green!25}{34.4} & \colorbox{green!25}{\textbf{15.1}} \\
\hspace{0.3cm}$\lambda_{\textnormal{fair (mo)}} = 1$ & 83.4 & \colorbox{green!25}{43.2} & \colorbox{green!25}{40.3} & \colorbox{green!25}{31.1} & 83.6 & 53.5 & 30.1 & 8.7 & 57.9 & 2.5 & 55.3 & 7.2 & 65.8 & 26.0 & 39.8 & 38.6 & 94.9 & 78.5 & 16.4 & 8.1 & 80.4 & 45.2 & 35.2 & 12.7 & 62.1 & 13.5 & 48.6 & 14.5 & 75.4 & 37.5 & 38.0 & 17.3 \\
\hspace{0.3cm}$\lambda_{\textnormal{fair (mo)}} = 100$ & 83.5 & \colorbox{green!25}{43.9} & \colorbox{green!25}{39.7} & \colorbox{green!25}{30.7} & 83.8 & \colorbox{green!25}{59.5} & \colorbox{green!25}{24.4} & 7.3 & 57.7 & 2.9 & 54.8 & 7.2 & 64.6 & 25.0 & 39.6 & 37.8 & 94.0 & \colorbox{red!25}{76.4} & 17.6 & 9.6 & 80.8 & \colorbox{green!25}{48.8} & \colorbox{green!25}{31.9} & \colorbox{green!25}{9.6} & 62.1 & \colorbox{green!25}{16.4} & \colorbox{green!25}{45.7} & 12.6 & 75.2 & 39.0 & 36.2 & 16.4 \\
\hspace{0.3cm}$\lambda_{\textnormal{fair (mo)}} = 5000$ & 83.5 & \colorbox{green!25}{43.7} & \colorbox{green!25}{39.8} & \colorbox{green!25}{31.3} & 83.8 & \colorbox{green!25}{60.0} & \colorbox{green!25}{23.8} & 7.1 & 57.7 & 3.6 & 54.1 & 7.4 & 64.3 & \colorbox{red!25}{24.3} & 39.9 & 37.8 & 93.9 & 78.9 & 15.0 & 7.7 & 80.6 & \colorbox{green!25}{49.1} & \colorbox{green!25}{31.5} & \colorbox{green!25}{9.4} & 62.0 & \colorbox{green!25}{17.0} & \colorbox{green!25}{45.0} & \colorbox{green!25}{11.7} & 75.1 & \colorbox{green!25}{39.5} & \colorbox{green!25}{35.6} & 16.0 \\

\bottomrule
\end{tabular}
}
\caption{Overall and subgroup-level performance evaluation of all considered methods on additional dataset-attribute configurations. We report mean over 5 random seeds and mean aggregation over all equally sized datasets. Best results are in \textbf{bold}, second best \underline{underlined}. The values that improve upon \textsc{Zero-Shot} by more than 2.0 percentage points are highlighted in \colorbox{green!25}{green}, while degradations greater than 2.0 percentage points are shown in \colorbox{red!25}{red}.}
\label{tab:additional_full_results}
\end{table*}

\clearpage
\subsection{Additional Discussion on \textsc{FairTPT}}
\label{app:sub-sec-fair-tpt}

\paragraph{Effectiveness Validation via ASI and ATC Metrics.} To evaluate the fairness and accuracy trade-off introduced by \textsc{FairTPT}, we report two metrics:
\begin{itemize}
    \item Average Sensitive Indifference (ASI): The average normalized entropy of sensitive attribute predictions over the dataset. A higher ASI indicates reduced reliance on sensitive attributes.
    \item Average Target Confidence (ATC): Defined as $1 -$ average normalized entropy of target attribute predictions over the dataset. A higher ATC reflects greater confidence in target predictions.
\end{itemize}
Table~\ref{tab:asi} and Table~\ref{tab:atc} present ASI and ATC values for all considered methods, averaged over five random seeds and aggregated across equally sized datasets. For \textsc{FairTPT} and \textsc{FairTPT (MO)}, we set $\lambda_\textnormal{fair} = 100$ and $\lambda_\textnormal{fair (mo)} = 100$, respectively. Hyperparameter details for all methods are provided in Table~\ref{tab:hyperparam}. We observe that \textsc{FairTPT} and \textsc{FairTPT (MO)} significantly increase ASI with minimal impact on ATC.

\begin{table}[h]
\centering
\resizebox{0.7\linewidth}{!}{
\begin{tabular}{lcccc}
\toprule
\textbf{Dataset} & \textbf{Before Update} & \textbf{\textsc{FairTPT}} & \textbf{\textsc{FairTPT (MO)}} & \textbf{\textsc{TPT}} \\
\midrule
\textsc{FairFace} (Gender $\times$ Race) & 42.1 & 53.9 & 57.4 & 39.5 \\
\textsc{CelebA} (Hair color $\times$ Gender) & 15.3 & 41.4 & 46.4 & 14.3 \\
\textsc{CelebA} (Smiling $\times$ Gender) & 15.3 & 40.1 & 50.9 & 11.9 \\
\textsc{WaterBirds} (Type $\times$ Background) & 56.0 & 70.9 & 71.1 & 51.2 \\
\textsc{UTKFace} (Age $\times$ Race) & 45.4 & 53.3 & 56.1 & 35.3 \\
\textsc{UTKFace} (Gender $\times$ Race) & 45.4 & 55.6 & 57.2 & 29.9 \\
\midrule
\textbf{Average} & 36.6 & 52.5 & 56.5 & 30.3 \\
\bottomrule
\end{tabular}
}
\caption{ASI evaluation across datasets. Mean over 5 seeds and aggregated results.}
\label{tab:asi}
\end{table}

\begin{table}[h]
\centering
\resizebox{0.7\linewidth}{!}{
\begin{tabular}{lcccc}
\toprule
\textbf{Dataset} & \textbf{Before Update} & \textbf{\textsc{FairTPT}} & \textbf{\textsc{FairTPT (MO)}} & \textbf{\textsc{TPT}} \\
\midrule
\textsc{FairFace} (Gender $\times$ Race) & 66.4 & 65.0 & 64.4 & 88.4 \\
\textsc{CelebA} (Hair color $\times$ Gender) & 57.5 & 53.9 & 55.1 & 96.0 \\
\textsc{CelebA} (Smiling $\times$ Gender) & 45.2 & 44.2 & 43.4 & 80.2 \\
\textsc{WaterBirds} (Type $\times$ Background) & 62.1 & 59.7 & 60.7 & 80.0 \\
\textsc{UTKFace} (Age $\times$ Race) & 46.1 & 45.2 & 44.8 & 85.0 \\
\textsc{UTKFace} (Gender $\times$ Race) & 71.6 & 70.6 & 69.3 & 91.4 \\
\midrule
\textbf{Average} & 58.1 & 56.4 & 56.3 & 86.8 \\
\bottomrule
\end{tabular}
}
\caption{ATC evaluation across datasets. Mean over 5 seeds and aggregated results.}
\label{tab:atc}
\end{table}

\paragraph{Hyperparameter Tuning and Runtime Analysis.} \textsc{FairTPT} exhibits inference-time cost comparable to \textsc{TPT} (prompt-tuning strategies). Table~\ref{tab:runtime} reports runtime per image (in seconds) under identical hardware (Nvidia H100 GPU) for \textsc{FairFace} (Gender $\times$ Race):
\begin{table}[h]
\centering
\resizebox{0.5\linewidth}{!}{
\begin{tabular}{lcccc}
\toprule
\textbf{Method} & \textbf{\textsc{TPT}} & \textbf{\textsc{Zero}} & \textbf{\textsc{FairTPT}} & \textbf{\textsc{FairTPT (MO)}} \\
\midrule
Runtime (sec) & 0.43 & 0.04 & 1.03 & 1.71 \\
\bottomrule
\end{tabular}
}
\caption{Inference-time runtime per image (sec) on FairFace (Gender $\times$ Race).}
\label{tab:runtime}
\end{table}

\textsc{OrthCali} requires solving an optimization problem upfront, making its inference-time cost closer to zero-shot. The reported \textsc{FairTPT} runtime includes automatic learning-rate adaptation (ELRA). Unlike \textsc{OrthCali}, which requires tuning $\lambda_{\text{orth}}$ using a \emph{labeled} validation set (contrary to the unsupervised test-time setting), \textsc{FairTPT} is robust to $\lambda_{\text{fair}}$ and does not require additional tuning.

\paragraph{Support for Multiple Sensitive Attributes.} \textsc{FairTPT} is not restricted to a single sensitive attribute and supports multiple attributes via: 
\begin{itemize}
    \item Independent Treatment: Extend the optimization objective in \Eqref{eq:fair-tpt-objective} by adding spurious entropy terms for each sensitive attribute (e.g., age, gender, race).
    \item Joint Treatment: Define a joint sensitive attribute over the product space (e.g., age $\times$ gender $\times$ race) and apply \Eqref{eq:fair-tpt-objective} directly.
\end{itemize}
As a test-time debiasing approach, \textsc{FairTPT} allows the fairness auditor to specify sensitive attributes at inference. For an unlabeled test image, the user can select attributes and their possible values for debiasing. If gender is specified, debiasing applies only to gender. The method does not automatically detect spurious attributes but can incorporate user-provided or externally inferred factors (e.g., via GPT-based tools).

\paragraph{Experiments using CLIP ViT-B/32 as Base Model.} Table~\ref{tab:main_results_reduced_b32} summarizes the overall and subgroup-level results (when using CLIP ViT-B/32 as base model). For all baselines we keep their recommended hyperparameters fixed across datasets. See Table~\ref{tab:hyperparam} (Appendix \ref{app:methods}) for the hyperparameter values of our methods and baselines. Overall, when averaged across all datasets, our methods' accuracy and subgroup metrics are comparable to or better than those of all baselines.

\setlength{\fboxsep}{1.5pt}
\begin{table*}[h]
\centering
\tiny
\resizebox{\linewidth}{!}{
\begin{tabular}{l@{\hskip 0.05in}
c@{\hskip 0.02in}c@{\hskip 0.02in}c@{\hskip 0.02in}c@{\hskip 0.05in}
c@{\hskip 0.02in}c@{\hskip 0.02in}c@{\hskip 0.02in}c@{\hskip 0.05in}
c@{\hskip 0.02in}c@{\hskip 0.02in}c@{\hskip 0.02in}c@{\hskip 0.05in}
c@{\hskip 0.02in}c@{\hskip 0.02in}c@{\hskip 0.02in}c@{\hskip 0.05in}
c@{\hskip 0.02in}c@{\hskip 0.02in}c@{\hskip 0.02in}c@{\hskip 0.05in}
c@{\hskip 0.02in}c@{\hskip 0.02in}c@{\hskip 0.02in}c@{\hskip 0.05in}|
c@{\hskip 0.02in}c@{\hskip 0.02in}c@{\hskip 0.02in}c}
\toprule
\textsc{Method} 
& \multicolumn{4}{c}{\textsc{FairFace}} 
& \multicolumn{8}{c}{\textsc{CelebA}} 
& \multicolumn{4}{c}{\textsc{WaterBirds}} 
& \multicolumn{8}{c}{\textsc{UTKFace}} 
& \multicolumn{4}{c}{\textsc{Average Results}} \\
\cmidrule(lr){2-5} 
\cmidrule(lr){6-13} 
\cmidrule(lr){14-17} 
\cmidrule(lr){18-25} 
& \multicolumn{4}{c}{$\text{Gender} \times \text{Race}$} 
& \multicolumn{4}{c}{$\text{Hair color} \times \text{Gender}$} 
& \multicolumn{4}{c}{$\text{Smiling} \times \text{Gender}$} 
& \multicolumn{4}{c}{$\text{Type} \times \text{Background}$} 
& \multicolumn{4}{c}{$\text{Age} \times \text{Race}$} 
& \multicolumn{4}{c}{$\text{Gender} \times \text{Race}$} 
& \multicolumn{4}{c}{} \\
\cmidrule(lr){2-5} 
\cmidrule(lr){6-9} 
\cmidrule(lr){10-13} 
\cmidrule(lr){14-17} 
\cmidrule(lr){18-21} 
\cmidrule(lr){22-25} 
\cmidrule(lr){26-29}
& A & WGA & B & EOD 
& A & WGA & B & EOD 
& A & WGA & B & EOD 
& A & WGA & B & EOD 
& A & WGA & B & EOD 
& A & WGA & B & EOD 
& A & WGA & B & EOD \\
\midrule


\textsc{Zero-Shot} & \textbf{93.3}& 82.0& 11.4& 17.3& \underline{80.3}& 72.1& 8.2& 17.6& \underline{85.4}& \underline{75.8}& \underline{9.6}& 6.6& 73.6& 50.9& 22.7& \underline{40.1}& \underline{81.5}& 62.6& 18.9& 32.6& 95.3& 85.0& 10.3& 12.5& 84.9& 71.4& 13.5& 21.1 \\
\midrule
\multicolumn{22}{l}{\textit{Episodic test-time adaptation methods}} \\
\textsc{TPT} & 93.1& \colorbox{green!25}{\textbf{84.8}}& \colorbox{green!25}{\underline{8.3}}& \colorbox{green!25}{\textbf{13.9}}& \colorbox{red!25}{77.5}& \colorbox{red!25}{69.9}& 7.6& \colorbox{green!25}{\underline{9.5}}& \textbf{86.7}& \colorbox{green!25}{\textbf{79.2}}& \colorbox{green!25}{\textbf{7.5}}& \colorbox{green!25}{\textbf{3.8}}& \colorbox{red!25}{71.2}& \colorbox{red!25}{44.1}& \colorbox{red!25}{27.0}& \colorbox{red!25}{48.8}& \colorbox{red!25}{61.9}& \colorbox{red!25}{33.6}& \colorbox{red!25}{28.2}& \colorbox{green!25}{\underline{20.5}}& 95.4& \colorbox{green!25}{\underline{88.8}}& \colorbox{green!25}{6.6}& \colorbox{green!25}{\underline{9.0}}& \colorbox{red!25}{81.0}& \colorbox{red!25}{66.8}& 14.2& \colorbox{green!25}{\textbf{17.6}} \\
\textsc{Zero} & \colorbox{red!25}{88.5}& \colorbox{red!25}{74.3}& \colorbox{red!25}{14.2}& 19.2& \textbf{80.8}& \underline{72.4}& 8.3& \colorbox{green!25}{13.9}& 83.7& \colorbox{red!25}{71.6}& \colorbox{red!25}{12.1}& \underline{5.6}& \underline{75.3}& \colorbox{green!25}{\underline{54.2}}& \underline{21.1}& 41.5& \colorbox{red!25}{72.8}& \colorbox{red!25}{52.9}& 19.9& \colorbox{green!25}{\textbf{20.3}}& \colorbox{red!25}{92.5}& 85.7& \colorbox{green!25}{6.8}& 11.6& \colorbox{red!25}{82.3}& \colorbox{red!25}{68.5}& 13.7& \colorbox{green!25}{\underline{18.7}} \\
\midrule
\multicolumn{22}{l}{\textit{Episodic test-time debiasing methods}} \\
\textsc{OrthCali} & 92.7& 80.5& 12.3& 17.5& 79.4& \colorbox{green!25}{\textbf{74.6}}& \colorbox{green!25}{\textbf{4.8}}& \colorbox{green!25}{\textbf{5.7}}& \colorbox{red!25}{80.0}& \colorbox{red!25}{61.2}& \colorbox{red!25}{18.8}& \colorbox{red!25}{27.5}& \colorbox{green!25}{\textbf{83.4}}& \colorbox{green!25}{\textbf{69.4}}& \colorbox{green!25}{\textbf{14.1}}& \colorbox{green!25}{\textbf{19.2}}& \colorbox{green!25}{\textbf{83.9}}& \colorbox{green!25}{\textbf{67.5}}& \colorbox{green!25}{\underline{16.5}}& 30.9& 95.0& 83.1& 12.0& 13.5& 85.7& 72.7& 13.1& 19.1 \\
\midrule
\multicolumn{22}{l}{\textit{Our method}} \\
\textsc{FairTPT} & \underline{93.2}& \colorbox{green!25}{\underline{84.7}}& \colorbox{green!25}{\textbf{8.5}}& \colorbox{green!25}{\underline{14.8}}& 79.0& 72.0& \underline{6.9}& \colorbox{green!25}{12.6}& 84.8& 74.8& 10.1& 7.8& 72.4& 48.9& 23.5& 40.6& 79.7& \underline{62.9}& \colorbox{green!25}{16.8}& \colorbox{green!25}{30.5}& 95.1& \colorbox{green!25}{\textbf{89.4}}& \colorbox{green!25}{\textbf{5.8}}& \colorbox{green!25}{\textbf{8.7}}& 84.0& 72.1& \textbf{12.0}& 19.2 \\
\textsc{FairTPT (MO)} & \textbf{93.3}& 83.6& 9.6& 15.6& 78.7& 71.7& 7.1& \colorbox{green!25}{12.8}& 84.7& 74.4& 10.4& 7.7& 72.6& 49.0& 23.6& 41.0& \colorbox{red!25}{78.9}& \underline{62.9}& \colorbox{green!25}{\textbf{16.0}}& \colorbox{green!25}{28.1}& 95.0& \colorbox{green!25}{88.5}& \colorbox{green!25}{\underline{6.5}}& \colorbox{green!25}{9.8}& 83.9& 71.7& \underline{12.2}& 19.2 \\
\bottomrule
\end{tabular}
}
\caption{Overall (Accuracy) and subgroup-level performance (Worst-Group Accuracy, Bias, and Equalized Odds Difference) evaluation of all the methods considered (ours and baselines) with CLIP-ViT-B/32 as base model. We report mean over 5 random seeds and mean aggregation over all equally sized datasets. Best results are in \textbf{bold}, second best \underline{underlined}. The values that improve upon \textsc{Zero-Shot} by more than $2.0$ percentage points are highlighted in \colorbox{green!25}{green}, while degradations greater than $2.0$ percentage points are shown in \colorbox{red!25}{red} (arbitrary threshold). We set $\lambda_\textnormal{fair} = 100$ and $\lambda_\textnormal{fair (mo)} = 100$ for \textsc{FairTPT} and \textsc{FairTPT (MO)}, respectively. Hyperparameter values of all methods are presented in Table~\ref{tab:hyperparam}.}
\label{tab:main_results_reduced_b32}
\end{table*}

\clearpage
\section{Scaling steps in the probability simplex \label{app:simplex}}

Given our rescaling of the learning rate, this section investigates the relationship between changes in a softmax probability vector and changes in its leading ($\argmax$) component. We establish the following result:

\begin{proposition}
Let $\boldsymbol{p}$ be a probability vector in the set ${\cal D}_i=\{ \boldsymbol{x}\in S_n \mid x_i \geq x_j, \forall j \neq i \}$, where $S_n$ denotes the $n$-dimensional probability simplex. We define the distance function $d: S_n^2 \to \mathbb{R}_+$ by
\[
d(\boldsymbol{p},\boldsymbol{q}) ~:=~ \sum_i \max(0,p_i-q_i) ~=~ \frac{1}{2}\sum_i \abs{p_i-q_i} , 
\]
and the index mapping $I: S_n \to \bss{n}$ (where $\bss{n} = \{1, \dots, n\}$) by $I(\boldsymbol{p}) = \min(\argmax_i p_i)$. Then, for all $(\boldsymbol{p},\boldsymbol{q}) \in S_n^2$, the following hold:
\begin{enumerate}
    \item[i)] $ I(\boldsymbol{p}) \neq I(\boldsymbol{q}) \implies d(\boldsymbol{p},\boldsymbol{q}) > \frac{1}{2}(\max_i p_i - \max_{j \neq I(\boldsymbol{p})} p_j)$    
    \item[ii)] $2d(\boldsymbol{p},\boldsymbol{q})> \displaystyle\max_{ \Omega \subset \bss{n} \backslash \{I(\boldsymbol{p})\}} \abs{p_1-\frac{1}{1+\abs{\Omega}}} + \displaystyle\sum_{i=2}^n \abs{p_i-\frac{1}{1+\abs{\Omega}} \sum_{j\in \Omega}\delta_{ij}} \implies I(\boldsymbol{p})\neq I(\boldsymbol{q})$ (where $\delta_{ij} = 1$ if $i = j$ and $0$ otherwise)    
    \item[iii)]
    $d(\boldsymbol{p},\boldsymbol{q}) > \brr{1-\frac{1}{n}} \implies I(\boldsymbol{p})\neq I(\boldsymbol{q})$
\end{enumerate}
\end{proposition}

\begin{proof}
We first note that $d$ is a metric on $S_n$ since it satisfies non-negativity, symmetry, and the triangle inequality. In fact, $d$ is convex in both arguments. Without loss of generality, assume $\boldsymbol{p}\in{\cal D}_1$. 

\begin{enumerate}
\item[i)] We seek the minimal change in $\boldsymbol{p}$ required to leave ${\cal D}_1$:
\[
\inf_{\boldsymbol{q}\in S_n\backslash {\cal D}_1}d(\boldsymbol{p},\boldsymbol{q}).
\]
Since $d$ is half the $\ell_1$ distance, the smallest perturbation that changes the leading index reduces $p_1$ and increases the second-largest component equally. This yields 
\[
\frac{1}{2}(p_1-\max_{i>1}p_i).
\]
In the case $n=2$, this reduces to $p_1 - \frac{1}{2}$.

\item[ii)] We next determine the largest distance within $\mathcal{D}_1$ from $\mathbf{p}$:
\[
\sup_{\boldsymbol{w}\in {\cal D }_1}d(\boldsymbol{p},\boldsymbol{w}) .
\]
Since $\mathcal{D}_1$ is convex and $d$ is convex in each argument, the supremum is attained at an extreme point of $\mathcal{D}_1$:
\[
\sup_{\boldsymbol{w}\in {\cal D }_1}d(\boldsymbol{p},\boldsymbol{w}) ~=~ \sup_{\boldsymbol{w}\in \text{ext}({\cal D }_1)}d(\boldsymbol{p},\boldsymbol{w}) .
\]
The extreme points have the form 
\[ 
\brr{\frac{1}{1+\abs{\Omega}}, \;a_2, \dots, a_n}, 
\] where $\Omega \subset \{2,\dots,n\}$ and $a_i = \frac{1}{1+\abs{\Omega}}$ if $i \in \Omega$, and $a_i = 0$ otherwise. This yields 
\[
\max_{\boldsymbol{w}\in \text{ext}({\cal D }_1)}2d(\boldsymbol{p},\boldsymbol{w}) ~=~ \max_{ \Omega \subset \bss{n} \backslash \bss{1}} \abs{p_1-\frac{1}{1+\abs{\Omega}}} +\displaystyle\sum_{i=2}^n \abs{p_i-\frac{1}{1+\abs{\Omega}}\sum_{j\in \Omega}\delta_{ij}} . 
\]
For $n=2$, this simplifies to
$$
\begin{aligned}
    \max_{\boldsymbol{w}\in \text{ext}({\cal D }_1)}d(\boldsymbol{p},\boldsymbol{w}) &=\frac{1}{2}\max \bcc{\abs{p_1-1} + \abs{1-p_1}, \abs{p_1-\frac{1}{2}} + \abs{1-p_1-\frac{1}{2}}} \\
    &= (1-p_1) \mathbf{1}\bss{p_1\leq\frac{3}{4}} + \brr{p_1-\frac{1}{2}} \mathbf{1}\bss{p_1>\frac{3}{4}} .
\end{aligned}
$$

\item[iii)] Finally, to obtain a bound independent of $\boldsymbol{p}$, we compute
$$
    \begin{aligned}
    \sup_{(\boldsymbol{p},\boldsymbol{w})\in {\cal D }_1^2}d(\boldsymbol{p},\boldsymbol{w}).
    \end{aligned}
$$
By convexity of ${\cal D}_1$ and $d$, this supremum is attained at extreme points: 
$$
    \begin{aligned}
    \sup_{(\boldsymbol{p},\boldsymbol{w})\in {\cal D }_1^2}d(\boldsymbol{p},\boldsymbol{w}) = \max_{(\boldsymbol{p},\boldsymbol{w})\in \text{ext}({\cal D }_1)^2}d(\boldsymbol{p},\boldsymbol{w}).
    \end{aligned}
$$
Given the definition of $d$, the maximum distance is attained for a pair of vectors of the form $(\frac{1}{m_1},\dots,\frac{1}{m_1},0,\dots,0)$ and $(\frac{1}{m_2},0,\dots,0,\frac{1}{m_2},\dots,\frac{1}{m_2})$, where $1\leq m_1, m_2\leq n$ and $m_1\leq m_2$ without loss of generality. The associated distance is $1-\frac{1}{m_2}$, which is maximized when $m_2=n$. This forces $m_1=1$, recovering the known result that the maximum is attained between the vertices $(1,0,\dots,0)$ and $(\frac{1}{n},\dots,\frac{1}{n})$ vertices. Hence,
$$
    \begin{aligned}
    \sup_{(\boldsymbol{p},\boldsymbol{w})\in {\cal D }_1^2}d(\boldsymbol{p},\boldsymbol{w})=\left(1-\frac{1}{n}\right).
    \end{aligned}
$$
For $n=2$, this equals $\frac12$, and for large $n$, it approaches $1$.
\end{enumerate}
\end{proof}

These results suggest the following:
\begin{enumerate}
    \item Statement (i) indicates that a measure of the initial confidence provides an upper bound on indifference, while (ii) and (iii) yield lower bounds for collapse. The learning rate must therefore be chosen carefully, taking into account both the input and the step size, to avoid either collapse or indifference. 
    \item A lower learning rate increases the likelihood of accuracy changes for the least confident samples, i.e., those lying close to vertices with at least two similar probabilities.
\end{enumerate}

We now compute explicitly the typical distance covered by one optimizer step in the linear regime. The parameter update is $\Delta\vt_\textnormal{ctx}=-\eta\nabla_{\vt_\textnormal{ctx}}\ell$, where $\nabla_{\vt_\textnormal{ctx}}\ell = w_1 \nabla_{\vt_\textnormal{ctx}} \ell_\textnormal{ent} (x, \bcc{\vt_\textnormal{ctx}; \gY})+w_2g_{\vt_\textnormal{ctx}}^{\cal S}$ with $w_1,w_2$ normalized weights given by our method or UPGrad aggregation, and $g_{\vt_\textnormal{ctx}}^{\cal S}$ the spurious entropy contribution. We recall that 
$$
\nabla_{\vt_\textnormal{ctx}} \ell_\textnormal{ent} (x, \bcc{\vt_\textnormal{ctx}; \gY}) = -\sum_{i}(1+\log \bar{p}_i)\nabla_{\vt_\textnormal{ctx}}\bar{p}_i,
$$
which, in the linear regime and in terms of $\boldsymbol{p}$, leads to
$$
\begin{aligned}
    d\left( \boldsymbol{p}({\vt_\textnormal{ctx}}),\boldsymbol{p}({\vt_\textnormal{ctx}}+\Delta{\vt_\textnormal{ctx}}) \right) & \simeq\frac{1}{2}\eta\sum_i \abs{\nabla_{\vt_\textnormal{ctx}}{l}\cdot\nabla_{\vt_\textnormal{ctx}} p_i} \\
    &= \frac{1}{2}\eta\sum_i \abs{-w_1\sum_{j}(1+\log \bar{p}_j)\nabla_{\vt_\textnormal{ctx}}\bar{p}_j\cdot\nabla_{\vt_\textnormal{ctx}} p_i+w_2g_{\vt_\textnormal{ctx}}^{\cal S}\cdot\nabla_{\vt_\textnormal{ctx}} p_i}.
\end{aligned}
$$
This expression highlights the high dependency of the distance, and so of an appropriate choice of $\eta$, on the initial softmax output. For a given $\boldsymbol{p}$, there is no straightforward way to choose $\eta$ to yield the same distance across all steps and inputs. Setting $\eta$ via ELRA to ensure a consistent change in the target marginal entropy is thus a heuristic approach. In the linear approximation,
$$\beta := 0.01 \simeq \eta \abs{\nabla_{\vt_\textnormal{ctx}} \ell \cdot\sum_{i}(1+\log \bar{p}_i)\nabla_{\vt_\textnormal{ctx}}\bar{p}_i}, $$
motivated by the direct link between confidence (measured by entropy) and accuracy in a calibrated model. 
}


\end{document}